%% file: Paper.tex
\documentclass[12pt]{article}
\pdfoutput=1

\usepackage{amsmath}
\usepackage{amsfonts}
\usepackage{amssymb}

\usepackage[T1]{fontenc}

\usepackage{booktabs,amsthm}
\usepackage{natbib,lscape}
\usepackage{macros}
\usepackage{pdfpages}
\usepackage{multirow,bigdelim}

\usepackage{graphicx,xcolor,mathrsfs}
\usepackage{setspace}
\usepackage{bbm, bm}
\usepackage{url}
\usepackage{varwidth}
\usepackage{multirow}
\usepackage{cancel}
\usepackage[normalem]{ulem}
\usepackage{subfig, float}
\usepackage{rotating}
\usepackage{array}
\allowdisplaybreaks
\usepackage{fancyvrb}

\newcommand{\PreserveBackslash}[1]{\let\temp=\\#1\let\\=\temp}
\newcolumntype{C}[1]{>{\PreserveBackslash\centering}p{#1}}
\renewcommand{\d}{\mathrm{d}}

\numberwithin{equation}{section}
\newcommand{\1}{\mathbbm{1}}

\def\bth{{\boldsymbol{\theta}}}

\def\E{\mathbb{E}}

\def\L{{\mathcal L}}

\usepackage{array,bm,bbm,url,float}
\usepackage{multirow}% http://ctan.org/pkg/multirow
\usepackage{url}
\usepackage[ruled, linesnumbered, lined, commentsnumbered]{algorithm2e}
\usepackage{zref-xr}
\zxrsetup{toltxlabel=true, tozreflabel=false}
\newcommand{\beq}{\begin{equation}}
\newcommand{\eeq}{\end{equation}}
\newcommand{\beas}{\begin{eqnarray*}}
\newcommand{\eeas}{\end{eqnarray*}}
\newcommand{\bea}{\begin{eqnarray}}
\newcommand{\eea}{\end{eqnarray}}
\newcommand{\bei}{\begin{itemize}}
\newcommand{\eei}{\end{itemize}}
\newcommand{\ben}{\begin{enumerate}}
\newcommand{\een}{\end{enumerate}}

\newcommand{\argmin}{\mathop{\rm arg\min}}
\newcommand{\argmax}{\mathop{\rm arg\max}}

\newtheorem{Lemma}{Lemma}
\newtheorem{fact}{Fact}
\newtheorem{example}{Example}

\newtheorem{Definition}{Definition}
\newtheorem{Theorem}{Theorem}

\newcommand{\A}{{\mathbf{A}}}

\newcommand{\bmu}{\bm{\mu}}
\def\bbeta{{\boldsymbol{\beta}}}

\newcommand{\R}{\mathbb{R}}
\newcommand{\Pro}{{\mathbb{P}}}

\newcommand{\supp}{{\rm supp}}
\newcommand{\Var}{{\rm Var}}

\newcommand{\ignore}[1]{}
\begin{document}

\title{The Cost of Privacy in Generalized Linear Models: Algorithms and Minimax Lower Bounds}
\date{}
	\author{}
{{\author{T. Tony Cai$^{1}$, Yichen Wang$^{1}$,  \ and \ Linjun Zhang$^{2}$\\
			Department of Statistics, University of Pennsylvania$^1$\\
			Department of Statistics, Rutgers University$^2$\\
			}}}
\maketitle
% abstract
\begin{abstract}
We propose differentially private algorithms for parameter estimation in both low-dimensional and high-dimensional sparse generalized linear models (GLMs)  by constructing private versions of projected gradient descent. We show that the proposed algorithms are nearly rate-optimal by characterizing their statistical performance and establishing privacy-constrained minimax lower bounds for GLMs. The lower bounds are obtained via a novel technique, which is based on Stein's Lemma and generalizes the tracing attack technique for privacy-constrained lower bounds. This lower bound argument can be of independent interest as it is applicable to general parametric models. Simulated and real data experiments are conducted to demonstrate the numerical performance of our algorithms.

\end{abstract}
 
 \input{1_Introduction}

\input{2_Formulation}
\input{3_Algorithms}

\input{4_Lower_bounds}
\input{5_Experiments}
\input{6_Discussion} 
\input{7_Proofs}

{\subsection*{Acknowledgments}
The research of T. Cai was supported in part by the National Science Foundation grants DMS-1712735 and DMS-2015259 and the National Institutes of Health grants R01-GM129781 and R01-GM123056. The research of L. Zhang was supported in part by the National Science Foundation grant DMS-2015378.}

% bibliography
\bibliographystyle{plain}
\bibliography{reference}	\label{lastpage}

\newpage
\appendix
\input{A_Upper_bound_proofs}
\input{B_Lower_bound_proofs} 

\end{document}

%% file: 1_Introduction.tex
%!TEX root = Privacy-GLM-JRSSB.tex
%%%%%%%%%%%%%%%%%%%%%%%%%%%%%%%%%%%%%%%%%%%%%%%%%%%%%%%%%%%%%%%%%%%
\section{Introduction}\label{sec: introduction}
 
Statistical and machine learning algorithms are gaining prominence in our daily lives, and so are demands for data privacy guarantees by these algorithms. The need for data privacy protection has in turn inspired the development of formal criteria and frameworks for data privacy, with differential privacy \cite{dwork2006our, dwork2006calibrating} (and its variants \cite{kasiviswanathan2011can, dwork2016concentrated, mironov2017renyi, dong2019gaussian}) being the most widely studied in theory \cite{dwork2010boosting, dwork2014algorithmic, dwork2015robust, abadi2016deep}, and adopted in practice \cite{apple2018privacy, abowd2016challenge, ding2017collecting, erlingsson2014rappor}. Much of its popularity can be attributed to the ease of building privacy-preserving algorithms that the differential privacy framework affords \cite{dwork2006calibrating, mcsherry2007mechanism, dwork2010boosting, dwork2014algorithmic}, but privacy can also come at a cost: it has been observed that requiring algorithms to be differentially private may sacrifice their statistical accuracy \cite{barber2014privacy, fienberg2010differential, lei2011differentially}.
  
The quest for privacy-preserving yet statistically accurate algorithms has since become a vibrant field of research. On the methodological front, a variety of popular computational and statistical methods has seen their differentially private counterparts, for some examples: causal inference \cite{lee2019private, lee2019privacy}, deep learning \cite{abadi2016deep, phan2016differential}, and multiple testing \cite{dwork2018differentially}. On the theoretical side, however, the study of statistical optimality of differentially private algorithms focuses more heavily on the simpler and more stylized problems, such as mean estimation \cite{barber2014privacy, kamath2018privately, kamath2020private}, top-$k$ selection \cite{bafna2017price, steinke2017tight}, and linear regression \cite{cai2019cost}.
 
In this paper, we study a broadly applicable model, the generalized linear model (GLM) \cite{nelder1972generalized, mccullagh1989generalized}, by proposing differentially private algorithms for parameter estimation with theoretical guarantees. We characterize their statistical performance, and prove their near-optimality by establishing minimax lower bounds. In this work, we consider both the classical low-dimensional setting and the contemporary high-dimensional setting. 
  
\subsection{Our Contribution and Related Literature} \label{sec: contribution} 
Our main contribution is two-fold: constructing differentially private algorithms for GLM parameter estimation (Section \ref{sec: glm upper bounds}), and establishing the near-optimality for the algorithms via privacy-constrained minimax lower bounds for GLM parameter estimation (Section \ref{sec: glm lower bounds}).

\textbf{Private algorithms for GLMs.} We construct algorithms, based on noisy gradient descent \cite{bassily2014private, bassily2019private} and noisy iterative hard thresholding \cite{cai2019cost, jain2014iterative, blumensath2009iterative}, for privately estimating the vector of GLM parameters. 

There has been an extensive literature on private logistic regression \cite{chaudhuri2009privacy, chaudhuri2011differentially, zhang2020privately}, and more broadly, private empirical risk minimization \cite{bassily2014private, kifer2012private, bassily2019private}. Our work is inspired by but distinct from previous works in its focus on parameter estimation accuracy, as opposed to excess risk of the solution; for statisticians, parameter estimation accuracy is also arguably a more informative measure of performance than excess risks. Since the log-likelihood function of GLM in general lacks strong convexity \cite{sur2017likelihood}, bounding the distance between an estimator and the true parameter requires more refined analysis of the algorithms. Theorem \ref{thm: non-sparse glm upper bound} shows that $(\varepsilon, \delta)$-differentially private estimation of the GLM parameters can be achieved by the noisy gradient descent algorithm (Algorithm \ref{algo: private glm} based on \cite{bassily2014private}) with an extra privacy cost of $\tilde O\left(\frac{d^2\log(1/\delta)}{n^2\varepsilon^2}\right)$ in terms of the squared $\ell_2$ risk, where $n, d$ respectively denote the sample size and the dimension of the parameter vector.

The difficulty posed by a ``flat'' log-likelihood landscape \cite{negahban2009unified, agarwal2010fast} is even more salient in the high-dimensional setting: when the number of parameters exceeds the sample size, strong convexity is categorically impossible for any objective function.  We instead leverage the sparsity of the parameter vector to design a noisy iterative hard thresholding algorithm (Algorithm \ref{algo: noisy iterative hard thresholding}), which attains convergence in $O(\log n)$ iterations and incurs an extra privacy cost of $\tilde O\left(\frac{(s^*\log d)^2\log(1/\delta)}{n^2\varepsilon^2}\right)$ in terms of the squared $\ell_2$ risk (Theorem \ref{thm: glm upper bound}), where $s^*$ denotes the sparsity of the parameter vector. In particular, the linear dependence on sparsity and logarithmic dependence on the ambient dimension suggest that differentially private estimation remains feasible in high dimensions, which contrasts with the impossibility of high-dimensional estimation \cite{duchi2018minimax, duchi2018right} under the more restrictive local differential privacy framework \cite{kasiviswanathan2011can}. The technical analysis of our algorithm can also be extended to private estimation in other sparse $M$-estimation problems that enjoy restricted strong convexity and restricted smoothness properties.

\textbf{Minimax lower bounds.} We develop a novel lower bound technique based on Stein's Lemma and show that the statistical accuracy of our algorithms are optimal up to logarithmic factors in the sample size, by establishing privacy-constrained minimax lower bounds for GLM parameter estimation (Theorems \ref{thm: low-dim glm lb} and \ref{thm: high-dim glm lb}).

Our strategy for establishing these lower bounds entails a broad generalization of the ``tracing attack'' techniques, first developed by \cite{bun2014fingerprinting, dwork2015robust} and further applied to various statistical problems, including sharp lower bounds for classical Gaussian mean estimation and linear regression \citep{kamath2018privately, cai2019cost}, as well as lower bounds for sparse mean estimation and linear regression in the high-dimensional setting \citep{steinke2017tight, cai2019cost}.
In these previous works, the design of tracing attacks seems to be largely ad hoc and catered to specific distribution families such as Gaussian or Beta-Binomial; a general principle for designing attacks has not been observed. Although some promising proposals have been made along this direction \citep{shokri2017membership, murakonda2019ultimate}, it remains unclear whether the suggested attacks in these works actually imply any lower bound results.

In the present paper, we address this problem by proposing a ``score attack'' based on and named after the score statistic, that is the gradient of the log-likelihood function with respect to the parameter vector. Not only does the score attack imply lower bounds for the GLM problems in the present paper, it also opens paths for lower bound analysis in a much greater range of statistical problems with differential privacy constraints, as the form of our score attack and its theoretical properties (Theorem \ref{thm: score attack general}) are applicable to general parametric families of distributions .

\subsection{Structure of the Paper}\label{sec: structure}
The paper is organized as follows. Section \ref{sec: formulation} formulates the problem and provides necessary background information. Section \ref{sec: glm upper bounds} describes the algorithms and analyzes in detail their privacy guarantees as well as statistical accuracy. Section \ref{sec: glm lower bounds} introduces the score attack framework for minimax lower bounds, and applies the framework to establish minimax lower bounds for the GLM problems. Section \ref{sec: experiments} provides simulated and real data examples that illustrate the numerical performance of our algorithms. Section~\ref{sec: discussion} summarizes our work and discusses its implication for future research. Main technical results are proved in Section \ref{sec: proofs}, and other auxiliary results in the Appendix. 

\textbf{Notation.} For real-valued sequences $\{a_n\}, \{b_n\}$, we write $a_n \lesssim b_n$ if $a_n \leq cb_n$ for some universal constant $c \in (0, \infty)$, and $a_n \gtrsim b_n$ if $a_n \geq c'b_n$ for some universal constant $c' \in (0, \infty)$.  We say $a_n \asymp b_n$ if $a_n \lesssim b_n$ and $a_n \gtrsim b_n$. $c, C, c_0, c_1, c_2, \cdots, $ and so on refer to universal constants in the paper, with their specific values possibly varying from place to place.

 For a vector $\bm v \in \R^d$ and a subset $S \subseteq [d]$, we use $\bm v_S$ to denote the restriction of vector $\bm v$ to the index set $S$. We write $\supp(\bm v) := \{j \in [d]: v_j \neq 0\}$. $\|\bm v\|_p$ denotes the vector $\ell_p$ norm for $ 1\leq p \leq \infty$, with an additional convention that $\|\bm v\|_0$ denotes the number of non-zero coordinates of $\bm v$. For a function $f: \R \to \R$, $\|f\|_\infty$ denotes the the essential supremum of $|f|$. For $t \in \R$ and $R > 0$, let $\Pi_R(t)$ denote the projection of $t$ onto the closed interval $[-R, R]$. For a random variable $X$, we use $\mathrm{ess}\sup(X)=\inf\{c:\Pro(X<c)=1\}$ to denote  the essential supremum of $X$.

%% file: 2_Formulation.tex
%!TEX root = Privacy-GLM-JRSSB.tex
%%%%%%%%%%%%%%%%%%%%%%%%%%%%%%%%%%%%%%%%%%%%%%%%%%%%%%%%%%%%%%%%%%%

\section{Problem Formulation}\label{sec: formulation}
In this section, we present a detailed description of the scope of statistical models (generalized linear models) and algorithms (differentially private algorithms) to be studied in this paper, and formally define the ``cost of privacy'' in terms of minimax risks.

\subsection{Generalized Linear Models}
We study parameter estimation in generalized linear models. In a generalized linear model, the response variable $y \in \R$, conditional on the design vector $\bm x \in \R^d$,  follows a distribution of the natural exponential family form, 
\begin{align}\label{eq: glm general form}
f_{\bbeta^*}(y|\bm x) = h(y, \sigma)\exp\left(\frac{(\bm x^\top \bbeta^*)g(y) - \psi(\bm x^\top \bbeta^*)}{c(\sigma)}\right),
\end{align}
where $c(\sigma)$ is a nuisance scale parameter and $\psi(\cdot)$ is the cumulant generating function of $y$ given $\bm x$. The generalized linear model is, first of all, a generalization of the linear model: setting $g(y) = y$, $\psi(u) = u^2/2$  and $c(\sigma) = \sigma^2$ in \eqref{eq: glm general form} recovers the (Gaussian) linear model. Model \eqref{eq: glm general form} also subsumes other special cases such as logistic and multinomial regression.

Throughout the paper, our goal is estimating $\bbeta^* \in \R^d$ using an i.i.d. sample $\{(y_i, \bm x_i)\}_{i \in [n]}$ drawn from the model \eqref{eq: glm general form}. We shall consider both the classical setting, where the dimension $d$ is dominated by the sample size $n$, and the high-dimensional sparse setting where $d$ potentially dominates $n$ but only a small proportion of $\bbeta^*$'s coordinates are non-zero.

In either case, the issue of data privacy is relevant, as any nontrivial estimator of $\bbeta^*$ must take the data $\{(y_i, \bm x_i)\}_{i \in [n]}$ as input. Before considering concrete estimators and their performance, let us first define the desired criteria of privacy protection.

\subsection{Differential Privacy}
Intuitively speaking, an algorithm $M$ applied over a data set $\bm X$ compromises data privacy if an adversary is able to correctly infer from the algorithm's output $M(\bm X)$ whether an individual datum $\bm x$ belongs to $\bm X$ or not. 

The notion of differential privacy formalizes this idea by requiring that, for every pair of data sets $\bm X$ and $\bm X'$ that differ by a single datum, hereafter called ``adjacent data sets'', the algorithm $M$ is randomized so that the distributions of $M(\bm X)$ and of $M(\bm X')$ are close to each other. 

\begin{Definition}[Differential Privacy, \cite{dwork2006calibrating}]\label{def: differential privacy}
A randomized algorithm $M: \mathcal X^n \to \mathcal R$ is $(\varepsilon, \delta)$-differentially private if for every pair of adjacent data sets $\bm X, \bm X' \in \mathcal X^n$ that differ by one individual datum and every (measurable) $S \subseteq \mathcal R$, 
\begin{align*}
	\Pro\left(M(\bm X) \in S\right) \leq e^\varepsilon \cdot \Pro\left(M(\bm X') \in S\right) + \delta,
\end{align*}
where the probability measure $\Pro$ is induced by the randomness of $M$ only.
\end{Definition}
The definition guarantees that, for small values of $\varepsilon, \delta \geq 0$, the distributions of $M(\bm X)$ and $M(\bm X')$ are almost indistinguishable. But beyond its strong privacy guarantees, the notion of  differential privacy is desirable also for the ease and flexibility of constructing differentially private algorithms. We summarize here some useful facts for our construction of algorithms in this paper.

First, a large class of non-private algorithms can be made differentially private via random perturbations. 
\begin{fact}[The Laplace and Gaussian mechanisms, \cite{dwork2006calibrating, dwork2014algorithmic}] \label{fc: laplace and gaussian mechanisms}
Let $M: \mathcal X^n \to \R^d$ be an algorithm that is not necessarily differentially private.
\begin{itemize}
\item Suppose $\sup_{\bm X, \bm X' \text{adjacent}} \|M(\bm X) - M(\bm X')\|_1 < B < \infty$. For $\bm w \in \R^d$ with its coordinates $w_1, w_2, \cdots, w_d \stackrel{\text{i.i.d.}}{\sim}$ Laplace$(B/\varepsilon)$, $M(\bm X) + \bm w$ is $(\varepsilon, 0)$-differentially private. 
\item If instead we have $\sup_{\bm X, \bm X' \text{adjacent}} \|M(\bm X) - M(\bm X')\|_2 < B < \infty$, for $\bm w \sim N_d(\bm 0, \sigma^2\bm I)$ with $\sigma^2 = 2B^2\log(2/\delta)/\varepsilon$, $M(\bm X) + \bm w$ is $(\varepsilon, \delta)$-differentially private.
\end{itemize}
That is, if a non-private algorithm's output is not too sensitive to changing any single datum in the input data set, perturbing the algorithm with Laplace or Gaussian noises produces a differentially private algorithm.
\end{fact}

Second, differential privacy is preserved under compositions, albeit with weaker privacy parameters. The composition theorems, stated below, explicitly quantify how privacy parameters degrade as private algorithms are composited.
\begin{fact}[Composition Theorems]\label{fc: composition theorems}
Consider an $(\varepsilon, \delta)$-differentially private algorithm $M_0: \mathcal X^n \to \mathcal R_0$ and $(\varepsilon, \delta)$-differentially private algorithms $M_{i}: \mathcal X^n \times \mathcal R_{i-1} \to \mathcal R_i$ for $i = 1, 2, \cdots, k-1$. Consider the composite algorithm $M = M_{k-1} \circ M_{k-2} \circ \cdots \circ M_0$.
\begin{itemize}
\item Composition theorem \cite{dwork2006calibrating}. $M$ is $(k\varepsilon, k\delta)$-differentially private.
\item Advanced composition \cite{dwork2010boosting}. For every $\delta' > 0$, $M$ is $(\sqrt{2k\log(1/\delta')}\varepsilon + k(e^\varepsilon-1)\varepsilon, k\delta + \delta')$-differentially private.
\end{itemize}
\end{fact}
It is worth noting that the notion of ``composition'' considered here, termed ``$k$-fold adaptive composition'' in the literature, is more general than the composition of functions in the usual sense: each part of the composite algorithm may access the same data set, or a different data set, after receiving the output from its previous part. In fact, if a differentially private algorithm is simply post-processed independent of data, there is no deterioration of privacy whatsoever.
\begin{fact}[Post-processing \cite{dwork2006calibrating, wasserman2010statistical}]\label{fc: post-processing}
	Consider an $(\varepsilon, \delta)$-differentially private algorithm $M: \mathcal X^n \to \mathcal R$. If $g$ is a measurable function, then $g(M)$ is also $(\varepsilon, \delta)$-differentially private.
\end{fact}
The composition and post-processing properties will be particularly useful for analyzing the privacy guarantees of the iterative algorithms considered in Section \ref{sec: glm upper bounds}. 

\subsection{The Cost of Privacy in Generalized Linear Models}
Once an algorithm is known to be differentially private, it is natural to ask whether the privacy guarantees come at the expense of accuracy: as seen in Fact \ref{fc: laplace and gaussian mechanisms}, often random perturbations are introduced to achieve differential privacy. In this paper, we assess the accuracy of algorithms via the lens of \textit{minimax risk}, defined as follows.

Let $\{f_\bth: \bth \in \Theta\}$ be a family of statistical models supported over $\mathcal X$ and indexed by the parameter $\bth$. Let $\bm X = \{\bm x_1, \bm x_2, \cdots, \bm x_n\}$ be an i.i.d. sample drawn from $f_{\bth^*}$ for some unknown $\bth^* \in \Theta$, $M: \mathcal X^n \to \Theta$ be an estimator. Let $\ell: \Theta \times \Theta \to \R_{+}$ be a metric on $\Theta$ and $\rho: \R_{+} \to \R_{+}$ be an increasing function. The (statistical) risk of $M$ is given by $\E\rho(\ell(M(\bm X), \bth^*))$, where the expectation is taken over the data distribution $f_{\bth^*}$ and the randomness of estimator $M$. 

Because the risk $\E\rho(\ell(M(\bm X), \bth^*))$ depends on the unknown $\bth^*$ and can be trivially minimized by setting $M(\bm X) \equiv \bth^*$, a more sensible measure of performance is the maximum risk over the entire class of distributions $\{f_\bth: \bth \in \Theta\}$, $\sup_{\bth \in \Theta} \E\rho(\ell(M(\bm X), \bth))$.

The minimax risk of estimating $\bth \in \Theta$ is then given by
\begin{align} \label{eq: unconstrained minimax risk}
	\inf_M \sup_{\bth \in \Theta} \E\rho(\ell(M(\bm X), \bth)).
\end{align}
By definition, this quantity characterizes the best possible worst-case performance that an estimator can hope to achieve over the class of models $\{f_\bth: \bth \in \Theta\}$.

In this paper, we study a \textit{privacy-constrained} version of the minimax risk: let $\mathcal M_{\varepsilon, \delta}$ be the collection of all $(\varepsilon, \delta)$-differentially private algorithms mapping from $\mathcal X^n$ to $\Theta$, we consider
\begin{align}\label{eq: privacy-constrained minimax risk}
\inf_{M \in \mathcal M_{\varepsilon, \delta}} \sup_{\bth \in \Theta} \E\rho(\ell(M(\bm X), \bth)).
\end{align}
As $\mathcal M_{\varepsilon, \delta}$ is a proper subset of all possible estimators, the privacy-constrained minimax risk as defined above will be at least as large as the unconstrained minimax risk, with the difference between these two minimax risks, \eqref{eq: unconstrained minimax risk} and \eqref{eq: privacy-constrained minimax risk} being the ``cost of privacy''.

In our case, the statistical models of interest are the generalized linear models \eqref{eq: glm general form} indexed by the parameter vector $\bbeta^*$, and we would like to precisely characterize the cost of privacy in GLM parameter estimation problems. This goal will be achieved in two steps: Section \ref{sec: glm upper bounds} provides upper bounds of the privacy-constrained minimax risk \eqref{eq: privacy-constrained minimax risk} via analysis of differentially private algorithms; Section \ref{sec: glm lower bounds} establishes corresponding lower bounds of the privacy-constrained minimax risk.

%% file: 3_Algorithms.tex
%!TEX root = Privacy-GLM-JRSSB.tex
%%%%%%%%%%%%%%%%%%%%%%%%%%%%%%%%%%%%%%%%%%%%%%%%%%%%%%%%%%%%%%%%%%%
 
\section{Differentially Private Algorithms for GLMs}\label{sec: glm upper bounds}
In this section, we develop differentially private algorithms for estimating parameters $\bbeta^* \in \R^d$ in the generalized linear model
\begin{align}
f_{\bbeta^*} (y|\bm x) = h(y, \sigma)\exp\left(\frac{\bm x^\top \bbeta^*g(y) - \psi(\bm x^\top \bbeta^*)}{c(\sigma)}\right); \bm x \sim f_{\bm x}. \label{eq: glm definition} 
\end{align}
With an i.i.d. sample $\bm Z = \{\bm z_i\}_{i \in [n]} = \{(y_i, \bm x_i)\}_{i \in [n]}$ drawn from the model \eqref{eq: glm definition}, the general approach is to minimize the following (scaled) negative log-likelihood function in a differentially private fashion:
\begin{align} \label{eq: glm log likelihood}
\L_n(\bbeta; \bm Z) = \frac{1}{n}\sum_{i=1}^n \left(\psi(\bm x_i^\top \bbeta) - g(y_i)\bm x_i^\top \bbeta\right).
\end{align}
We may write $\L_n(\bbeta)$ as a shorthand when the relevant data set is unambiguous.

Since $\L_n$ is convex in $\bbeta$, the problem is an instance of differentially private convex optimization, for which there are many well-studied methods. Roughly speaking, these methods can be organized into two categories depending on the form of random perturbations involved: ``one-shot'' methods \cite{chaudhuri2009privacy, chaudhuri2011differentially, kifer2012private} in which random noises are added only once to the objective function or before reporting the final solution, or iterative, gradient-descent type methods \cite{bassily2014private, bassily2019private} in which random noises are added to each iteration of the algorithm. 

As discussed in Section \ref{sec: contribution}, existing convergence results for these methods are focused on the excess risk of a differentially private minimizer $\bbeta^{\text{priv}}$ of \eqref{eq: glm log likelihood}, compared to the non-private solution $\hat\bbeta$, that is,  $\E\L_n(\bbeta^{\text{priv}}) - \E\L_n(\hat\bbeta)$. The lack of strong convexity in the generalized linear model \eqref{eq: glm definition} precludes the possibility of obtaining bounds for parameter estimation from excess risk bounds. For example, for the logistic regression model, which is obtained from \eqref{eq: glm definition} by setting $g(y) = y$ and $\psi(u) = \log(1+e^u)$. The Hessian of $\L_n$, 
\begin{align*}
	\nabla^2 \L_n(\bbeta)  = \frac{1}{n}\sum_{i=1}^n \psi''(\bm x_i^\top \bbeta) \bm x_i \bm x_i^\top = \frac{1}{n}\sum_{i=1}^n \frac{e^{\bm x_i^\top \bbeta}}{(1 + e^{\bm x_i^\top \bbeta})^2} \bm x_i \bm x_i^\top,
\end{align*}
has its smallest eigenvalue approaching $0$ as $\|\bbeta\|_2 \to \infty$ even in the favorable setting where $n$ is much greater than the dimension of $\bbeta$. When $n$ is dominated by the dimension, $\nabla^2 \L_n(\bbeta)$ is simply rank-deficient and therefore cannot be positive-definite. The absence of strong convexity in GLMs also implies that the ``one-shot'' algorithms are not guaranteed to be differentially private (see \cite{chaudhuri2011differentially, kifer2012private} and the references therein), unless a quadratic penalty term is added to $\L_n$.

Our approach, then, is to consider gradient descent type algorithms. Although strong convexity fails to hold for $\L_n$ globally, it turns out that $\L_n$ satisfies a ``restricted'' and ``local'' sense of strong convexity \cite{negahban2009unified}, to be made precise in Section \ref{sec: non-sparse glm upper bound}, is sufficient for the noisy gradient descent algorithm to enjoy fast convergence and optimal statistical accuracy. In Section \ref{sec: non-sparse glm upper bound}, we analyze in detail the privacy guarantees and convergence rates of the noisy gradient descent algorithm, which works well for the classical setting of $d = o(n)$. 

The high-dimensional, $d \gtrsim n$ setting is considered in Section \ref{sec: sparse glm upper bound}. In this case, consistent estimation of $\bbeta^*$ is not possible even without privacy constraints, unless additional assumptions such as sparsity of $\bbeta^*$ are imposed. When $\bbeta^*$ is indeed sparse, we introduce a noisy iterative hard thresholding algorithm that allows the random perturbations to scale with the sparsity (``intrinsic dimension'') of $\bbeta^*$ rather than the ambient dimension $d$, thereby achieving the optimal statistical accuracy with privacy constraints.

\subsection{The Classical Low-dimensional Setting}\label{sec: non-sparse glm upper bound}
We first consider the classical low-dimensional setting where $d = o(n)$. For minimizing the negative GLM log-likelihood
\begin{align*}
\L_n(\bbeta; \bm Z) = \frac{1}{n}\sum_{i=1}^n \left(\psi(\bm x_i^\top \bbeta) - g(y_i)\bm x_i^\top \bbeta\right)
\end{align*}
in a differentially private fashion, we consider the noisy gradient descent algorithm, first proposed by \cite{bassily2014private} in its generic form for arbitrary convex functions. The following algorithm is a specialization the generic algorithm to GLMs.

\vspace{2mm}

\begin{algorithm}[H]\label{algo: private glm}
	\SetAlgoLined
	\SetKwInOut{Input}{Input}
	\SetKwInOut{Output}{Output}
	\Input{$\L_n(\bbeta, \bm Z)$, data set $\bm Z$, step size $\eta^0$, privacy parameters $\varepsilon, \delta$, noise scale $B$, number of iterations $T$, truncation parameter $R$, initial value $\bbeta^0 \in \R^d$.}
	\For{$t$ in $0$ \KwTo $T-1$}{
		Generate $\bm w_t \in \R^d$ with $w_{t1}, w_{t2}, \cdots, w_{td} \stackrel{\text{i.i.d.}}{\sim} N\left(0, (\eta^0)^2 2B^2\frac{ d\log(2T/\delta)}{n^2(\varepsilon/T)^2}\right)$\;
		Compute $\bbeta^{t + 1} = \bbeta^t - (\eta_0/n)\sum_{i=1}^n  (\psi'(\bm x_i^\top \bbeta^t)-\Pi_{R}(y_i))\bm x_i + \bm w_t$\;
	}
	\Output{$\bbeta^{(T)}$.}
	\caption{Differentially Private Generalized Linear Regression}
\end{algorithm}

Before delving into the analysis of its privacy guarantees and convergence rates, we collect some necessary assumptions here for the clarity of ensuing technical results.
\begin{itemize}
	\item [(D1)] Bounded design: there is a constant $\sigma_{\bm x}< \infty$ such that $\|\bm x\|_{2} < \sigma_{\bm x}\sqrt{d}$ almost surely. 
	\item [(D2)] Bounded moments of design: $\E \bm x = \bm 0$, and the covariance matrix $\Sigma_{\bm x} = \E \bm x \bm x^\top$ satisfies $0 < 1/C < \lambda_{\min}(\Sigma_{\bm x}) \leq \lambda_{\max} (\Sigma_{\bm x}) < C$ for some constant $0 < C < \infty$.
	\item [(G1)] The function $\psi$ in the GLM \eqref{eq: glm definition} satisfies $\|\psi'\|_\infty < c_1$ for some constant $c_1 < \infty$.
	\item [(G2)] The function $\psi$ satisfies $\|\psi''\|_\infty < c_2$ for some constant $c_2 < \infty$.
\end{itemize}
These assumptions are comparable to those required for the theoretical analysis of GLMs in the non-private setting; for examples, see \cite{negahban2009unified, loh2015regularized, wainwright2019high} and the references therein.

Let us first consider the privacy guarantees of Algorithm \ref{algo: private glm}. Because the algorithm is a composition of $T$ individual steps, if each step is $(\varepsilon/T, \delta/T)$-differentially private, the overall algorithm would be $(\varepsilon, \delta)$-differentially private in view of Fact \ref{fc: composition theorems}. This is indeed the case under appropriate assumptions.

\begin{Lemma}\label{lm: non-sparse glm privacy}
	If assumptions (D1) and (G1) hold, then choosing $B = 4(R + c_1)\sigma_{\bm x}$ guarantees that Algorithm \ref{algo: private glm} is $(\varepsilon, \delta)$-differentially private.
\end{Lemma}
Lemma \ref{lm: non-sparse glm privacy} is proved in Section \ref{sec: proof of lm: non-sparse glm privacy}. Although the privacy guarantee holds for any number of iterations $T$, choosing $T$ properly has significant implications for the accuracy of Algorithm \ref{algo: private glm}, as a larger value of $T$ introduces greater noises into the algorithm in order to achieve privacy.

Existing results on noisy gradient descent typically call for $O(n)$ \cite{bassily2019private} or $O(n^2)$ \cite{bassily2014private} iterations for minimizing generic convex functions. For the GLM problem, we shall show that $O(\log n)$ iterations suffice, thanks to the restricted strong convexity and restricted smoothness of generalized linear models.

\begin{fact}[\cite{loh2015regularized}, Proposition 1 paraphrased]\label{lm: glm rsc and rsm}
	If assumptions (D1) and (D2) hold, there is a constant $\alpha > 0$ that depends on $\sigma_{\bm x}, C, \psi$ and satisfies
	\begin{align}
	& \langle \nabla \L_n(\bbeta_1) - \nabla \L_n(\bbeta_2), \bbeta_1 - \bbeta_2 \rangle \geq \begin{cases}
	\alpha\|\bbeta_1 - \bbeta_2\|_2^2  - \frac{c^2\sigma^2_{\bm x}}{2\alpha}\frac{\log d}{n}\|\bbeta_1 - \bbeta_2\|_1^2 & \text{if } \|\bbeta_1 - \bbeta_2\|_2 \leq 3, \\
	3\alpha \|\bbeta_1 - \bbeta_2\|_2  - \sqrt{2}c\sigma_{\bm x}\sqrt{\frac{\log d}{n}}\|\bbeta_1 - \bbeta_2\|_1& \text{if } \|\bbeta_1 - \bbeta_2\|_2 > 3, \\
	\end{cases}  \label{eq: glm rsc} 
	\end{align}
	with probability at least $1 - c_3\exp(-c_4n)$. If we further assume (G2), there is a constant $\gamma \geq \alpha > 0$ that depends on $\sigma_{\bm x}, M, c_2$ and satisfies
	\begin{align}
	& \langle \nabla \L_n(\bbeta_1) - \nabla \L_n(\bbeta_2), \bbeta_1 - \bbeta_2 \rangle \leq \gamma\|\bbeta_1 - \bbeta_2\|_2^2 + \frac{4\gamma}{3}\frac{\log d}{n}\|\bbeta_1 - \bbeta_2\|_1^2. \label{eq: glm rsm}
	\end{align}
	with probability at least $1 - c_3\exp(-c_4n)$.
\end{fact}
These weaker versions of strong convexity and smoothness, as it turns out, are sufficient for Algoirthm \ref{algo: private glm} to attain linear convergence, which is the same rate for minimizing strongly convex and smooth functions, and cannot be further improved in general \cite{nesterov2003introductory}. Therefore, $O(\log n)$ iterations would allow the algorithm to converge to an accuracy of $O(n^{-1})$ within $\hat\bbeta$, the true minimizer of $\L_n$, in terms of squared $\ell_2$ norm; as $\E\|\hat\bbeta - \bbeta^*\|^2_2$ is of order $d/n$, from a statistical perspective, there is little reason to run the algorithm further than $O(\log n)$ iterations.
\begin{Theorem}\label{thm: non-sparse glm upper bound}
	Let $\{(y_i, \bm x_i)\}_{i \in [n]}$ be an i.i.d. sample from the GLM \eqref{eq: glm definition}. Suppose assumptions (D1), (D2), (G1) and (G2) are true. Let the parameters of Algorithm $\ref{algo: private glm}$ be chosen as follows.
	\begin{itemize}
		\item Set step size $\eta^0 = 3/4\gamma$, where $\gamma$ is the smoothness constant defined in Fact \ref{lm: glm rsc and rsm}.
		\item Set $R = \min\left(\mathrm{ess}\sup |y_1|, c_1 + \sqrt{2c_2c(\sigma)\log n}\right) \lesssim \sqrt{c(\sigma) \log n}.$
		\item Noise scale $B$. Set $B = 4(R + c_1)\sigma_{\bm x}$.
		\item Number of iterations $T$. Let $T = (2\gamma/\alpha)\log(9n)$, where $\alpha, \gamma$ are the strong convexity and smoothness constants defined in in Fact \ref{lm: glm rsc and rsm}.
		\item Initialization $\bbeta^0$. Choose $\bbeta^0$ so that $\|\bbeta^0 - \hat\bbeta\|_2 \leq 3$, where $\hat\bbeta = \argmin \L_n(\bbeta; Z)$.
	\end{itemize}
	If $n \geq K \cdot \left(Rd\sqrt{\log(1/\delta)}\log n \log\log n/\varepsilon\right)$ for a sufficiently large constant $K$, the output of Algorithm \ref{algo: private glm} satisfies
	\begin{align}
	\|\bbeta^{(T)} - \bbeta^*\|^2_2 \lesssim c(\sigma)\left({\frac{d}{n}} + \frac{d^2 \log(1/\delta)\log^3 n}{n^2\varepsilon^2}\right), \label{eq: non-sparse glm upper bound}
	\end{align}
	with probability at least $1 - c_3\exp(-c_4n) - c_3\exp(-c_4d) - c_3\exp(-c_4\log n).$
\end{Theorem}
Theorem \ref{thm: non-sparse glm upper bound} is proved in Section \ref{sec: proof of thm: non-sparse glm upper bound}. Some further comments may help clarify the theorem. For the choice of various algorithm tuning parameters, we note that the step size, number of iterations and initialization are chosen to assure convergence; in particular the initialization condition, as in \cite{loh2015regularized}, is standard in the literature and can be extended to $\|\bbeta^0 - \hat\bbeta\|_2 \leq 3\max(1, \|\bbeta^*\|_2)$. 

The choice of truncation level $R$ is to ensure privacy while keeping as many data intact as possible; when the distribution of $y$ has bounded support, for example in the logistic model, it can be chosen to be an $O(1)$ constant and therby saving an extra factor of $O(\log n)$ in the second term of \eqref{eq: non-sparse glm upper bound}. The choice of $B$ which depends on $R$ then ensures the privacy of Algorithm \ref{algo: private glm} as seen in Lemma \ref{lm: non-sparse glm privacy}.

Finally, the scaling of $n$ versus $d, \varepsilon$ and $\delta$ in Theorem \ref{thm: non-sparse glm upper bound} is nearly optimal, as our lower bound result, Theorem \ref{thm: low-dim glm lb}, shall imply that no estimator can achieve low $\ell_2$ error unless the assumed scaling holds, and that the statistical accuracy of Algorithm \ref{algo: private glm} cannot be further improved except possibly for factors of $\log n$.

\subsection{The High-dimensional Sparse Setting}\label{sec: sparse glm upper bound}
In this section, we construct differentially private algorithms for estimating GLM parameters when the dimension $d$ dominates the sample size $n$. In this setting, even without privacy requirements, directly minimizing the negative log-likelihood function $\L_n(\bbeta)$ no longer achieves any meaningful statistical accuracy, because the objective function $\L_n$ can have infinitely many minimizers due to a rank-deficient Hessian matrix $\nabla^2 \L_n(\bbeta)
= \frac{1}{n}\sum_{i=1}^n \psi''(\bm x_i^\top \bbeta) \bm x_i \bm x_i^\top$.

The problem is nevertheless solvable when the true parameter vector $\bbeta^*$ is $s^*$-sparse with $s^* = o(d)$, that is when at most $s^*$ out of $d$ coordinates of $\bbeta^*$ are non-zero. For estimating a sparse $\bbeta^*$, the primary challenge lies in (approximately) solving the non-convex optimization problem $\hat\bbeta = \argmin_{\bbeta: \|\bbeta\|_0 \leq s^*} \L_n(\bbeta; \bm Z)$. Some popular non-private approaches include convex relaxation via $\ell_1$ regularization of $\L_n$ \cite{negahban2009unified, agarwal2010fast}, or projected gradient descent onto the non-convex feasible set $\{\bbeta: \|\bbeta\|_0 \leq s^*\}$, also known as iterative hard thresholding \cite{blumensath2009iterative, jain2014iterative}:

\vspace{2mm}

\begin{algorithm}[h]
	\caption{Iterative Hard Thresholding (IHT)}\label{algo: iht}
	\SetAlgoLined
	\SetKwInOut{Input}{Input}
	\SetKwInOut{Output}{Output}
	\SetKwFunction{NoisyHT}{NoisyHT}
	\Input{Objective function $f(\bth)$, sparsity $s$, step size $\eta$, number of iterations $T$.}
	Initialize $\bth^0$ with $\|\bth^0\|_0 \leq s$, set $t = 0$\;
	\For{$t$ in $0$ \KwTo $T-1$}{
		$\bth^{t+1} = P_s\left(\bth^t - \eta\nabla f(\bth^t)\right)$, where $P_s(\bm v) = \argmin_{\bm z: \|\bm z\|_0 = s}\|\bm v - \bm z\|_2^2$\; 
	}
	\Output{$\bth^{(T)}$.}
\end{algorithm}
\vspace{2mm} 
In each iteration, the algorithm updates the solution via gradient descent, keeps its largest $s$ coordinates in magnitude, and sets the other coordinates to $0$. 

For privately fitting high-dimensional sparse GLMs, we shall construct a noisy version of Algorithm \ref{algo: iht}, and show in Section \ref{sec: noisy iht for sparse glm} that it again enjoys a linear rate of convergence, as a consequence of Fact \ref{lm: glm rsc and rsm} and sparsity of $\bbeta^*$. As a first step towards this goal, we consider in Section \ref{sec: general iht algo} a noisy, differentially private version of the projection operator $P_s$, as well as a noisy iterative hard thresholding algorithm applicable to any objective function that satisifies restricted strong convexity and restricted smoothness. 

\subsubsection{The Noisy Iterative Hard Thresholding Algorithm}\label{sec: general iht algo}
At the core of our algoirthm is a noisy, differentially private algorithm that identifies the top-$s$ largest coordinates of a given vector with good accuracy. The following ``Peeling'' algorithm \cite{dwork2018differentially} serves this purpose, with fresh Laplace noises added to the underlying vector and one coordinate ``peeled'' from the vector in each iteration.

\vspace{2mm}

\begin{algorithm}[h]	\caption{Noisy Hard Thresholding (NoisyHT)}
\label{algo: noisy hard thresholding}
	\SetAlgoLined
	\SetKwInOut{Input}{Input}
	\SetKwInOut{Output}{Output}
	\Input{vector-valued function $\bm v = \bm v(\bm Z) \in \R^d$, data $\bm Z$, sparsity $s$, privacy parameters $\varepsilon, \delta$, noise scale $\lambda$.}
	Initialize $S = \emptyset$\;
	\For{$i$ in $1$ \KwTo $s$}{
		Generate $\bm w_i \in \R^d$ with $w_{i1}, w_{i2}, \cdots, w_{id} \stackrel{\text{i.i.d.}}{\sim} \text{Laplace}\left(\lambda \cdot \frac{2\sqrt{3s\log(1/\delta)}}{\varepsilon}\right)$\;
		Append $j^* = \argmax_{j \in [d] \setminus S} |v_j| + w_{ij}$ to $S$\;
	}
	Set $\tilde P_s(\bm v) = \bm v_S$\;
	Generate $\tilde {\bm w}$ with $\tilde w_{1}, \cdots, \tilde w_{d} \stackrel{\text{i.i.d.}}{\sim} \text{Laplace}\left(\lambda \cdot \frac{2\sqrt{3s\log(1/\delta)}}{\varepsilon}\right)$\;
	\Output{$\tilde P_s(\bm v) + \tilde {\bm w}_S$.}
\end{algorithm}

The algorithm is guaranteed to be $(\varepsilon, \delta)$-differentially private when the vector-valued function $\bm v(\bm Z)$ is not sensitive to replacing any single datum.
\begin{Lemma}[\cite{dwork2018differentially, cai2019cost}]\label{lm: noisy hard thresholding privacy}
	If for every pair of adjacent data sets $\bm Z, \bm Z'$ we have $\|\bm v(\bm Z)- \bm v(\bm Z')\|_\infty < \lambda$, then NoisyHT is an $(\varepsilon, \delta)$-differentially private algorithm.
\end{Lemma}

The accuracy of Algorithm \ref{algo: noisy hard thresholding} is quantified by the next lemma.

\begin{Lemma}\label{lm: noisy hard thresholding overall accuracy}
	Let $\tilde P_s$ be defined as in Algorithm $\ref{algo: noisy hard thresholding}$. For any index set $I$, any $\bm v \in \R^I$ and $\hat{\bm v}$ such that $\|\hat{\bm v}\|_0 \leq \hat s \leq s$, we have that for every $c > 0$, 
	\begin{align*}
	\|\tilde P_s(\bm v) - \bm v\|_2^2 \leq (1+1/c) \frac{|I|-s}{|I|-\hat s} \|\hat{\bm v} - \bm v\|_2^2 + 4(1 + c)\sum_{i \in [s]} \|\bm w_i\|^2_\infty.
	\end{align*}
\end{Lemma}
Lemma \ref{lm: noisy hard thresholding overall accuracy} is proved in Section \ref{sec: proof of noisy hard thresholding properties}. In comparison, the exact, non-private projection operator $P_s$ satisfies (\cite{jain2014iterative}, Lemma 1)
\begin{align*}
	\|P_s(\bm v) - \bm v\|_2^2 \leq \frac{|I|-s}{|I|-\hat s} \|\hat{\bm v} - \bm v\|_2^2.
\end{align*}
Algorithm \ref{algo: noisy hard thresholding}, therefore, is as accurate as its non-private counterpart up to a constant multiplicative factor and some additive noise. Taking the private top-$s$ projection algorithm, we have the following noisy iterative hard thresholding algorithm.

\vspace{2mm}

\begin{algorithm}[h]	\caption{Noisy Iterative Hard Thresholding (NoisyIHT)}
\label{algo: noisy iterative hard thresholding}
	\SetAlgoLined
	\SetKwInOut{Input}{Input}
	\SetKwInOut{Output}{Output}
	\SetKwFunction{NoisyHT}{NoisyHT}
	\Input{Objective function $\L_n(\bth, \bm Z) = n^{-1}\sum_{i=1}^n l(\bth, \bm z_i)$, data set $\bm Z$, sparsity level $s$, step size $\eta^0$, privacy parameters $\varepsilon, \delta$, noise scale $B$, number of iterations $T$.}
	Initialize $\bth^0$ with $\|\bth^0\|_0 \leq s$, set $t = 0$\;
	\For{$t$ in $0$ \KwTo $T-1$}{
		$\bth^{t+1} = \NoisyHT\left(\bth^t - \eta^0\nabla \L_n(\bth^t; \bm Z), \bm Z, s, \varepsilon/T, \delta/T, (\eta^0/n) B\right)$\; 
	}
	\Output{$\bth^{(T)}$.}
\end{algorithm}

\vspace{2mm}

Compared to the non-private Algorithm \ref{algo: iht}, we simply replaced the exact projection $P_s$ with the noisy projection given by Algorithm \ref{algo: noisy hard thresholding}. The privacy guarantee of Algorithm \ref{algo: noisy iterative hard thresholding} is then inherited from that of Algorithm \ref{algo: noisy hard thresholding}.

\begin{Lemma}\label{lm: noisy iterative hard thresholding privacy}
	If for every pair of adjacent data $\bm z, \bm z'$ and every $\bth \in \Theta$ we have $\| \nabla l(\bth; \bm z)- \nabla l(\bth; \bm z')\|_\infty < B$, then NoisyIHT is an $(\varepsilon, \delta)$-differentially private algorithm.
\end{Lemma}

The lemma is proved in Section \ref{sec: proof of lm: noisy iterative hard thresholding privacy}. Similar to the noisy gradient descent (Algorithm \ref{algo: private glm}), the privacy guarantee of Algorithm \ref{algo: noisy iterative hard thresholding} is valid for any choice of $T$, however a fast rate of convergence would allow us to select a small $T$ and thereby introducing less noise into the algorithm. To our delight, restricted strong convexity and restricted smoothness again lead to a linear rate of convergence even in the high-dimensional sparse setting.

\begin{Theorem}\label{thm: noisy iterative hard theresholding convergence}
	Let $\hat \bth = \argmin_{\|\bth\|_0 \leq s^*} \L_n(\bth; \bm Z)$. For iteration number $t \geq 0$, suppose
	\begin{align}
	& \langle \nabla \L_n(\bth^t) - \nabla \L_n(\hat \bth), \bth^t - \hat \bth \rangle \geq \alpha \|\bth^t - \hat \bth\|_2^2 \label{eq: rsc general} \\
	& \langle \nabla \L_n(\bth^{t+1}) - \nabla \L_n(\hat \bth), \bth^{t+1} - \hat \bth \rangle \leq \gamma \|\bth^{t+1} - \hat \bth\|_2^2. \label{eq: rsm general}
	\end{align}
	for constants $0 < \alpha < \gamma$. Let $\bm w_1, \bm w_2, \cdots, \bm w_s$ be the noise vectors added to $\bth^t - \eta^0\nabla \L_n(\bth^t; \bm Z)$ when the support of $\bth^{t+1}$ is iteratively selected, $S^{t+1}$ be the support of $\bth^{t+1}$, and $\tilde {\bm w}$ be the noise vector added to the selected $s$-sparse vector. Then, for $\eta_0 = {2}/{3\gamma}$, there exists an absolute constant $c_0$ so that, choosing $s \geq c_0(\gamma/\alpha)^2 s^*$ guarantees
	\begin{align*}
	\L_n(\bth^{t+1}) - \L_n(\hat \bth) &\leq \left(1- \rho \cdot \frac{\alpha}{\gamma} - \frac{2s^*}{s + s^*}\right)\left(\L_n(\bth^t) - \L_n(\hat \bth)\right) + C_\gamma\left(\sum_{i \in [s]} \|\bm w_i\|^2_\infty + \|\tilde {\bm w}_{S^{t+1}}\|_2^2\right),
	\end{align*}
	where $0 < \rho < 1$ is an absolute constant, and $C_\gamma > 0$ is a constant depending on $\gamma$. 
\end{Theorem}
Theorem \ref{thm: noisy iterative hard theresholding convergence} is proved in Section \ref{sec: proof of thm: noisy iterative hard thresholding convergence}. While conditions \eqref{eq: rsc general} and \eqref{eq: rsm general} are similar to the ordinary strong convexity and smoothness conditions in appearance, they are in fact much weaker because $\hat\bth$, $\bth^t$ are both $s$-sparse. It is unclear yet, however, whether these weaker conditions are satisfied by the GLM log-likelihood function, and whether the linear convergence in terms of $\L_n$ implies any positive result for parameter estimation accuracy $\|\bth^{(T)} - \hat\bth\|^2_2$. In the next section, we resolve these issues for high-dimensional sparse GLMs and obtain a parameter estimation accuracy result.  

\subsubsection{Noisy Iterative Hard Thresholding for High-Dimensional Sparse GLMs}\label{sec: noisy iht for sparse glm}
Assuming that the true GLM parameter vector $\bbeta^*$ satisfies $\|\bbeta^*\|_0 \leq s^*$, we now specialize the results of Section \ref{sec: general iht algo} to the GLM negative log-likelihood function
\begin{align*}
\L_n(\bbeta; \bm Z) = \frac{1}{n}\sum_{i=1}^n \left(\psi(\bm x_i^\top \bbeta) - g(y_i)\bm x_i^\top \bbeta\right).
\end{align*}
\begin{algorithm}[H]\label{algo: private sparse glm}
	\SetAlgoLined
	\SetKwInOut{Input}{Input}
	\SetKwInOut{Output}{Output}
	\SetKwFunction{NoisyHT}{NoisyHT}
	\Input{$\L_n(\bbeta, \bm Z)$, data set $\bm Z$, sparsity level $s$, step size $\eta^0$, privacy parameters $\varepsilon, \delta$, noise scale $B$, number of iterations $T$, truncation parameter $R$.}
	Initialize $\bbeta^0$ with $\|\bbeta^0\|_0 \leq s$, set $t = 0$\;
	\For{$t$ in $0$ \KwTo $T-1$}{
		Compute $\bbeta^{t + 0.5} = \bbeta^{(T)}- (\eta_0/n)\sum_{i=1}^n  (\psi'(\bm x_i^\top \bbeta^t)-\Pi_{R}(y_i))\bm x_i$\;
		$\bbeta^{t+1} = \NoisyHT\left(\bbeta^{t + 0.5}, \bm Z, s, \varepsilon/T, \delta/T, \eta^0 B/n\right)$\; 
	}
	\Output{$\bbeta^{(T)}$.}
	\caption{Differentially Private Sparse Generalized Linear Regression}
\end{algorithm}
\vspace{2mm}
Some assumptions about the data set $\{(y_i, \bm x_i)\}_{i \in [n]}$ and its distribution will be helpful for analyzing the accuracy and privacy guarantees of Algorithm \ref{algo: private sparse glm}. The necessary assumptions for the high-dimensional sparse case are identical to those for the low-dimensional case, except with (D1) replaced by (D1'), as follows.
\begin{itemize}
	\item [(D1')] Bounded design: there is a constant $\sigma_{\bm x}< \infty$ such that $\|\bm x\|_{\infty} < \sigma_{\bm x}$ almost surely.
\end{itemize}
Because Algorithm \ref{algo: private sparse glm} is a special case of the general Algorithm \ref{algo: noisy iterative hard thresholding}, the privacy guarantee of Algorithm \ref{algo: private sparse glm} reduces to specializing Lemma \ref{lm: noisy iterative hard thresholding privacy} to GLMs, as follows.
\begin{Lemma}\label{lm: glm privacy}
	If assumptions (D1') and (G1) are true, then choosing $B = 4(R + c_1)\sigma_{\bm x}$ guarantees that Algorithm \ref{algo: private sparse glm} is $(\varepsilon, \delta)$-differentially private.
\end{Lemma}
The lemma is proved in Section \ref{sec: proof of lm: glm privacy}.

For the parameter estimation accuracy of Algorithm \ref{algo: private sparse glm}, Fact \ref{lm: glm rsc and rsm} combined with the sparsity of $\hat\bbeta$, $\bbeta^*$ and $\bbeta^t$ for every $t$ are sufficient for conditions \eqref{eq: rsc general} and \eqref{eq: rsm general} in Theorem \ref{thm: noisy iterative hard theresholding convergence} to hold. Invoking Theorem \ref{thm: noisy iterative hard theresholding convergence} in a proof by induction then leads to an upper bound for $\|\bbeta^{(T)} - \bbeta^*\|^2_2$. Below we state the main result; the detailed proof is in Section \ref{sec: proof of thm: glm upper bound}.

\begin{Theorem}\label{thm: glm upper bound}
	Let $\{(y_i, \bm x_i)\}_{i \in [n]}$ be an i.i.d. sample from the GLM \eqref{eq: glm definition} with the true parameter vector $\|\bbeta^*\|_0 \leq s^*$. Suppose assumptions (D1'), (D2), (G1) and (G2) are true. Let the parameters of Algorithm $\ref{algo: private sparse glm}$ be chosen as follows.
	\begin{itemize}
		\item Set sparsity level $s = 4c_0(\gamma/\alpha)^2 s^*$ and step size $\eta^0 = 1/(2\gamma)$, where the constant $c_0$ is defined in Theorem \ref{thm: noisy iterative hard theresholding convergence} and constants $\alpha$, $\gamma$ are defined in Proposition \ref{lm: glm rsc and rsm}.
		\item Set $R = \min\left(\mathrm{ess}\sup |y_1|, c_1 + \sqrt{2c_2c(\sigma)\log n}\right) \lesssim \sqrt{c(\sigma) \log n}.$
		\item Noise scale $B$. Set $B = 4(R + c_1)\sigma_{\bm x}$.
		\item Number of iterations $T$. Let $T = (2\gamma/\rho\alpha)\log(6\gamma n)$, where $\rho$ is an absolute constant defined in Theorem 1.1.
		\item Initialization $\bbeta^0$. Choose $\bbeta^0$ so that  $\|\bbeta^0\|_0 \leq s$ and $\|\bbeta^0 - \hat\bbeta\|_2 \leq 3$, where $\hat\bbeta$ $= \argmin_{\|\bbeta\|_0 \leq s^*} \L_n(\bbeta; Z)$.
	\end{itemize}
	If $n \geq K \cdot \left(Rs^*\log d \sqrt{\log(1/\delta)}\log n/\varepsilon\right)$ for a sufficiently large constant $K$, it holds with probability at least $1 - c_3\exp(-c_4\log(d/s^*\log n)) - c_3\exp(-c_4n) - c_3\exp(-c_4\log n)$ that $\bbeta^{(T)}$, the output of Algorithm \ref{algo: private sparse glm} satisfies
	\begin{align}
	\|\bbeta^{(T)} - \bbeta^*\|^2_2 \lesssim c(\sigma)\left(\frac{s^*\log d}{n} + \frac{(s^*\log d)^2 \log(1/\delta)\log^3 n}{n^2\varepsilon^2}\right). \label{eq: glm upper bound}
	\end{align}
\end{Theorem}
Theorem \ref{thm: glm upper bound} is proved in Section \ref{sec: proof of thm: glm upper bound}. Similar to the low-dimensional GLM algorithm, the step size, number of iterations and initialization are chosen to ensure convergence; the initialization condition, as in \cite{loh2015regularized}, is standard in the literature and can be extended to $\|\bbeta^0 - \hat\bbeta\|_2 \leq 3\max(1, \|\bbeta^*\|_2)$. 

The choice of truncation level $R$ is to ensure privacy while keeping as many data intact as possible; when the distribution of $y$ has bounded support, for example in the logistic model, it can be chosen to be an $O(1)$ constant and therby saving an extra factor of $O(\log n)$ in the second term of \eqref{eq: non-sparse glm upper bound}. The scaling of $n$ versus $d, s^*, \varepsilon$ and $\delta$ in Theorem \ref{thm: glm upper bound} is nearly optimal, as the corresponding lower bound, Theorem \ref{thm: high-dim glm lb}, shall show that no estimator can achieve low $\ell_2$ error unless the assumed scaling holds, and that the statistical accuracy of Algorithm \ref{algo: private sparse glm} cannot be further improved except possibly for factors of $\log n$.

%% file: 4_Lower_bounds.tex
%!TEX root = Privacy-GLM-JRSSB.tex
%%%%%%%%%%%%%%%%%%%%%%%%%%%%%%%%%%%%%%%%%%%%%%%%%%%%%%%%%%%%%%%%%%%
\section{Privacy-constrained Minimax Lower Bounds}\label{sec: glm lower bounds}

Section \ref{sec: glm upper bounds} proposed differentially private algorithms for estimating GLM parameters and obtained convergence rates for these algorithms. We shall show in this section that the convergence rates cannot be improved by any other $(\varepsilon, \delta)$-differentially private estimator beyond possibly factors of $\log n$, via privacy-constrained lower bounds of the form
\begin{align}\label{eq: lb general form}
\inf_{M \in \mathcal M_{\varepsilon, \delta}} \sup_{\bbeta \in \Theta} \E\|M(\bm y, \bm X) - \bbeta\|^2_2 \gtrsim r(n, d, \Theta, \sigma, \varepsilon, \delta),
\end{align}
where $\mathcal M_{\varepsilon, \delta}$ is the collection of all $(\varepsilon, \delta)$-differentially private estimators, $\Theta \subseteq \R^d$ is a parameter space to which the true value of $\bbeta$ is assumed to belong, and the expectation is taken over $\bm y, \bm X$ and the randomness of $M$. 

We shall provide precise forms of the lower bound $r(n, d, \Theta, \sigma, \varepsilon, \delta)$ for both the low-dimensional and high-dimensional sparse GLMs, via a broad generalization of the ``tracing attack'' argument \cite{bun2014fingerprinting, dwork2015robust, dwork2017exposed, steinke2017between} for privacy-constrained minimax lower bounds.

A tracing attack is an algorithm that takes a single candidate datum as input and attempts to infer whether this candidate belongs to a given data set or not, by comparing the candidate with some summary statistics computed from the data set. Statisticians can think of a tracing attack as a hypothesis test which rejects the null hypothesis that the candidate is out of the data set for large values of some test statistic. The hypothesis testing formulation naturally motivates some desiderata for a tracing attack: 
\begin{itemize}
	\item Soundness (type I error control): if the candidate does not belong to the data set, the tracing attack is likely to takes small values. 
	\item Completeness (type II error control): if the candidate does belong, the tracing attack is likely to take large values.
\end{itemize}
For example, \cite{dwork2015robust, kamath2018privately, cai2019cost} showed that, if the random sample $\bm X$  and the candidate $\bm z$ are drawn from a Gaussian distribution with mean $\bmu$ , tracing attacks of the form $\langle M(\bm X) - \bmu, \bm z - \bmu \rangle$ is sound and complete provided that $M(\bm X)$ is an accurate estimator of $\bmu$. This accuracy requirement in turn connects tracing attacks with risk lower bounds for differentially private algorithms: if an estimator $M(\bm X)$ is differentially private, it cannot possibly be too close to the estimand, or the existence of tracing attacks leads to a contradiction with the guarantees of differential privacy.

Designing sound and complete tracing attacks, therefore, is crucial to the sharpness of privacy-constrained minimax lower bounds. Besides the Gaussian mean tracing attack mentioned above, there are some successful tracing attacks proposed for specific problems, such as top-$k$ selection \cite{steinke2017tight} or linear regression \cite{cai2019cost}, but a general recipe for the design and analysis of tracing attacks has not been available. In Section \ref{sec: general lb}, we construct a tracing attack applicable to general parametric families of distributions, and describe its utility for privacy-constrained minimax lower bounds. This general approach is then specialized to low-dimensional and high-dimensional sparse GLMs, in Sections \ref{sec: non-sparse glm lower bound} and \ref{sec: sparse glm lower bound} respectively, to establish lower bound results that match the upper bound results in Section \ref{sec: glm upper bounds} up to factors of $\log n$.

\subsection{The Score Attack}\label{sec: general lb}
Given a parametric family of distributions $\{f_\bth(\bm x): \bth \in \Theta\}$ with $\Theta \subseteq \R^d$, the score statistics, or simply the score, is given by $S_\bth(\bm x) := \nabla_\bth \log f_\bth(\bm x)$. If $\bm x \sim f_\bth$, we have $\E S_\bth(\bm x) = \bm 0$ and $\Var S_\bth(\bm x) = \mathcal I(\bth)$, where $\mathcal I(\bth)$ is the Fisher information matrix of $f_\bth$.

Using the score statistic, we define the score attack as 
\begin{align}\label{eq: score attack}
	\A_\bth(\bm z, M(\bm X)) := \langle M(\bm X) - \bth, S_\bth(\bm z)   \rangle.
\end{align}
The score attack conjectures that $\bm z$ belongs to $\bm X$ for large values of $\A_\bth(\bm z, M(\bm X))$. In particular, if $f_\bth(\bm x)$ is the density of $N(\bth, \bm I)$, the score attack coincides with the tracing attacks for Gaussian means studied in \citep{dwork2015robust, kamath2018privately, cai2019cost}.

As argued earlier, an effective tracing attack should ideally be ``sound'' (low type I error) and ``complete'' (low Type II error). This is indeed the case for our score attack. 
\begin{Theorem}\label{thm: score attack general}
Let $\bm X = \{\bm x_1, \bm x_2, \cdots, \bm x_n\}$ be an i.i.d. sample drawn from $f_\bth$. For each $i \in [n]$, let $\bm X'_{i}$ denote the data set obtained from $\bm X$ by replacing $\bm x_i$ with an independent copy $\bm x'_i \sim f_\bth$.
\begin{enumerate}
	\item Soundness: for each $i \in [n]$, 
	\begin{align}\label{eq: soundness general}
	\E \A_\bth(\bm x_i, M(\bm X'_i)) = 0; ~ \E |\A_\bth(\bm x_i, M(\bm X'_i))| \leq \sqrt{\E\|M(\bm X) - \bth\|_2^2}\sqrt{\lambda_{\max}(\mathcal I(\bth))}.
	\end{align}
	\item Completeness: if for every $j \in [d]$, $\log f_\bth(\bm X)$ is continuously differentiable with respect to $\theta_j$ and $|\frac{\partial}{\partial\theta_j}\log f_\bth(\bm X)| < g_j(X)$ such that $\E |g_j(\bm X) M(\bm X)_j | < \infty$, we have 
	\begin{align}\label{eq: completeness general}
		\sum_{i \in [n]} \E \A_\bth (\bm x_i, M(\bm X)) = \sum_{j \in [d]} \frac{\partial}{\partial \theta_j} \E M(\bm X)_j.
	\end{align}
\end{enumerate} 
\end{Theorem}
Theorem \ref{thm: score attack general} is proved in Section \ref{sec: proof of thm: score attack general}. The special form of ``completeness'' for Gaussian and Beta-Binomial families have been discovered as ``fingerprinting lemma'' in the literature \citep{tardos2008optimal, bun2014fingerprinting, steinke2017tight, kamath2018privately}.

It may not be clear yet how the soundness and completeness properties would imply lower bounds for $\E\|M(\bm X) - \bth\|_2^2$. For the specific attacks designed for Gaussian mean estimation \citep{kamath2018privately} and top-$k$ selection \citep{steinke2017tight}, it has been observed that, if $M$ is an $(\varepsilon, \delta)$-differentially private algorithm, one can prove inequalities of the form $\E \A_\bth (\bm x_i, M(\bm X)) \leq \E \A_\bth (\bm x_i, M(\bm X'_i)) + O(\varepsilon) \E |\A_\bth (\bm x_i, M(\bm X'_i))|$. Suppose such relations hold for the score attack as well, the soundness property \eqref{eq: soundness general} would then imply
\begin{align*}
	\sum_{i \in [n]} \E \A_\bth (\bm x_i, M(\bm X)) \leq \sqrt{\E\|M(\bm X) - \bth\|_2^2} \cdot n\sqrt{\lambda_{\max}(\mathcal I(\bth))} O(\varepsilon).
\end{align*}
We give precise statement of such an inequality in Section \ref{sec: attack upper bound general}.

On the other hand, if we can also bound $\sum_{i \in [n]} \E \A_\bth (\bm x_i, M(\bm X))$ from below by some positive quantity, a lower bound for $\E\|M(\bm X) - \bth\|_2^2$ is immediately implied. Completeness may help us in this regard: when $\E M(\bm X)_j$ is close to $\theta_j$, it is reasonable to expect that $\frac{\partial}{\partial \theta_j} \E M(\bm X)_j$ is bounded away from zero. Indeed several versions of this argument, often termed ``strong distribution'', exist in the literature \citep{dwork2015robust, steinke2017between} and have led to lower bounds for Gaussian mean estimation and top-$k$ selection. In Section \ref{sec: attack lower bound general}, we consider a systematic approach to lower bounding $\frac{\partial}{\partial \theta_j} \E M(\bm X)_j$ via Stein's Lemma \cite{stein1972bound, stein2004use}. The technical results in Sections \ref{sec: attack upper bound general} and \ref{sec: attack lower bound general} combined with Theorem \ref{thm: score attack general} would enable us to later prove concrete minimax lower bounds for GLMs.

\subsubsection{Score Attacks and Differential Privacy}\label{sec: attack upper bound general}
In Theorem \ref{thm: score attack general}, we have found that, when the data set $\bm X'_i$ does not include $\bm x_i$, the score attack is unlikely to take large values:
\begin{align*}
\E \A_\bth(\bm x_i, M(\bm X'_i)) = 0; ~ \E |\A_\bth(\bm x_i, M(\bm X'_i))| \leq \sqrt{\E\|M(\bm X) - \bth\|_2^2}\sqrt{\lambda_{\max}(\mathcal I(\bth))}.
\end{align*}
If $M$ is differentially private, the distribution of $M(\bm X'_i)$ is close to that of $M(\bm X)$; as a result, the inequalities above can be related to the case where the data set $\bm X$ does include the candidate $\bm x_i$.
\begin{Lemma}\label{lm: score attack upper bound}
	If $M$ is an $(\varepsilon, \delta)$-differentially private algorithm with $0 < \varepsilon < 1$ and $\delta \geq 0$,  then for every $T > 0$, 
	\begin{align}\label{eq: score attack upper bound}
		\E \A_\bth(\bm x_i, M(\bm X)) \leq 2\varepsilon \sqrt{\E\|M(\bm X) - \bth\|_2^2}\sqrt{\lambda_{\max}(\mathcal I(\bth))} + 2\delta T + \int_T^\infty \Pro\left(|\A_\bth(\bm x_i, M(\bm X))| > t \right).
	\end{align}
\end{Lemma}
Lemma \ref{lm: score attack upper bound} is proved in Section \ref{sec: proof of lm: score attack upper bound}. The quantity on the right side of \eqref{eq: score attack upper bound} is determined by the statistical model $f_\bth(\bm x)$ and the choice of $T$. In Sections \ref{sec: non-sparse glm lower bound} and \ref{sec: sparse glm lower bound}, we work out its specific forms for low-dimensional and high-dimensional sparse GLMs.

\subsubsection{Score Attacks and Stein's Lemma}\label{sec: attack lower bound general}
Let us denote $\E_{X|\bth} M(\bm X)$ by $g(\bth)$, then $g$ is a map from $\Theta$ to $ \Theta$, and we are interested in bounding $\frac{\partial}{\partial \theta_j} g_j(\bth)$ from below. Stein's Lemma \cite{stein1972bound, stein2004use}, as stated below, suggests some promising directions.

\begin{Lemma}[Stein's Lemma]\label{lm: stein's lemma}
	Let $Z$ be distributed according to some density $p(z)$ that is continuously differentiable with respect to $z$ and let $h: \R \to \R$ be a differentiable function such that $\E |h'(Z)| < \infty$. We have
	\begin{align*}
		\E h'(Z) = \E\left[\frac{-h(Z)p'(Z)}{p(Z)}\right].
	\end{align*}
In particular, if $p(z) = (2\pi)^{-1/2}e^{-z^2/2}$, we have $\E h'(Z) = \E Z h(Z)$.
\end{Lemma}
Stein's Lemma implies that, by imposing appropriate prior distributions on $\bth$, one can obtain a lower bound for $\frac{\partial}{\partial\theta_j} g_j(\bth)$ on average over the prior distribution of $\bth$, as follows.
\begin{Lemma}\label{lm: score attack stein's lemma}
	Let $\bth$ be distributed according to a density $\bm \pi$ with marginal densities $\{\pi_j\}_{j \in [d]}$. If for every $j \in [d]$, $\pi_j, g_j$ satisfy the regularity conditions in Lemma \ref{lm: stein's lemma}, we have
	\begin{align}\label{eq: score attack stein's lemma}
		\E_{\bm \pi} \left(\sum_{j \in [d]} \frac{\partial}{\partial \theta_j} g_j(\bth)\right) \geq \E_{\bm \pi} \left(\sum_{j \in [d]} \frac{-\theta_j \pi'_j(\theta_j)}{\pi_j(\theta_j)}\right) - \sqrt{\E_{\bm \pi}\|g(\bth) - \bth\|^2_2 \cdot \E_{\bm \pi} \left[\sum_{j \in [d]}\left(\frac{\pi'_j(\theta_j)}{\pi_j(\theta_j)}\right)^2\right]}
	\end{align}
\end{Lemma}
Lemma \ref{lm: score attack stein's lemma} is proved in Section \ref{sec: proof of lm: score attack stein's lemma}. Often we may assume without the loss of generality that $\E_{\bm \pi}\|g(\bth) - \bth\|^2_2 \leq \E_{\bm \pi} \E_{\bm X|\bth} \|M(\bm X) - \bth\|_2^2 < C$ for some constant $C$ when the sample size $n$ is sufficiently large, the right side is completely determined by the choice of $\pi$, as the following example illustrates:
\begin{example}\label{ex: Gaussian stein's lemma}
	Let $\bm \pi$ be the density of $N(\bm 0, \bm I)$, then \eqref{eq: score attack stein's lemma} reduces to
	\begin{align*}
		\E_{\bm \pi} \left(\sum_{j \in [d]} \frac{\partial}{\partial \theta_j} g_j(\bth)\right) \geq \sum_{j \in [d]} \E_{\pi_j} \theta_j^2 - \sqrt{C} \sqrt{\sum_{j \in [d]} \E_{\pi_j} \theta_j^2} = d - \sqrt{C d} \gtrsim d. 
	\end{align*}
\end{example}

In view of the completeness property \eqref{eq: completeness general}, Lemma \ref{lm: score attack stein's lemma} suggests an \textit{average} lower bound for $\sum_{i \in [n]} \E \A_\bth (\bm x_i, M(\bm X))$ over some prior distribution $\bm \pi(\bth)$, with the specific form of this average lower bound entirely determined by the choice of $\bm \pi$.

\subsubsection{From Score Attacks to Lower Bounds}\label{sec: attack to lower bounds}
Let us combine Theorem \ref{thm: score attack general} with Lemmas \ref{lm: score attack upper bound} and \ref{lm: score attack stein's lemma} to understand how the score attack leads to privacy-constrained minimax lower bounds.

Let $\bm \pi$ be a prior distribution supported over the parameter space $\Theta$ with marginal densities $\{\pi_j\}_{j \in [d]}$, and assume without the loss of generality that $\E_{\bm X|\bth} \|M(\bm X) - \bth\|_2^2 < C$ for every $\bth \in \Theta$. The completeness part of Theorem \ref{thm: score attack general} and Lemma \ref{lm: score attack stein's lemma} imply that
\begin{align*}
	\sum_{i \in [n]} \E_{\bm \pi}\E_{\bm X|\bth} \A_\bth (\bm x_i, M(\bm X)) \geq \E_{\bm \pi} \left(\sum_{j \in [d]} \frac{-\theta_j \pi'_j(\theta_j)}{\pi_j(\theta_j)}\right) - \sqrt{C}\sqrt{  \E_{\bm \pi} \left[\sum_{j \in [d]}\left(\frac{\pi'_j(\theta_j)}{\pi_j(\theta_j)}\right)^2\right]}
\end{align*}
Since Lemma \ref{lm: score attack upper bound} holds for every $\bth$, it follows from the Lemma that
\begin{align*}
		&\sum_{i \in [n]} \E_{\bm \pi}\E_{\bm X|\bth} \A_\bth (\bm x_i, M(\bm X)) \\
		&\leq 2n\varepsilon \sqrt{\E_{\bm \pi} \E_{\bm X|\bth}\|M(\bm X) - \bth\|_2^2}\sqrt{\lambda_{\max}(\mathcal I(\bth))} + 2n\delta T + \sum_{i \in [n]}\int_T^\infty \Pro\left(|\A_\bth(\bm x_i, M(\bm X))| > t \right).
\end{align*}
These two inequalities are true for every $(\varepsilon, \delta)$-differentially private $M$, and they therefore suggest a lower bound for $\inf_{M \in \mathcal M_{\varepsilon, \delta}} \E_{\bm \pi} \E_{\bm X|\bth}\|M(\bm X) - \bth\|_2^2$, which in turn lower bounds $\inf_{M \in \mathcal M_{\varepsilon, \delta}} \sup_{\bth \in \Theta}$ $\E_{\bm X|\bth}\|M(\bm X) - \bth\|_2^2$ since the maximum risk is greater than the average risk regardless of the prior distribution.

Following this strategy, we shall obtain the privacy-constrained minimax lower bounds for GLM problems, by choosing an appropriate prior distribution $\bm \pi$ and working out the specific forms of the two inequalities \eqref{eq: score attack upper bound} and \eqref{eq: score attack stein's lemma} in the context of GLMs.

\subsection{The Classical  Low-dimensional Setting}\label{sec: non-sparse glm lower bound}
We first consider the low-dimensional $d = o(n)$ setting.  For the generalized linear model
\begin{align*}
f_{\bbeta} (y|\bm x) = h(y, \sigma)\exp\left(\frac{\bm x^\top \bbeta y - \psi(\bm x^\top \bbeta)}{c(\sigma)}\right); \bm x \sim f_{\bm x}, 
\end{align*}
and a candidate datum $(\tilde y, \tilde{\bm x})$, the score attack, as defined by \eqref{eq: score attack}, takes the form
\begin{align}\label{eq: low-dim glm attack}
\A_{\bbeta} ((\tilde y, \tilde{\bm x}), M(\bm y, \bm X)) = \frac{1}{c(\sigma)} \big\langle M(\bm y, \bm X) - \bbeta, [\tilde y - \psi'(\tilde{\bm x}^\top\bbeta)] \tilde {\bm x} \big\rangle.
\end{align}
For the prior distribution of $\bbeta$, we choose $\bm \pi(\bbeta)$ to be the density of $N(\bm 0, \bm I)$. The strategy outlined in Section \ref{sec: general lb} implies the following lower bound result.
\begin{Theorem}\label{thm: low-dim glm lb}
	Consider i.i.d. observations $(y_1, \bm x_1), \cdots, (y_n, \bm x_n) \in \R \times \R^d$, where $\bm x \sim f_{\bm x}$ such that $\E(\bm x \bm x^\top)$ is diagonal with $0 < \lambda_{\max}(\E(\bm x \bm x^\top)) < C < \infty$, $\|\bm x\|_2 \lesssim \sqrt{d}$ almost surely, and $y$ given $\bm x$ follows the conditional distribution
	\begin{align*}
	f_{\bbeta}(y|\bm x) = h(y, \sigma)\exp\left(\frac{\bm x^\top\bbeta y - \psi(x^\top\bbeta)}{c(\sigma)}\right).
	\end{align*}
	If $0 < \|\psi^{''}\|_\infty < c_2 < \infty$, $0 < \varepsilon < 1$, $0 < \delta < n^{-(1+\gamma)}$ for some $\gamma > 0$, then for sufficiently large $n$ and every $(\varepsilon, \delta)$-differentially private $M$ such that $\|M(\bm y, \bm X) - \bbeta\|^2_2 \lesssim d$ and $\E\|M(\bm y, \bm X) - \bbeta\|^2_2 = o(1)$,
	\begin{align}\label{eq: low-dim glm lb}
	\sup_{\bbeta \in \R^d} \E\|M(\bm y, \bm X) - \bbeta\|^2_2 \gtrsim c(\sigma)\frac{d^2}{n^2\varepsilon^2}.
	\end{align}
\end{Theorem}
Theorem \ref{thm: low-dim glm lb} is proved in Section \ref{sec: proof of thm: low-dim glm lb}. The $(\varepsilon, \delta)$-differentially private estimators $\mathcal M_{\varepsilon, \delta}$ are also subject to the non-private minimax risks lower bound for GLMs, $\inf_{M}\sup_{\bbeta} \E\|M(\bm y, \bm X)- \bbeta\|_2^2 \gtrsim c(\sigma) d/n$. It then follows from \eqref{eq: low-dim glm lb} that
\begin{align*}
	\inf_{M \in \mathcal M_{\varepsilon, \delta}}\sup_{\bbeta} \E\|M(\bm y, \bm X)- \bbeta\|_2^2 \gtrsim c(\sigma) \left(\frac{d}{n} + \frac{d^2}{n^2\varepsilon^2}\right).
\end{align*}
The lower bound matches the statistical accuracy of noisy gradient descent, Theorem \ref{thm: non-sparse glm upper bound}, up to factors of $\log n$ under the usual setting of $\delta = n^{-\alpha}$ for some constant $\alpha > 1$. Besides showing the optimality of noisy gradient descent, this comparison also suggests that the cost of privacy, as measured by the squared $\ell_2$-norm, in GLM parameter estimation is of the order $d^2/n^2\varepsilon^2$.

\subsection{The High-Dimensional Sparse Setting}\label{sec: sparse glm lower bound}
We now consider the setting where $d$, the dimension of $\Theta$, dominates the sample size $n$, but each $\bbeta \in \Theta$ is assumed to be $s^*$-sparse, that is $\|\bbeta\|_0 \leq s^*$. As seen in the following theorem, the sparsity assumption leads to a lower bound that depends primarily on the sparsity, or the ``intrinsic dimension'' of $\bbeta$, and only logarithmically on the ambient dimension $d$.

For high-dimensional sparse GLMs, we consider a modification of the classical GLM score attack \eqref{eq: low-dim glm attack}, the sparse GLM score attack:
\begin{align}\label{eq: high-dim glm attack}
\A_{\bbeta,s^*}((\tilde y, \tilde{\bm x}), M(\bm y, \bm X)) = \frac{1}{c(\sigma)} \big\langle (M(\bm y, \bm X) - \bbeta)_{\supp(\bbeta)}, [\tilde y - \psi'(\tilde{\bm x}^\top \bbeta)] \tilde {\bm x} \big\rangle.
\end{align}

For the prior $\bm \pi$, we have to choose some distribution supported over the set $\{\bbeta: \bbeta \in \R^d, \|\bbeta\|_0 \leq s^*\}$. Specifically, we consider $\bbeta$ generated as follows: let $\tilde \beta_1, \tilde \beta_2, \cdots, \tilde \beta_d$ be drawn i.i.d. from $N(0, 1)$, let $I_{s^*}$ be be the index set of $\tilde \bbeta$ with top $s^*$ greatest absolute values so that $|I_{s^*}| = s^*$ by definition, and define $\beta_j = \tilde \beta_j \1(j \in I_{s^*})$.

The score attack strategy then leads to the following lower bound result.
\begin{Theorem}\label{thm: high-dim glm lb}
	Consider $n$ i.i.d. observations $(y_1, \bm x_1), \cdots, (y_n, \bm x_n)$, where $\bm x \sim f_{\bm x}$ such that $\E(\bm x \bm x^\top)$ is diagonal with $0 < \lambda_{\max}(\E(\bm x \bm x^\top)) < C < \infty$, $\|\bm x\|_\infty < c < \infty$ almost surely, and $y$ given $\bm x$ follows the conditional distribution
	\begin{align*}
	f_{\bbeta}(y|\bm x) = h(y, \sigma)\exp\left(\frac{\bm x^\top\bbeta y - \psi(\bm x^\top\bbeta)}{c(\sigma)}\right).
	\end{align*}
	If $0 < \|\psi^{''}\|_\infty = c_2 < \infty$, $0 < \varepsilon < 1$, $0 < \delta < n^{-(1+\gamma)}$ for some $\gamma > 0$, $s = o\left(d^{1-\gamma}\right)$ for some $\gamma > 0$, then for sufficiently large $n$ and every $(\varepsilon, \delta)$-differentially private $M$ such that $\|M(\bm y, \bm X) - \bbeta\|^2_2 \lesssim s^*$ and $\E\|M(\bm y, \bm X) - \bbeta\|^2_2 = o(1)$,
	\begin{align}\label{eq: high-dim glm lb}
	 \sup_{\bbeta \in \R^d, \|\bbeta\|_0 \leq s^*} \E\|M(\bm y, \bm X) - \bbeta\|^2_2 \gtrsim c(\sigma)\frac{(s^*\log d)^2}{n^2\varepsilon^2}.
	\end{align}
\end{Theorem}
Theorem \ref{thm: high-dim glm lb} is proved in Section \ref{sec: proof of thm: high-dim glm lb}. In conjunction with the non-private minimax lower bound $\inf_{M} \sup_{\bbeta \in \R^d, \|\bbeta\|_0 \leq s} \E\|M(\bm y, \bm X) - \bbeta\|^2_2 \gtrsim c(\sigma)s^*\log d/n$, \eqref{eq: high-dim glm lb} implies 
\begin{align*}
	\inf_{M \in \mathcal M_{\varepsilon, \delta}} \sup_{\bbeta \in \R^d, \|\bbeta\|_0 \leq s} \E\|M(\bm y, \bm X) - \bbeta\|^2_2 \gtrsim c(\sigma)\left(\frac{s^*\log d}{n} + \frac{(s^*\log d)^2}{n^2\varepsilon^2}\right).
\end{align*}
By comparing the privacy-constrained minimax lower bound with Theorem \ref{thm: glm upper bound}, we can see that the noisy iterative hard thresholding algorithm for sparse GLMs is optimal up to factors of $\log n$ under the usual setting of $\delta = n^{-\alpha}$, and that the cost of privacy, as measured by squared $\ell_2$ norm, in sparse GLM parameter estimation is of the order $(s^*\log d)^2/n^2\varepsilon^2$. 

%% file: 5_Experiments.tex
%!TEX root = Privacy-GLM-JRSSB.tex
%%%%%%%%%%%%%%%%%%%%%%%%%%%%%%%%%%%%%%%%%%%%%%%%%%%%%%%%%%%%%%%%%%%
\section{Numerical Results}\label{sec: experiments}
In this section, we investigate the numerical performance of the proposed  privacy-preserving algorithms by conducting experiments with both simulated and real data sets. The numerical results also illustrate our theoretical findings on differentially private GLM parameter estimation.

\subsection{Simulated Data}\label{sec: simulated data}

For the low-dimensional GLM, our simulated data set is constructed as follows. For our desired choice of $d$ and $n$, we sample $\bbeta$ uniformly at random from the unit sphere in $\R^d$, draw coordinates of the design vector $\bm x_i$ independently from the uniform distribution over $(-1, 1)$ for each $i \in [n]$, and sample $y_i$ from the logistic regression model, that is $y_i$ following the Bernoulli distribution with success probability $\frac{1}{1+\exp(-\bm x_i^\top \bbeta)}$. Using the simulated data, we study the numerical performance of Algorithm \ref{algo: private glm} via three sets of experiments. In each experiment, the algorithm is initialized with $\bbeta = \bm 0 \in \R^d$, with step size $\eta^0 = 1$ for each iteration.
\begin{enumerate}
	\item [(a).] Fix $n = 40000, \varepsilon = 0.5$ and $\delta = (2n)^{-1}$, and compare the iterates of Algorithm \ref{algo: private glm} with the true $\bbeta$ for $d = 10, 20$, or $40$. As displayed in Figure 1(a), the log error $\log(\|\bbeta^t - \bbeta\|_2^2)$ is linear in $t$ when $d = 10$ but deteriorates as $d$ increases, confirming the theoretical result in Theorem \ref{thm: non-sparse glm upper bound}.
	\item [(b).] Fix $d = 20, \varepsilon = 0.5$ and $\delta = (2n)^{-1}$, and compare the iterates of Algorithm \ref{algo: private glm} with the true $\bbeta$ for $n = 20000, 40000$, or $80000$. As predicted by Theorem \ref{thm: non-sparse glm upper bound}, $\log(\|\bbeta^t - \bbeta\|_2^2)$ is linear in $t$ when $n = 80000$ but deteriorates as $n$ decreases.
	\item [(c).] Fix $d = 20, n = 40000$ and $\delta = (2n)^{-1}$, and compare the iterates of Algorithm \ref{algo: private glm} with the true $\bbeta$ for $\varepsilon = 0.2, 0.5, 0.8$, or $\infty$ (non-private). The decrease in $\log(\|\bbeta^t - \bbeta\|_2^2)$ as $\varepsilon$ increases is consistent with Theorem \ref{thm: non-sparse glm upper bound}.
\end{enumerate}
\begin{figure}[H]
		\subfloat[]{\includegraphics[width= 0.33\textwidth]{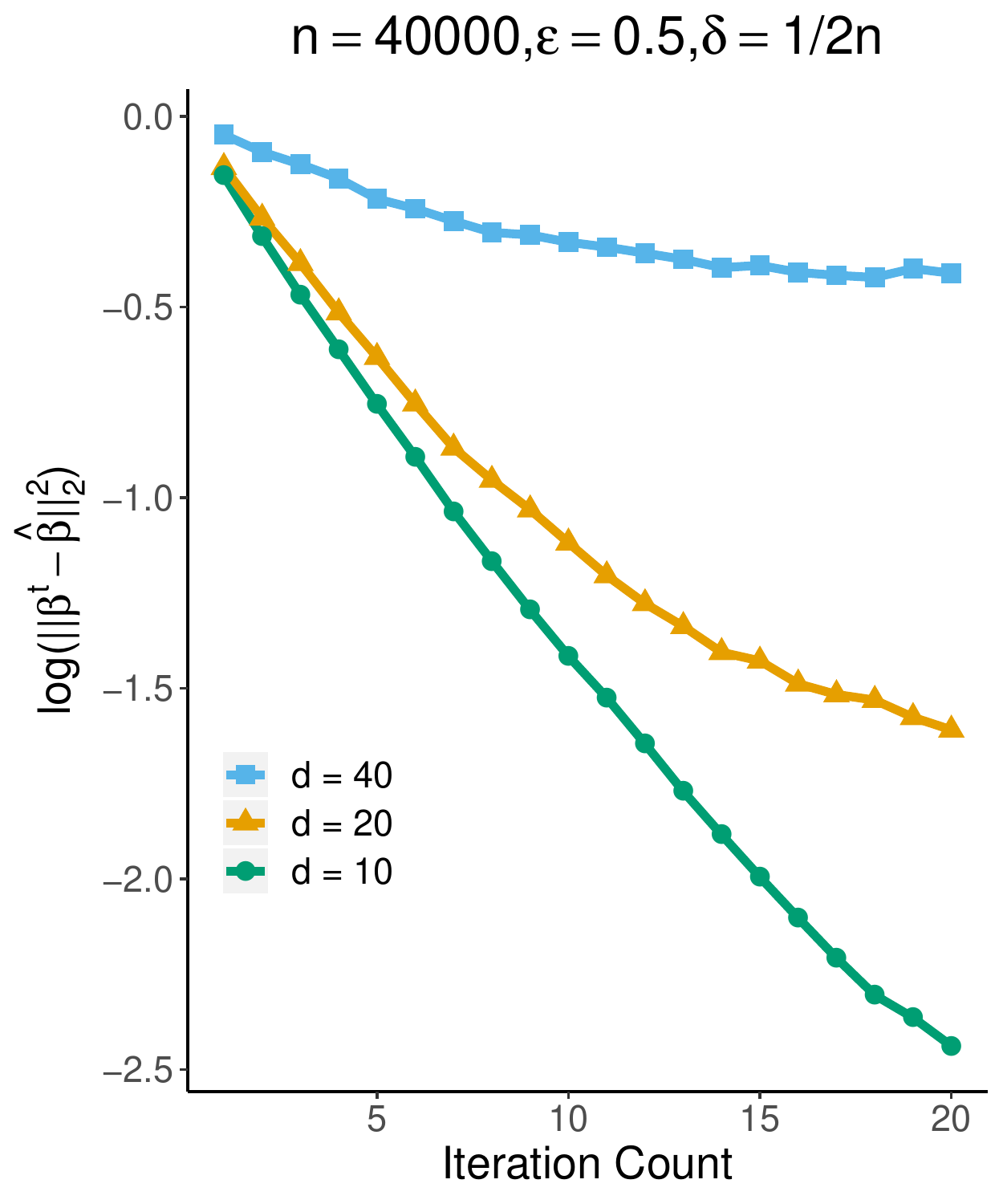}}
		\subfloat[]{\includegraphics[width= 0.33\textwidth]{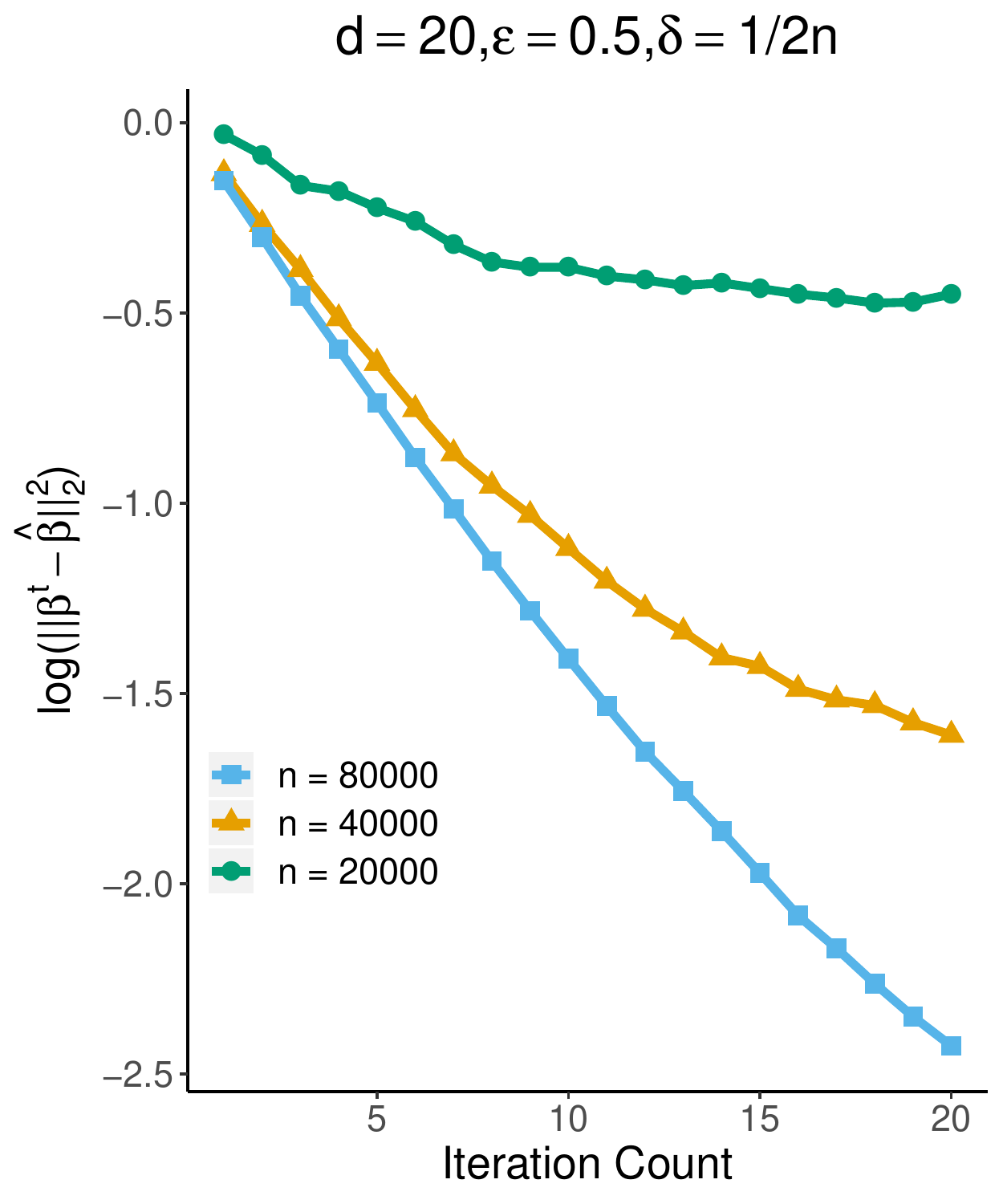}}
		\subfloat[]{\includegraphics[width= 0.33\textwidth]{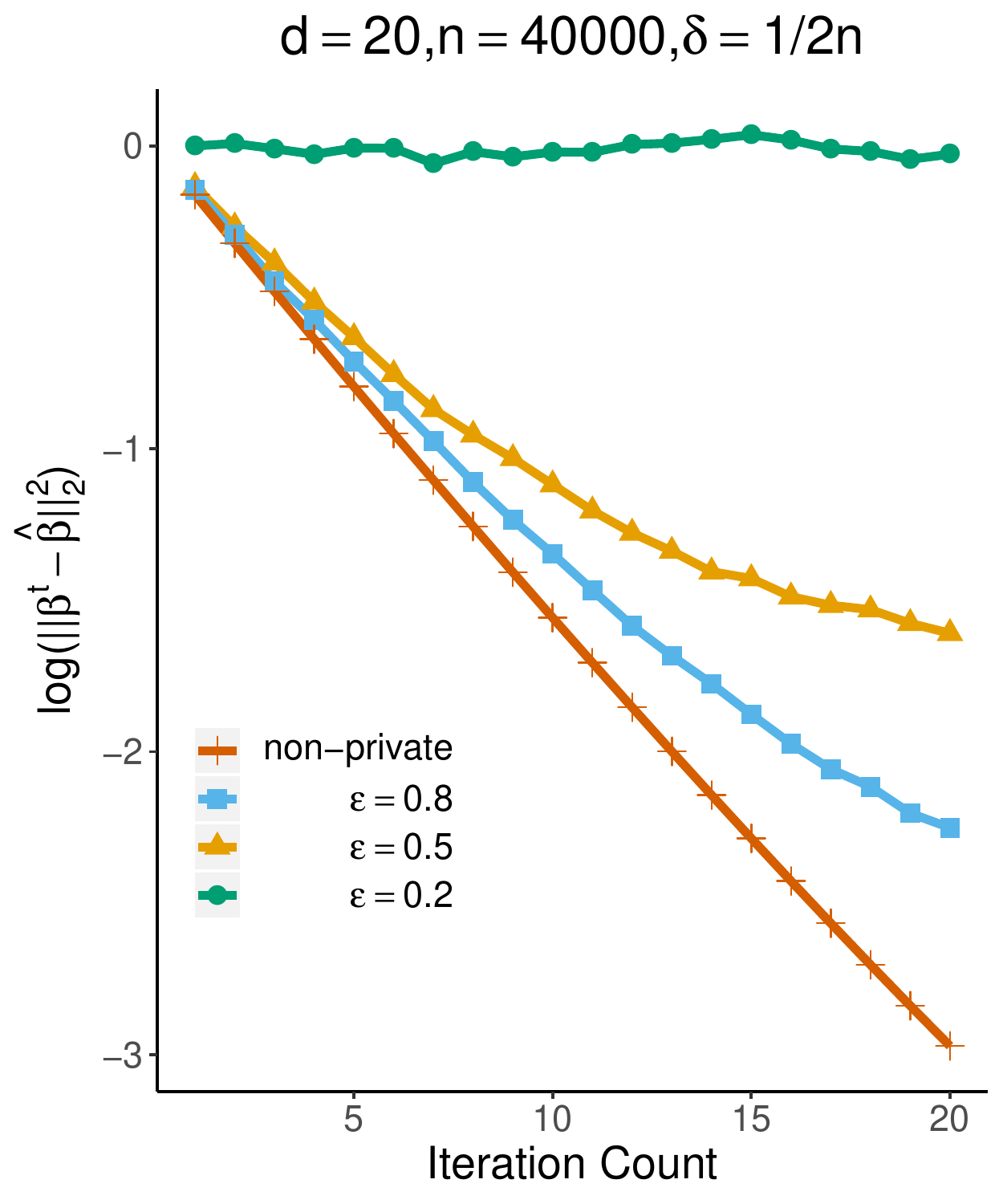}}
		\caption{Log-distance between the iterates of Algorithm \ref{algo: private glm} and the true parameter $\bbeta$ under various settings of $n, d, \varepsilon$ and $\delta$.}
\end{figure}

For the high-dimensional sparse GLM, the simulated data set is constructed in the identical way as the low-dimensional case, except that the $s$-sparse true parameter $\bbeta$ is obtained by concatenating a random draw from the unit sphere in $\R^s$ with $\bm 0 \in \R^{d-s}$. We have three sets of experiments to study the numerical performance of Algorithm \ref{algo: private sparse glm}. In each experiment, the algorithm is initialized with $\bbeta = \bm 0 \in \R^d$, with step size $\eta^0 = 1$ for each iteration and the sparsity level set at twice of the true sparsity.
\begin{enumerate}
	\item [(a).] Fix $d = 10000, n = 40000, \varepsilon = 0.5$ and $\delta = (2n)^{-1}$, and compare the iterates of Algorithm \ref{algo: private sparse glm} with the true $\bbeta$ for $s = 10, 20$, or $40$. As suggested by Theorem \ref{thm: glm upper bound}, the log error $\log(\|\bbeta^t - \bbeta\|_2^2)$ is linear in $t$ when $s = 10$ but deteriorates as $s$ increases.
	\item [(b).] Fix $d = 10000, s = 10, \varepsilon = 0.5$ and $\delta = (2n)^{-1}$, and compare the iterates of Algorithm \ref{algo: private sparse glm} with the true $\bbeta$ for $n = 20000, 40000$, or $80000$. $\log(\|\bbeta^t - \bbeta\|_2^2)$ is linear in $t$ when $n = 80000$ or $n = 40000$, but deteriorates as $n$ decreases.
	\item [(c).] Fix $d = 10000, n = 40000, s = 10$ and $\delta = (2n)^{-1}$, and compare the iterates of Algorithm \ref{algo: private sparse glm} with the true $\bbeta$ for $\varepsilon = 0.2, 0.5, 0.8$, or $\infty$ (non-private). The decrease in $\log(\|\bbeta^t - \bbeta\|_2^2)$ as $\varepsilon$ increases confirms Theorem \ref{thm: glm upper bound}.
\end{enumerate}
	\begin{figure}[H]
		\subfloat[]{\includegraphics[width= 0.33\textwidth]{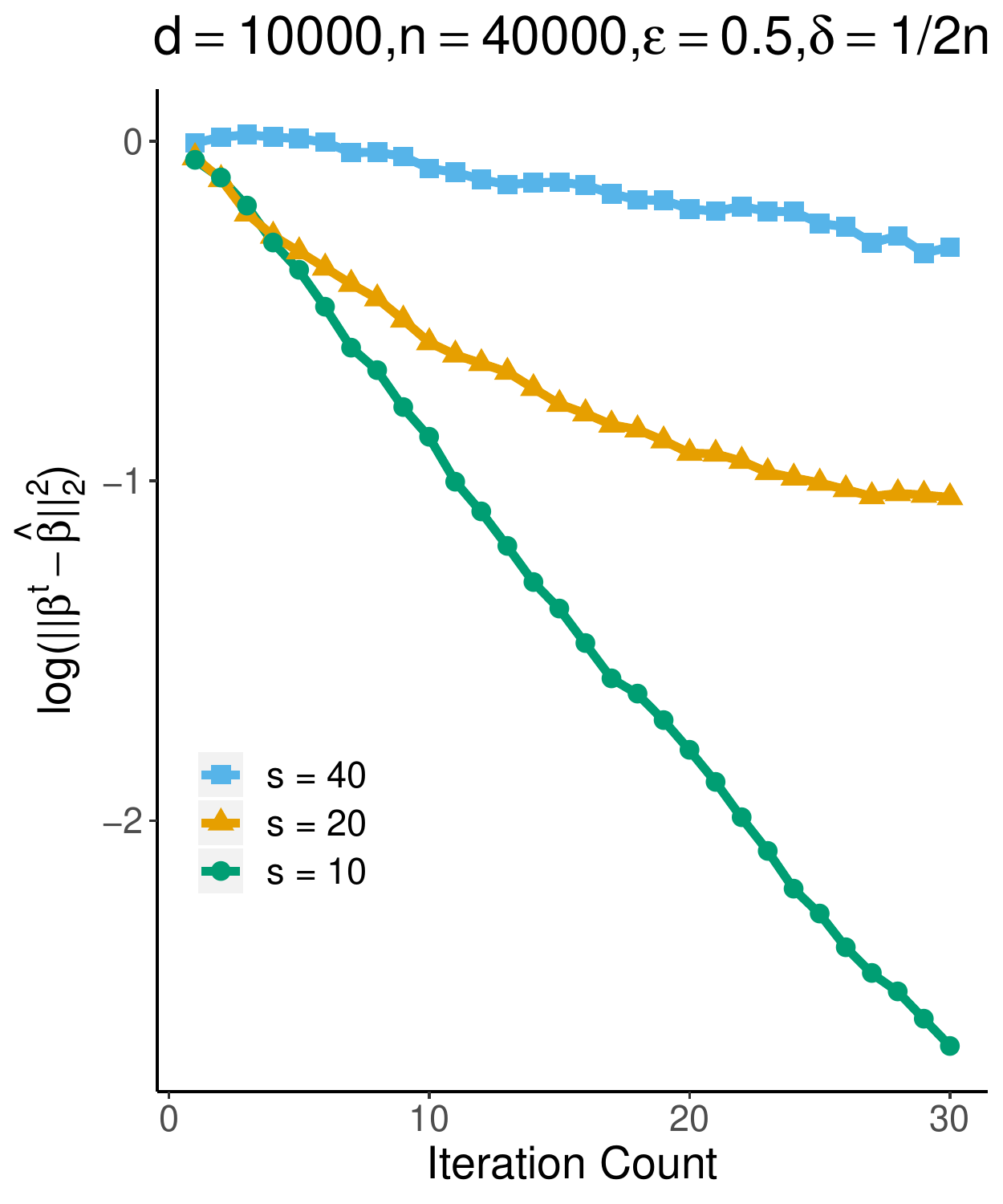}}
		\subfloat[]{\includegraphics[width= 0.33\textwidth]{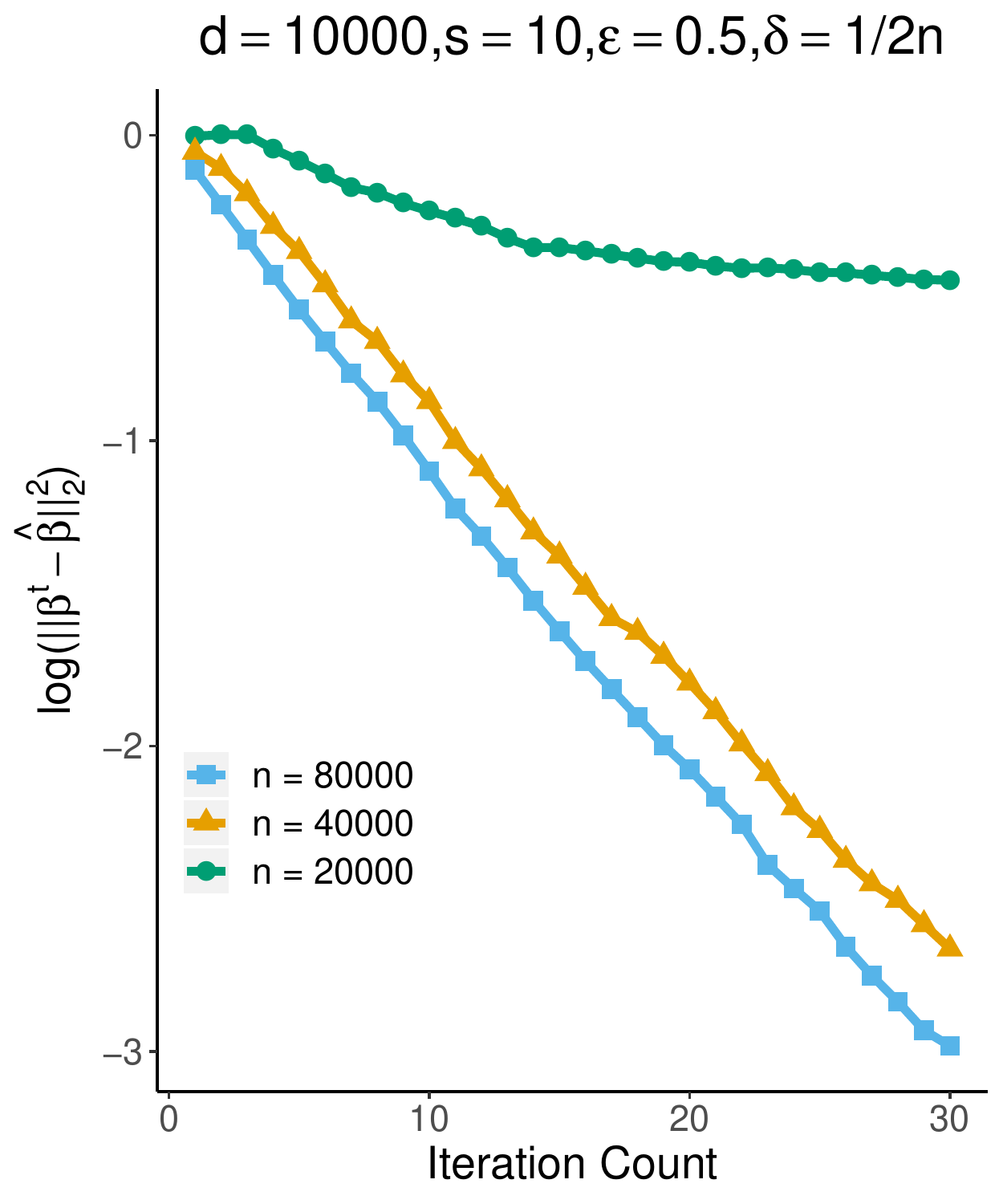}}
		\subfloat[]{\includegraphics[width= 0.33\textwidth]{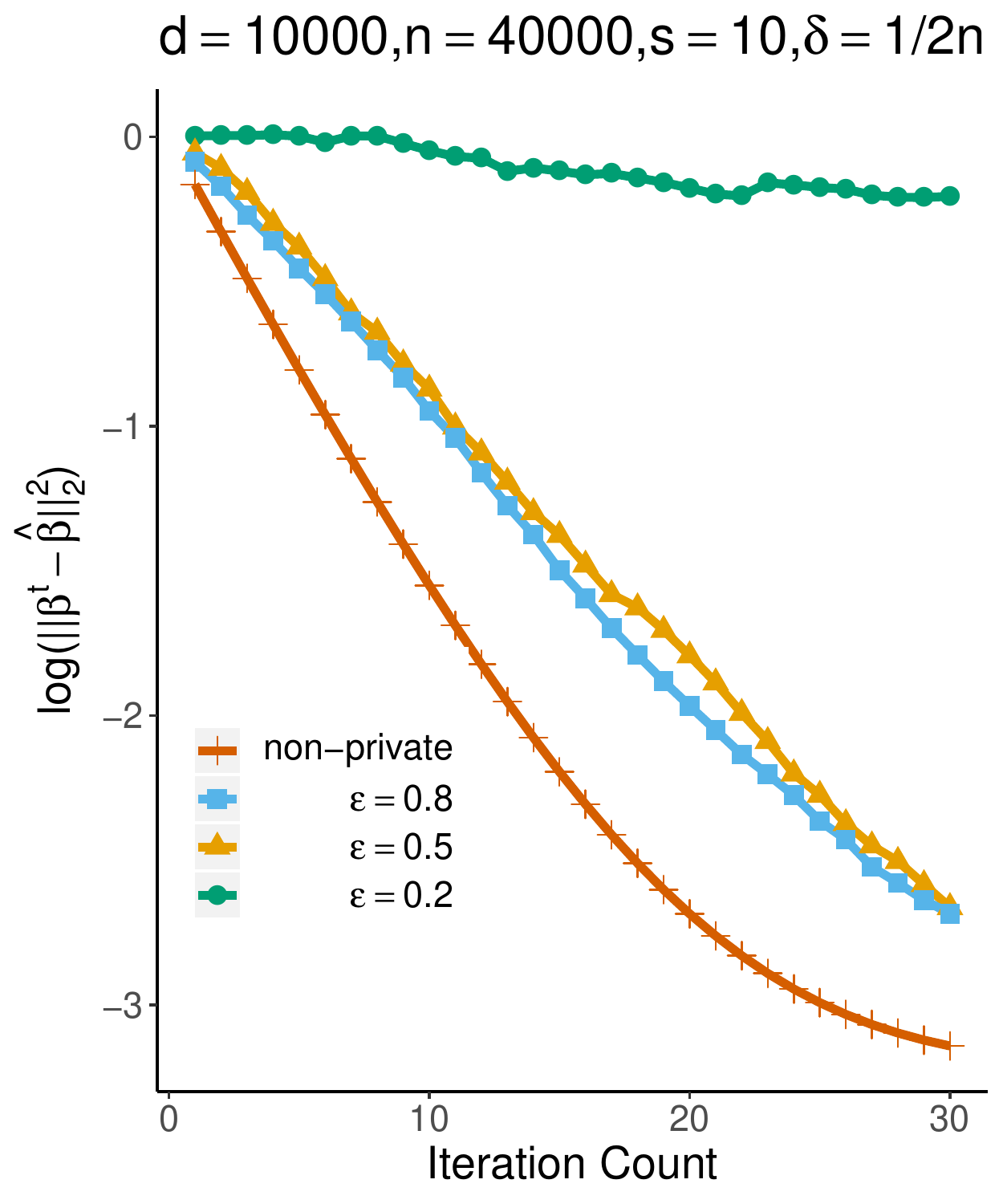}}
		\caption{Log-distance between the iterates of Algorithm \ref{algo: private sparse glm} and the true parameter $\bbeta$ under various settings of $n, d, s, \varepsilon$ and $\delta$.}
	\end{figure}

\subsection{Real Data}\label{sec: real data}
For the real data experiment, we consider the Swarm Behavior Data Set, collected by the Human Perception of Swarming project at the University of New South Wales (\url{https://unsw-swarm-survey.netlify.app/}) and made publicly available at the UCI Machine Learning Repository \cite{Dua:2019}. In this data set, each of $n = 24016$ instances contains $d = 2400$ attributes describing the behavior (velocity, direction, location, etc.) of 200 individuals in the system, and with each instance assigned a binary class label, ``flocking'' or ``not flocking''. A system of individual birds, insects, or people are said to be ``flocking'' if they are perceived moving as a group with the same velocity without colliding each other. 

In our experiment, we attempt to classify these instances into ``flocking'' or ``not flocking'' by our Algorithm \ref{algo: private sparse glm} for high-dimensional sparse GLMs. We randomly split the data set into two halves, train a sparse logistic regression model using one half, and predict the labels of the other half by this logistic model. For fitting the sparse logistic model on the training set, we run Algorithm \ref{algo: private sparse glm} for $50$ iterations with step size $\eta^0 = 0.5$ and initial value $\bbeta^0  = \bm 0 \in \R^{2401}$ (including the intercept). For various settings of $s, \varepsilon$ and $\delta$, the average misclassification rate (and its standard error) over repetitions of the experiment are displayed in the tables below. The results suggest that the classification accuracy indeed worsens as the privacy requirement becomes more stringent, but the loss of accuracy is mild compared to the non-private $\varepsilon = \infty$ case.
\begin{figure}[H]
	\subfloat[$\delta = 1/2n$]
	{\begin{tabular}{|c|c|c|c|}
			\hline
			& $s = 25$     & $s=50$       & $s=100$      \\ \hline
			$\varepsilon = 0.2$  & 0.33(.05) & 0.21(.05) & 0.13(.05) \\ \hline
			$\varepsilon = 0.5$ & 0.28(.05) & 0.20(.05) & 0.10(.02) \\ \hline
			$\varepsilon = \infty$         & 0.30(.05) & 0.21(.05) & 0.09(.03) \\ \hline
	\end{tabular}}
\quad
	\subfloat[$\delta = 1/n^2$]
	{\begin{tabular}{|c|c|c|c|}
			\hline
			& $s = 25$   & \multicolumn{1}{c|}{$s=50$} & $s=100$    \\ \hline
			$\varepsilon = 0.2$ & 0.32(.05) & 0.22(.05)                  & 0.13(.06) \\ \hline
			$\varepsilon = 0.5$ & 0.33(.04) & 0.19(.05)                  & 0.08(.02) \\ \hline
			$\varepsilon = \infty$         & 0.30(.05) & 0.21(.05)                  & 0.09(.02) \\ \hline
		\end{tabular}
	}
\caption{Mean and standard error of misclassification rates of Algorithm \ref{algo: private sparse glm} in the randomly drawn test subset of the Swarm Behavior Data Set.}
\end{figure}

%% file: 6_Discussion.tex
%!TEX root = Privacy-GLM-JRSSB.tex
%%%%%%%%%%%%%%%%%%%%%%%%%%%%%%%%%%%%%%%%%%%%%%%%%%%%%%%%%%%%%%%%%%%
\section{Discussion}\label{sec: discussion}

In this paper, we studied the cost of differential privacy in estimating the parameters of the GLMs. We designed differentially private algorithms, based on projected gradient descent, that achieve fast, linear convergence to the optimal non-private solution, and analyzed their statistical accuracy with respect to the true parameters. The theoretical properties of our algorithms are demonstrated in numerical experiments with real and simulated data sets.

The accuracy of these algorithms are shown to be optimal up to logarithmic factors, via lower bounds of the privacy-constrained minimax risk. These lower bounds are established by the score attack framework, which generalizes prior works on tracing attacks for privacy-constrained minimax lower bounds. The upper bounds and lower bounds together have led to a clear characterization of the cost of privacy in estimating GLM parameters.

This paper suggests several promising directions of further research. On the algorithmic side, since our convergence analysis of differentially private algorithms can be applied to other $M$-estimation problems satisfying restricted strong convexity and restricted smoothness, it is of interest to study their performance in problems such as low-rank matrix recovery and regression. Our results on high-dimensional sparse GLMs also raise questions on the interplay between privacy and other structural assumptions, for example, group-structured sparsity, approximate sparsity, or low-rankness as mentioned above.

On the statistical optimality side, our score attack framework may lead to lower bounds for a much larger variety of statistical models than generalized linear models. It is also of significant value to prove sharper lower bounds that can potentially capture the remaining logarithmic gap between upper and lower bounds, and develop sharp lower bounds for differentially private confidence intervals or hypothesis testing.  

%% file: 7_Proofs.tex
%!TEX root = Privacy-GLM-JRSSB.tex

\section{Proofs}\label{sec: proofs}
In this section, we prove the main technical results of this paper, Theorems \ref{thm: glm upper bound} and \ref{thm: score attack general}. 
\subsection{Proof of Theorem \ref{thm: glm upper bound}}\label{sec: proof of thm: glm upper bound}
\begin{proof}[Proof of Theorem \ref{thm: glm upper bound}]
	We shall first define several favorable events under which the desired convergence does occur, and then show that the probability that any of the favorable events fails to happen is negligible. These events are,
	\begin{align*}
	&\mathcal E_1 = \{\eqref{eq: glm rsc} \text{ and } \eqref{eq: glm rsm} \text{ hold}\}, \mathcal E_2  = \{\Pi_R(y_i) = y_i, \forall i \in [n]\}, \mathcal E_3  = \{\|\bbeta^t - \hat\bbeta\|_2 \leq 3, 0 \leq t \leq T\}.
	\end{align*}
	
	We first analyze the behavior of Algorithm \ref{algo: private sparse glm} under these events. The assumed scaling of $n \geq K \cdot \left(Rs^*\log d \sqrt{\log(1/\delta)}\log n/\varepsilon\right)$ implies that $n \geq K' s^*\log d/n$ for a sufficiently large $K'$. Since $\|\bbeta^t\|_0 \leq s \asymp s^*$ for every $t$ and $\|\hat\bbeta\|_0 \leq s^*$ by definition, the RSM condition \eqref{eq: glm rsm} implies that for every $t$,
	\begin{align}
	& \langle \nabla \L_n(\bbeta^{t}) - \nabla \L_n(\hat\bbeta), \bbeta^{t} - \hat\bbeta \rangle \leq  \frac{4\gamma}{3}\|\bbeta^{t} - \hat\bbeta\|_2^2. \label{eq: glm rsm modified}
	\end{align}
	Similarly, under event $\mathcal E_3$, the RSC condition \eqref{eq: glm rsc} implies that 
	\begin{align}
	& \langle \nabla \L_n(\bbeta^{t}) - \nabla \L_n(\hat\bbeta), \bbeta^{t} - \hat\bbeta \rangle \geq  \frac{2\alpha}{3}\|\bbeta^{t} - \hat\bbeta\|_2^2. \label{eq: glm rsc modified}
	\end{align}
	These two inequalities and our choice of parameters $s, \eta$ now allow Theorem \ref{thm: noisy iterative hard theresholding convergence} to apply. Let $\bm w^t_1, \bm w^t_2, \cdots, \bm w^t_s$ be the noise vectors added to $\bbeta^t - \eta^0\nabla \L_n(\bbeta^t; Z)$ when the support of $\bbeta^{t+1}$ is iteratively selected, $S^{t+1}$ be the support of $\bbeta^{t+1}$, and $\tilde {\bm w}^t$ be the noise vector added to the selected $s$-sparse vector. Define $\bm W_t =C_\gamma\left(\sum_{i \in [s]} \|\bm w^t_i\|^2_\infty + \|\tilde {\bm w}^t_{S^{t+1}}\|_2^2\right)$, then Theorem \ref{thm: noisy iterative hard theresholding convergence} leads to
	\begin{align}
	\L_n(\bbeta^{(T)}) - \L_n(\hat\bbeta) &\leq \left(1-\rho\frac{\alpha}{2\gamma}\right)^T\left(\L_n(\bbeta^0) - \L_n(\hat\bbeta)\right)  + \sum_{k=0}^{T-1}\left(1-\rho\frac{\alpha}{2\gamma}\right)^{T-k-1}\bm W_k \notag\\
	&\leq  
	\left(1-\rho\frac{\alpha}{2\gamma}\right)^T\frac{2\gamma}{3}\|\bbeta_0 - \hat\bbeta\|_2^2 + \sum_{k=0}^{T-1}\left(1-\rho\frac{\alpha}{2\gamma}\right)^{T-k-1}\bm W_k \notag \\
	&\leq  
	\left(1-\rho\frac{\alpha}{2\gamma}\right)^T6\gamma + \sum_{k=0}^{T-1}\left(1-\rho\frac{\alpha}{2\gamma}\right)^{T-k-1}\bm W_k.\label{eq: glm suboptimality upper bound}
	\end{align}
	The second inequality is a consequence of \eqref{eq: glm rsm modified}, and the third inequality follows from the assumption that $\|\bbeta_0 - \hat\bbeta\|_2 \leq 3$. 
	On the other hand, we can lower bound $\L_n(\bbeta^{(T)}) - \L_n(\hat\bbeta)$ as follows: by \eqref{eq: glm rsc modified},
	\begin{align}
	\L_n(\bbeta^{(T)}) - \L_n(\hat\bbeta) \geq \L_n(\bbeta^{(T)}) - \L_n(\bbeta^*) \geq \frac{\alpha}{3}\|\bbeta^{(T)} - \bbeta^*\|_2^2 - \langle \nabla \L_n(\bbeta^*), \bbeta^* - \bbeta^{(T)} \rangle. \label{eq: glm suboptimality lower bound}
	\end{align}
	Combining \eqref{eq: glm suboptimality upper bound} and \eqref{eq: glm suboptimality lower bound} yields
	\begin{align}
	\frac{\alpha}{3}\|\bbeta^{(T)} - \bbeta^*\|_2^2 &\leq  \langle \nabla \L_n(\bbeta^*), \bbeta^* - \bbeta^{(T)} \rangle + \left(1-\rho\frac{\alpha}{2\gamma}\right)^T6\gamma + \sum_{k=0}^{T-1}\left(1-\rho\frac{\alpha}{2\gamma}\right)^{T-k-1}\bm W_k \notag \\
	&\leq  \|\nabla \L_n(\bbeta^*)\|_\infty \sqrt{s+s^*}\|\bbeta^* - \bbeta^{(T)}\|_2 + \left(1-\rho\frac{\alpha}{2\gamma}\right)^T6\gamma + \sum_{k=0}^{T-1}\left(1-\rho\frac{\alpha}{2\gamma}\right)^{T-k-1}\bm W_k  \notag \\
	&= \|\nabla \L_n(\bbeta^*)\|_\infty \sqrt{s+s^*}\|\bbeta^* - \bbeta^{(T)}\|_2 + \frac{1}{n} + \sum_{k=0}^{T-1}\left(1-\rho\frac{\alpha}{2\gamma}\right)^{T-k-1}\bm W_k. \label{eq: glm upper bound fundamental inequality}
	\end{align}
	The last step follows from our choice of  $T = (2\gamma/\rho\alpha)\log(6\gamma n)$. Now let us define two events that allow for high-probability bounds of the right side.
	\begin{align*}
	&\mathcal E_4 = \left\{\max_t \bm W_t \leq  K\left(\frac{Rs^*\log d \sqrt{\log(1/\delta)}\log n}{n\varepsilon}\right)^2\right\},
	&\mathcal E_5 = \left\{\|\nabla \L_n(\bbeta^*)\|_\infty \leq 4\sigma_x\sqrt{c_2}\sqrt{\frac{\log d}{n}}\right\}.
	\end{align*}
	Under $\mathcal E_4, \mathcal E_5$, we can conclude from \eqref{eq: glm upper bound fundamental inequality} that
	\begin{align*}
	\|\bbeta^{(T)} - \bbeta^*\|_2 \lesssim \sqrt{c(\sigma)}\left(\sqrt{\frac{s^*\log d}{n}} + \frac{s^*\log d \sqrt{\log(1/\delta)}\log^{3/2} n}{n\varepsilon}.\right)
	\end{align*}
	
	We have shown so far that the desired rate of convergence \eqref{eq: glm upper bound} holds when $\mathcal E_i$ occurs for $1 \leq i \leq 5$; we now turn to controlling the probability that any of the five events fails to happen, $\sum_{i=1}^5 \Pro(\mathcal E^c_i)$.
	\begin{itemize}
		\item By Proposition \ref{lm: glm rsc and rsm}, $\Pro(\mathcal E_1^c) \leq c_3\exp(-c_4n)$ under the assumptions of Theorem \ref{thm: glm upper bound}.
		\item We have $\Pro(\mathcal E_2^c) \leq c_3\exp(-c_4\log n)$ by the choice of $R$, and assumptions (G1), (G2) which imply the following bound of moment generating function of $y_i$: we have
		\begin{align*}
			\log \E \exp\left(\lambda 
			\cdot \frac{y_i - \psi'(\bm x_i^\top\bbeta)}{c(\sigma)}\Big| \bm x_i\right) &= \frac{1}{c(\sigma)}\left(\psi(\bm x_i^\top\bbeta + \lambda) - \psi(\bm x_i^\top\bbeta) - \lambda\psi'(\bm x_i^\top\bbeta)\right) \\
			&\leq \frac{1}{c(\sigma)} \cdot \frac{\lambda^2 \psi^{''}(\bm x_i^\top\bbeta + \tilde \lambda)}{2}
		\end{align*}
		for some $\tilde \lambda \in (0, \lambda)$. It follows that $\E \exp\left(\lambda \cdot \frac{y_i - \psi'(\bm x_i^\top\bbeta)}{c(\sigma)}\Big| \bm x_i\right) \leq \exp\left(\frac{c_2\lambda^2}{2c(\sigma)}\right)$ because $\|\psi^{''}\|_\infty <c_2$.
		\item For $\mathcal E_3$, we have $\Pro(\mathcal E^c_3) \leq T \cdot c_3\exp(-c_4\log(d/s^*)) =  c_3\exp(-c_4\log(d/s^*\log n))$ by the initial condition $\|\bbeta^0 - \hat\bbeta\|_2^3$ and proof by induction via the following lemma, to be proved in Section \ref{sec: proof of lm: glm contraction}.
		\begin{Lemma}\label{lm: glm contraction}
			Under the assumptions of Theorem \ref{thm: glm upper bound} Let $\bbeta^{k}, \bbeta^{k+1}$ be the $k$th and $(k+1)$th iterates of Algorithm \ref{algo: private sparse glm}. If $\|\bbeta^{k} - \hat\bbeta\|_2 \leq 3$, we have $\|\bbeta^{k+1} - \hat\bbeta\|_2 \leq 3$ with probability at least $1 - c_3\exp(-c_4\log (d/s^*))$.
		\end{Lemma}
		\item For $\mathcal E_4$, we invoke an auxiliary lemma to be proved in Section \ref{sec: proof of lm: laplace noise bound}.
		\begin{Lemma}\label{lm: laplace noise bound}
			Consider $\bm w \in \R^k$ with $w_1, w_2, \cdots, w_k \stackrel{\text{i.i.d.}}{\sim}$ Laplace$(\lambda)$. For every $C > 1$, 
			\begin{align*}
			& \Pro\left(\|\bm w\|_2^2 > kC^2\lambda^2\right) \leq ke^{-C}\\
			& \Pro\left(\|\bm w\|_\infty^2 > C^2\lambda^2\log^2k\right) \leq e^{-(C-1)\log k}.
			\end{align*}
		\end{Lemma} 	
		For each iterate $t$, the individual coordinates of $\tilde{\bm w}^t$, $\bm w^t_i$ are sampled i.i.d. from the Laplace distribution with scale $(2\gamma)^{-1} \cdot \frac{2B\sqrt{3s\log(T/\delta)}}{n\varepsilon/T}$, where the noise scale $B \lesssim R$ and $T \asymp \log n$ by our choice. If $n \geq K \cdot \left(Rs^*\log d \sqrt{\log(1/\delta)}\log n/\varepsilon\right)$ for a sufficiently large constant $K$, Lemma \ref{lm: laplace noise bound} implies that, with probability at least $1-c_3\exp(-c_4\log(d/(s^*\log n))$, $\max_{t} \bm W_t$ is bounded by $K\left(\frac{Rs^*\log d \sqrt{\log(1/\delta)}\log n}{n\varepsilon}\right)^2$ for some appropriate constant $K$. 
		
		\item Under assumptions of Theorem \ref{thm: glm upper bound}, it is a standard probabilistic result (see, for example, \cite{wainwright2019high} pp. 288) that $\Pro(\mathcal E^c_5) \leq 2e^{-2\log d}$.
	\end{itemize} 
	We have $\sum_{i=1}^5 \Pro(\mathcal E^c_i) \leq  c_3\exp(-c_4\log(d/s^*\log n)) + c_3\exp(-c_4n) + c_3\exp(-c_4\log n)$. The proof is complete.
\end{proof}

\subsection{Proof of Theorem \ref{thm: score attack general}}\label{sec: proof of thm: score attack general}
\begin{proof}
	For soundness, we note that $\bm x_i$ and $M(\bm X'_i)$ are independent, and therefore
	\begin{align*}
	\E \mathcal A_\bth(\bm x_i, M(\bm X'_i)) = \E \langle M(\bm X'_i) - \bth, S_\bth(\bm x_i)\rangle = \langle \E M(\bm X'_i) - \bth,  \E S_\bth(\bm x_i)\rangle  = \bm 0.
	\end{align*}
	The last equality is true by the property of the score that $\E S_\bth(\bm z) = \bm 0$ for any $\bm z \sim f_\bth$. As to the first absolute moment, we apply Jensen's inequality,
	\begin{align*}
	\E |\mathcal A_\bth(\bm x_i, M(\bm X'_i))| &\leq \sqrt{\E \langle M(\bm X'_i) - \bth, S_\bth(\bm x_i)\rangle^2}  \\
	&\leq \sqrt{\E (M(\bm X'_i) - \bth)^\top (\Var S_\bth(\bm x_i))(M(\bm X'_i) - \bth)} \leq \sqrt{\E\|M(\bm X) - \bth\|^2_2} \sqrt{\lambda_{\max}(\mathcal I(\bth))}.
	\end{align*}
	
	For completeness, we first simplify
	\begin{align*}
	\sum_{i \in [n]} \E \A_\bth (\bm x_i, M(\bm X))
	= \E \Big\langle M(\bm X) - \bth, \sum_{i \in [n]} S_\bth(\bm x_i)\Big\rangle = \E \Big\langle M(\bm X), \sum_{i \in [n]} S_\bth(\bm x_i)\Big\rangle.
	\end{align*}
	By the definition of score and that $\bm x_1, \cdots, \bm x_n$ are i.i.d., $\sum_{i \in [n]} S_\bth(\bm x_i) = S_\bth(\bm x_1,\cdots, \bm x_n) = S_\bth(\bm X)$. It follows that
	\begin{align*}
	\E \Big\langle M(\bm X), \sum_{i \in [n]} S_\bth(\bm x_i)\Big\rangle = \E \Big\langle M(\bm X), S_\bth(\bm X)\Big\rangle = \sum_{j \in [d]} \E \left[M(\bm X)_j \frac{\partial}{\partial \theta_j} \log f_\bth(\bm X) \right].
	\end{align*}
	
	For each term in the right-side summation, one may exchange differentiation and integration thanks to the regularity conditions on $f_\bth$, and therefore
	\begin{align*}
	\E \left[M(\bm X)_j \frac{\partial}{\partial \theta_j} \log f_\bth(\bm X) \right] & = \E \left[M(\bm X)_j (f_\bth(\bm X))^{-1}\frac{\partial}{\partial \theta_j} f_\bth(\bm X) \right]  \\
	&= \frac{\partial}{\partial \theta_j}\E \left[M(\bm X)_j (f_\bth(\bm X))^{-1} f_\bth(\bm X) \right] = \frac{\partial}{\partial \theta_j}\E M(\bm X)_j.
	\end{align*}
\end{proof}

\subsubsection{Proof of Lemma \ref{lm: score attack upper bound}}\label{sec: proof of lm: score attack upper bound}
\begin{proof}
	Let $A_i := \A_\bth (\bm x_i, M(\bm X))$, $A'_i := \A_\bth (\bm x_i, M(\bm X'_i))$, and let $Z^+ = \max(Z, 0)$ and $Z^- = -\min(Z, 0)$ denote the positive and negative parts of a random variables $Z$ respectively. We have
	\begin{align*}
	\E A_i = \E A_i^+ - \E A_i^- = \int_0^\infty \Pro(A_i^+ > t) \d t - \int_0^\infty \Pro(A_i^- > t) \d t.
	\end{align*}
	For the positive part, if $0 < T < \infty$ and $0 < \varepsilon < 1$, we have
	\begin{align*}
	\int_0^\infty \Pro(A_i^+ > t) \d t &= \int_0^T \Pro(A_i^+ > t) \d t + \int_T^\infty \Pro(A_i^+ > t) \d t \\
	&\leq \int_0^T \left(e^\varepsilon\Pro(A_i^+ > t) + \delta\right)\d t + \int_T^\infty \Pro(A_i^+ > t) \d t \\
	&\leq \int_0^\infty \Pro({A'_i}^+ > t) \d t + 2\varepsilon\int_0^\infty \Pro({A'_i}^+ > t) \d t + \delta T + \int_T^\infty \Pro(|A_i| > t) \d t.
	\end{align*}
	Similarly for the negative part,
	\begin{align*}
	\int_0^\infty \Pro(A_i^- > t) \d t &= \int_0^T \Pro(A_i^- > t) \d t + \int_T^\infty \Pro(A_i^- > t) \d t \\
	& \geq \int_0^T \left(e^{-\varepsilon} \Pro({A'_i}^- > t) - \delta\right)\d t + \int_T^\infty \Pro(A_i^- > t) \d t \\
	&\geq \int_0^T \Pro({A'_i}^- > t) \d t - 2\varepsilon\int_0^T \Pro({A'_i}^- > t) - \delta T + \int_T^\infty \Pro(A_i^- > t) \d t \\
	&\geq \int_0^\infty \Pro({A'_i}^- > t) \d t - 2\varepsilon\int_0^\infty \Pro({A'_i}^- > t) - \delta T.
	\end{align*}
	It then follows that
	\begin{align*}
	\E A_i &\leq \int_0^\infty \Pro({A'_i}^+ > t) \d t - \int_0^\infty \Pro({A'_i}^- > t) \d t + 2\varepsilon \int_0^\infty \Pro(|A'_i| > t) \d t + 2\delta T + \int_T^\infty \Pro(|A_i| > t) \d t \\
	&= \E A'_i + 2\varepsilon\E|A_i| + 2\delta T + \int_T^\infty \Pro(|A_i| > t) \d t.
	\end{align*}
	The proof is now complete by soundness \eqref{eq: soundness general}.
\end{proof}

\subsubsection{Proof of Lemma \ref{lm: score attack stein's lemma}}\label{sec: proof of lm: score attack stein's lemma}
\begin{proof}
	For each $j \in [d]$, by Lemma \ref{lm: stein's lemma}, we have
	\begin{align*}
	\E_{\pi_j} \left(\frac{\partial}{\partial \theta_j} g_j(\bth)\right) = \E_{\pi_j} \left(\frac{\partial}{\partial \theta_j} \E[g_j(\bth)|\theta_j]\right) = \E_{\pi_j}\left[\frac{-\E[g_j(\bth)|\theta_j]\pi'_j(\theta_j)}{\pi_j(\theta_j)}\right]
	\end{align*}
	Because $|g_j(\bth) - \theta_j| \leq \|g(\bth) - \bth\|_2 \leq \E_{\bm X|\bth}\|M(\bm X) - \bth\|_2$ for every $\bth \in \Theta$, we have
	\begin{align*}
	\E_{\pi_j}\left[\frac{-\E[g(\bth)|\theta_j]\pi'_j(\theta_j)}{\pi_j(\theta_j)}\right] &\geq \E_{\pi_j}\left[\frac{-\theta_j\pi'_j(\theta_j)}{\pi_j(\theta_j)}\right] - \E_{\pi_j}\left[\E_{\bm X|\bth}\|M(\bm X) - \bth\|_2 \cdot \left|\frac{\pi'_j(\theta_j)}{\pi_j(\theta_j)}\right|\right] \\
	&\geq \E_{\pi_j}\left[\frac{-\theta_j\pi'_j(\theta_j)}{\pi_j(\theta_j)}\right] - \sqrt{\E_{\pi_j}\E_{\bm X|\bth}\|M(\bm X) - \bth\|^2_2}\sqrt{\E_{\pi_j}\left[\left(\frac{\pi'_j(\theta_j)}{\pi_j(\theta_j)}\right)^2\right].}
	\end{align*}
	So we have obtained
	\begin{align*}
	\E_{\pi_j} \left(\frac{\partial}{\partial \theta_j} g_j(\bth)\right) \geq \E_{\pi_j}\left[\frac{-\theta_j\pi'_j(\theta_j)}{\pi_j(\theta_j)}\right] - \sqrt{\E_{\pi_j}\E_{\bm X|\bth}\|M(\bm X) - \bth\|^2_2}\sqrt{\E_{\pi_j}\left[\left(\frac{\pi'_j(\theta_j)}{\pi_j(\theta_j)}\right)^2\right].}
	\end{align*}
	Now we take expectation over $\bm\pi(\bth)/\pi_j(\theta_j)$ and sum over $j \in [d]$ to complete the proof.
\end{proof}

%% file: A_Upper_bound_proofs.tex
%!TEX root = Privacy-GLM-JRSSB.tex

\section{Omitted Proofs in Section \ref{sec: glm upper bounds}}\label{sec: ub proofs}

\subsection{Proof of Lemma \ref{lm: non-sparse glm privacy}}\label{sec: proof of lm: non-sparse glm privacy}
\begin{proof}[Proof of Lemma \ref{lm: non-sparse glm privacy}]
	Consider two data sets $\bm Z$ and $\bm Z'$ that differ only by one datum, $(y, \bm x) \in \bm Z$ versus $(y', \bm x') \in \bm Z'$. For any $t$, we have
	\begin{align*}
		\|\bbeta^{t+1}(\bm Z) - \bbeta^{t+1}(\bm Z')\|_2 &\leq \frac{\eta^0}{n}\left(|\psi'(\bm x^\top \bbeta^t)-\Pi_R(y)|\|\bm x\|_2 + |\psi'((\bm x')^\top \bbeta^t)-\Pi_R(y')|\|\bm x'\|_2\right) \\
		&\leq \frac{\eta^0}{n}4(R+c_1)\sigma_{\bm x}\sqrt{d},
	\end{align*}
	where the last step follows from (D1) and (G1). By the Gaussian mechanism, Fact \ref{fc: laplace and gaussian mechanisms}, $\bbeta^{t+1}(\bm Z)$ is $(\varepsilon/T, \delta/T)$-differentially private, implying that Algorithm \ref{algo: private glm} is $(\varepsilon, \delta)$-differentially private.
\end{proof}

\subsection{Proof of Theorem \ref{thm: non-sparse glm upper bound}}\label{sec: proof of thm: non-sparse glm upper bound}
\begin{proof}[Proof of Theorem \ref{thm: non-sparse glm upper bound}]
	We shall first define several favorable events under which the desired convergence does occur, and then show that the probability that any of the favorable events fails to happen is negligible. The events are,
	\begin{align*}
	&\mathcal E_1 = \{\eqref{eq: glm rsc} \text{ and } \eqref{eq: glm rsm} \text{ hold}\}, \mathcal E_2  = \{\Pi_R(y_i) = y_i, \forall i \in [n]\}, \mathcal E_3  = \{\|\bbeta^t - \hat\bbeta\|_2 \leq 3, 0 \leq t \leq T\}.
	\end{align*}
	
	Let us first analyze the behavior of Algorithm \ref{algo: private glm} under these events. The scaling of $n \geq K \cdot \left(Rd\sqrt{\log(1/\delta)}\log n \log\log n/\varepsilon\right)$ for a sufficiently large $K$ implies that $n \geq K'd\log d$ for a sufficiently large $K'$. Since $\|\bbeta_1 - \bbeta\|_1 \leq \sqrt{d}\|\bbeta_1-\bbeta_2\|_2$ for all $\bbeta_1, \bbeta_2 \in \R^d$, the RSM condition \eqref{eq: glm rsm} implies that for every $t$,
	\begin{align}
	& \langle \nabla \L_n(\bbeta^{t}) - \nabla \L_n(\hat\bbeta), \bbeta^{t} - \hat\bbeta \rangle \leq  \frac{4\gamma}{3}\|\bbeta^{t} - \hat\bbeta\|_2^2. \label{eq: non-sparse glm rsm modified}
	\end{align}
	Similarly, under event $\mathcal E_3$, the RSC condition \eqref{eq: glm rsc} implies that 
	\begin{align}
	& \langle \nabla \L_n(\bbeta^{t}) - \nabla \L_n(\hat\bbeta), \bbeta^{t} - \hat\bbeta \rangle \geq  \frac{2\alpha}{3}\|\bbeta^{t} - \hat\bbeta\|_2^2. \label{eq: non-sparse glm rsc modified}
	\end{align} 
	
	To analyze the convergence of Algorithm \ref{algo: private glm}, define $\tilde \bbeta^{t+1} = \bbeta^t - \eta^0\nabla\L_n(\bbeta^t)$, so that $\bbeta^{t+1} = \tilde \bbeta^{t+1} + \bm w_t$. Let $\hat\bbeta = \argmin_\bbeta \L_n(\bbeta)$. It follows that
	\begin{align}\label{eq: non-sparse glm master expansion 1}
	\|\bbeta^{t+1} - \hat\bbeta\|_2^2 \leq \left(1 + \frac{\alpha}{4\gamma}\right)\|\tilde \bbeta^{t+1} - \hat\bbeta\|_2^2 + \left(1 + \frac{4\gamma}{\alpha}\right)\|\bm w_t\|_2^2.
	\end{align}
	Now for $\|\tilde \bbeta^{t+1} - \hat\bbeta\|_2^2$,
	\begin{align}\label{eq: non-sparse glm contraction}
	\|\tilde \bbeta^{t+1} - \hat\bbeta\|_2 = \|\bbeta^t - \hat\bbeta\|_2^2 - 2\eta^0 \langle \nabla\L_n(\bbeta^t), \bbeta^t  - \hat\bbeta\rangle + \left(\eta^0\right)^2 \|\nabla \L_n(\bbeta^t)\|_2^2.
	\end{align}
	We would like to bound the last two terms via the strong convexity \eqref{eq: non-sparse glm rsc modified} and smoothness \eqref{eq: non-sparse glm rsm modified}, as follows
	\begin{align*}
	&\L_n(\tilde\bbeta^{t+1}) - \L_n(\hat\bbeta) = \L_n(\tilde\bbeta^{t+1}) - \L_n(\bbeta^t) + \L_n(\bbeta^{t}) - \L_n(\hat\bbeta) \\
	&\leq \langle \nabla\L_n(\bbeta^t), \tilde\bbeta^{t+1}  - \bbeta^t\rangle + \frac{2\gamma}{3}\|\tilde\bbeta^{t+1}  - \bbeta^t\|_2^2 +  \langle \nabla\L_n(\bbeta^t), \bbeta^{t}  - \hat\bbeta\rangle - \frac{\alpha}{3}\|\bbeta^{t}  - \hat\bbeta\|_2^2\\
	&= \langle \nabla\L_n(\bbeta^t), \tilde\bbeta^{t+1}  - \hat\bbeta\rangle + \frac{3}{8\gamma}\|\nabla\L_n(\bbeta^t)\|_2^2 - \frac{\alpha}{3}\|\bbeta^{t}  - \hat\bbeta\|_2^2\\
	&= \langle \nabla\L_n(\bbeta^t), \tilde\bbeta^{t}  - \hat\bbeta\rangle - \frac{3}{8\gamma}\|\nabla\L_n(\bbeta^t)\|_2^2 - \frac{\alpha}{3}\|\bbeta^{t}  - \hat\bbeta\|_2^2\\
	&= \langle \nabla\L_n(\bbeta^t), \tilde\bbeta^{t}  - \hat\bbeta\rangle - \frac{\eta^0}{2}\|\nabla\L_n(\bbeta^t)\|_2^2 - \frac{\alpha}{3}\|\bbeta^{t}  - \hat\bbeta\|_2^2.
	\end{align*}
	Since $\L_n(\tilde\bbeta^{t+1}) - \L_n(\hat\bbeta) \geq 0$, the calculations above imply that
	\begin{align*}
	- 2\eta^0 \langle \nabla\L_n(\bbeta^t), \bbeta^t  - \hat\bbeta\rangle + \left(\eta^0\right)^2 \|\nabla \L_n(\bbeta^t)\|_2^2 \leq -\frac{\alpha}{2\gamma}\|\bbeta^t - \hat\bbeta\|_2^2.
	\end{align*}
	Substituting back into \eqref{eq: non-sparse glm contraction} and \eqref{eq: non-sparse glm master expansion 1} yields
	\begin{align*}
	\|\bbeta^{t+1} - \hat\bbeta\|_2^2 \leq \left(1 - \frac{\alpha}{4\gamma}\right)\|\bbeta^t - \bbeta^0\|_2^2 + \left(1 + \frac{4\gamma}{\alpha}\right)\|\bm w_t\|_2^2.
	\end{align*}
	It follows by induction over $t$, the choice of $T = \frac{4\gamma}{\alpha} \log(9n)$ and $\|\bbeta^0 - \hat\bbeta\|_2 \leq 3$ that
	\begin{align}\label{eq: non-sparse glm master expansion 2}
	\|\bbeta^T - \hat\bbeta\|_2^2 &\leq \frac{1}{n} + \left(1+\frac{4\gamma}{\alpha}\right)\sum_{k=0}^{T-1} \left(1 - \frac{\alpha}{4\gamma}\right)^{T-k-1}\|\bm w_k\|_2^2.
	\end{align}
	The noise term can be controlled by the following lemma:
	\begin{Lemma}\label{lm: gaussian noise bound}
		For $X_1, X_2, \cdots, X_k \stackrel{\text{i.i.d.}}{\sim} \chi^2_d$, $\lambda > 0$ and $0 < \rho < 1$, 
		\begin{align*}
		\Pro\left(\sum_{j=1}^k \lambda \rho^j X_j > \frac{\rho\lambda d}{1-\rho} + t\right) \leq \exp\left(-\min\left(\frac{(1-\rho^2)t^2}{8\rho^2\lambda^2d}, \frac{t}{8\rho\lambda}\right)\right).
		\end{align*}
	\end{Lemma}
	To apply the tail bound, we let $\lambda = (\eta^0)^2 2B^2 \frac{d\log(2T/\delta)}{n^2(\varepsilon/T)^2}$. It follows that, with $t = K\lambda d$ for a sufficiently large constant $K$, the noise term in \eqref{eq: non-sparse glm master expansion 2} is bounded by $K\lambda d \asymp \left(\frac{Rd \sqrt{\log(1/\delta)}\log n}{n\varepsilon}\right)^2$ with probability at least $1 - c_3\exp(-c_4 d)$. 
	
	Therefore, we have shown so far that, under events $\mathcal E_1, \mathcal E_2, \mathcal E_3$, it holds with probability at least $1 - c_3\exp(-c_4 d)$ that
	\begin{align}\label{eq: non-sparse glm master expansion 3}
	\|\bbeta^T - \hat\bbeta\|_2 &\lesssim \sqrt{\frac{1}{n}} + \frac{Rd \sqrt{\log(1/\delta)}\log n}{n\varepsilon}.
	\end{align}
	Combining with the statistical rate of convergence of $\|\hat\bbeta - \bbeta^*\|$ yields the desired rate of
	\begin{align*}
	\|\bbeta^T - \bbeta^*\|_2 &\lesssim \sqrt{c(\sigma)}\left(\sqrt{\frac{d}{n}} + \frac{d \sqrt{\log(1/\delta)}\log^{3/2} n}{n\varepsilon}\right).
	\end{align*}
	
	It remains to show that the events $\mathcal E_1, \mathcal E_2, \mathcal E_3$ occur with overwhelming probability.
	\begin{itemize}
		\item By Proposition \ref{lm: glm rsc and rsm}, $\Pro(\mathcal E_1^c) \leq c_3\exp(-c_4n)$ under the assumptions of Theorem \ref{thm: non-sparse glm upper bound}.
		\item We have $\Pro(\mathcal E_2^c) \leq c_3\exp(-c_4\log n)$ by the choice of $R$, and assumptions (G1), (G2) which imply the following bound of moment generating function of $y_i$: we have
		\begin{align*}
			\log \E \exp\left(\lambda 
			\cdot \frac{y_i - \psi'(\bm x_i^\top\bbeta)}{c(\sigma)}\Big| \bm x_i\right) &= \frac{1}{c(\sigma)}\left(\psi(\bm x_i^\top\bbeta + \lambda) - \psi(\bm x_i^\top\bbeta) - \lambda\psi'(\bm x_i^\top\bbeta)\right) \\
			&\leq \frac{1}{c(\sigma)} \cdot \frac{\lambda^2 \psi^{''}(\bm x_i^\top\bbeta + \tilde \lambda)}{2}
		\end{align*}
		for some $\tilde \lambda \in (0, \lambda)$. It follows that $\E \exp\left(\lambda \cdot \frac{y_i - \psi'(\bm x_i^\top\bbeta)}{c(\sigma)}\Big| \bm x_i\right) \leq \exp\left(\frac{c_2\lambda^2}{2c(\sigma)}\right)$ because $\|\psi^{''}\|_\infty <c_2$.
		\item For $\mathcal E_3$, we have the following lemma to be proved in \ref{sec: proof of lm: non-sparse glm contraction}
		\begin{Lemma}\label{lm: non-sparse glm contraction}
			Under the assumptions of Theorem \ref{thm: non-sparse glm upper bound}, if $\|\bbeta^0 - \hat\bbeta\|_2 \leq 3$, then $\|\bbeta^t - \hat\bbeta\| _2 \leq 3$ for all $ 0 \leq t \leq T$ with probability at least $1 - c_3\exp(-c_4d)$.
		\end{Lemma}
	\end{itemize}
	We have shown that $\sum_{i=1}^3 \Pro(\mathcal E^c_i) \leq c_3\exp(-c_4d) + c_3\exp(-c_4n) + c_3\exp(-c_4\log n)$. The proof is complete.
\end{proof}

\subsubsection{Proof of Lemma \ref{lm: gaussian noise bound}}
\begin{proof}[Proof of Lemma \ref{lm: gaussian noise bound}]
	Since $\E \sum_{j=1}^k \lambda \rho^j X_j \leq \lambda d \sum_{j=1}^k \rho^j < \frac{\rho\lambda d}{1-\rho}$, we have 
	\begin{align*}
	\Pro\left(\sum_{j=1}^k \lambda \rho^j X_j > \frac{\rho\lambda d}{1-\rho} + t\right) \leq \Pro\left(\sum_{j=1}^k \lambda \rho^j (X_j - \E X_j) > t\right).
	\end{align*}
	The (centered) $\chi^2_d$ random variable is sub-exponential with parameters $(2\sqrt{d}, 4)$, the weighted sum is also sub-exponential, with parameters at most $\left(2\lambda\sqrt{d}\sqrt{\sum_{j=1}^k \rho^{2j}}, 4\lambda\rho\right)$. The desired tail bound now follows directly from standard sub-exponential tail bounds.
\end{proof}

\subsubsection{Proof of Lemma \ref{lm: non-sparse glm contraction}}\label{sec: proof of lm: non-sparse glm contraction}
\begin{proof}[Proof of Lemma \ref{lm: non-sparse glm contraction}]
	We prove the lemma by induction. Suppose $\|\bbeta^t - \hat\bbeta\|_2 \leq 3$, by \eqref{eq: non-sparse glm rsm modified} we have
	\begin{align*}
	&\L_n(\bbeta^{t+1}) - \L_n(\hat\bbeta) = \L_n(\bbeta^{t+1}) - \L_n(\bbeta^t) + \L_n(\bbeta^t) - \L_n(\hat\bbeta) \\
	&\leq \langle \nabla\L_n(\bbeta^t), \bbeta^{t+1}  - \bbeta^t\rangle + \frac{2\gamma}{3}\|\tilde\bbeta^{t+1}  - \bbeta^t\|_2^2 +  \langle \nabla\L_n(\bbeta^t), \bbeta^{t}  - \hat\bbeta\rangle \\
	&= \frac{4\gamma}{3}\langle \bbeta^t - \bbeta^{t+1}, \bbeta^{t+1}  - \hat\bbeta \rangle + \frac{2\gamma}{3}\|\tilde\bbeta^{t+1}  - \bbeta^t\|_2^2 +   \frac{4\gamma}{3}\langle \bm w_t, \bbeta^{t+1} - \hat\bbeta\rangle \\
	&\leq \frac{2\gamma}{3}\left(\|\bbeta^t - \hat\bbeta\|_2^2 - \|\bbeta^{t+1} - \hat\bbeta\|_2^2\right) + \frac{16\gamma^2}{\alpha}\|\bm w_t\|_2^2 + \frac{\alpha}{9}\|\bbeta^{t+1} - \hat\bbeta\|_2^2.
	\end{align*}
	Assume by contradiction that $\|\bbeta^{t+1} - \hat\bbeta\|_2 > 3$. By \eqref{eq: glm rsc} and \eqref{eq: non-sparse glm rsc modified}, we have $\L_n(\bbeta^{t+1}) - \L_n(\hat\bbeta) \geq \alpha\|\bbeta^{t+1} - \hat\bbeta\|_2$ and therefore 
	\begin{align*}
	\left(2\gamma + \frac{2\alpha}{3}\right)\|\bbeta^{t+1} - \hat\bbeta\|_2 \leq 6\gamma + \frac{16\gamma^2}{\alpha}\|\bm w_t\|_2^2.
	\end{align*}
	Recall that the coordinates of $\bm w_t$ are i.i.d. Gaussian with variance of the order $\frac{d\log(1/\delta)\log^2 n}{n^2\varepsilon^2}$. By the scaling of $n \geq K \cdot \left(Rd\sqrt{\log(1/\delta)}\log n \log\log n/\varepsilon\right)$ and the choice of $T \asymp \log n$, it holds with probability at least $1-c_3\exp(-c_4d)$ that $\frac{16\gamma^2}{\alpha}\|\bm w_t\|_2^2 < 2\alpha$ for every $0 \leq t \leq T$. We then have $\left(2\gamma + \frac{2\alpha}{3}\right)\|\bbeta^{t+1} - \hat\bbeta\|_2 \leq 6\gamma + 2\alpha$, which is a contradiction with the original assumption. 
\end{proof}

\subsection{Proof of Lemma \ref{lm: noisy hard thresholding overall accuracy}} \label{sec: proof of noisy hard thresholding properties}

\begin{proof}[Proof of Lemma \ref{lm: noisy hard thresholding overall accuracy}]
	Let $T$ be the index set of the top $s$ coordinates of $\bm v$ in terms of absolute values. We have
	\begin{align*}
	\|\tilde P_s(\bm v) - \bm v\|_2^2 &= \sum_{j \in S^c} v_j^2 = \sum_{j \in S^c \cap T^c} v_j^2 + \sum_{j \in S^c \cap T} v_j^2\\
	& \leq \sum_{j \in S^c \cap T^c} v_j^2 + (1+1/c)\sum_{j \in S \cap T^c} v_j^2 + 4(1 + c)\sum_{i \in [s]} \|\bm w_i\|^2_\infty.
	\end{align*}
	The last step is true by observing that $|S \cap T^c| = |S^c \cap T|$ and applying the following lemma.
	
	\begin{Lemma}\label{lm: noisy hard thresholding subset accuracy}
		Let $S$ and $\{\bm w\}_{i \in [s]}$ be defined as in Algorithm $\ref{algo: noisy hard thresholding}$. For every $R_1 \subseteq S$ and $R_2 \in S^c$ such that $|R_1| = |R_2|$ and every $c > 0$, we have
		\begin{align*}
		\|\bm v_{R_2}\|_2^2 \leq (1 + c)\|\bm v_{R_1}\|_2^2 + 4(1 + 1/c)\sum_{i \in [s]} \|\bm w_i\|^2_\infty.  
		\end{align*}
	\end{Lemma}
	
	 Now, for an arbitrary $\hat{\bm v}$ with $\|\hat{\bm v}\|_0 = \hat s \leq s$, let $\hat S = \supp(\hat{\bm v})$. We have
	\begin{align*}
	\frac{1}{|I|-s}\sum_{j \in T^c} v_j^2 = \frac{1}{|T^c|}\sum_{j \in T^c} v_j^2 \stackrel{(*)}{\leq} \frac{1}{|(\hat S)^c|}\sum_{j \in (\hat S)^c} v_j^2 = \frac{1}{|I|-\hat s}\sum_{j \in (\hat S)^c} v_j^2 \leq \frac{1}{|I|-\hat s}\sum_{j \in (\hat S)^c} \|\hat{\bm v} - \bm v\|_2^2
	\end{align*}
	The (*) step is true because $T^c$ is the collection of indices with the smallest absolute values, and $|T^c| \leq |\hat S^c|$. We then combine the two displays above to conclude that
	\begin{align*}
	\|\tilde P_s(\bm v) - \bm v\|_2^2 &\leq \sum_{j \in S^c \cap T^c} v_j^2 + (1+1/c)\sum_{j \in S \cap T^c} v_j^2 + 4(1 + c)\sum_{i \in [s]} \|\bm w_i\|^2_\infty \\
	&\leq (1+1/c)\sum_{j \in T^c} v_j^2 + 4(1 + c)\sum_{i \in [s]} \|\bm w_i\|^2_\infty \\
	&\leq (1+1/c) \frac{|I|-s}{|I|-\hat s} \|\hat{\bm v} - \bm v\|_2^2 + 4(1 + c)\sum_{i \in [s]} \|\bm w_i\|^2_\infty.
	\end{align*}
\end{proof}
\subsubsection{Proof of Lemma \ref{lm: noisy hard thresholding subset accuracy}}
\begin{proof}[Proof of Lemma \ref{lm: noisy hard thresholding subset accuracy}]
	Let $\psi: R_2 \to R_1$ be a bijection. By the selection criterion of Algorithm \ref{algo: noisy hard thresholding}, for each $j \in R_2$ we have $|v_j| + w_{ij} \leq |v_{\psi(j)}| + w_{i\psi(j)}$, where $i$ is the index of the iteration in which $\psi(j)$ is appended to $S$. It follows that, for every $c > 0$, 
	\begin{align*}
	v_j^2 &\leq \left(|v_{\psi(j)}| + w_{i\psi(j)} - w_{ij} \right)^2 \\
	&\leq (1+1/c) v_{\psi(j)}^2 + (1 + c)(w_{i\psi(j)} - w_{ij})^2 \leq (1+1/c)v_{\psi(j)}^2 + 4(1+c)\|\bm w_i\|_\infty^2
	\end{align*}
	Summing over $j$ then leads to
	\begin{align*}
	\|\bm v_{R_2}\|_2^2 \leq (1 + 1/c)\|\bm v_{R_1}\|_2^2 + 4(1 + c)\sum_{i \in [s]} \|\bm w_i\|^2_\infty.  
	\end{align*}
\end{proof}

\subsection{Proof of Lemma \ref{lm: noisy iterative hard thresholding privacy}}\label{sec: proof of lm: noisy iterative hard thresholding privacy}
\begin{proof}[Proof of Lemma \ref{lm: noisy iterative hard thresholding privacy}]
	In view of Lemma \ref{lm: noisy hard thresholding privacy}, it suffices to control
	\begin{align*}
		\|\eta^0\nabla \L_n(\bth^t; \bm Z) - \eta^0\nabla \L_n(\bth^t; \bm Z')\|_\infty \leq (\eta^0/n)\|\nabla l(\bth; \bm z)- \nabla l(\bth; \bm z')\|_\infty < (\eta^0/n) B.
	\end{align*}
	It follows that each iteration of Algorithm \ref{algo: noisy iterative hard thresholding} is $(\varepsilon/T, \delta/T)$ differentially private. The overall privacy of Algorithm \ref{algo: noisy iterative hard thresholding} is then a consequence of composition theorem, Fact \ref{fc: composition theorems}.
\end{proof}

\subsection{Proof of Theorem \ref{thm: noisy iterative hard theresholding convergence}}\label{sec: proof of thm: noisy iterative hard thresholding convergence}
\begin{proof}[Proof of Theorem \ref{thm: noisy iterative hard theresholding convergence}]
	We first introduce some notation useful throughout the proof.
	\begin{itemize}
		\item Let $S^t = \supp(\bth^t)$, $S^{t+1} = \supp(\bth^{t+1})$ and $S^* = \supp(\hat \bth)$, and define $I^t = S^{t+1} \cup S^t \cup S^*$. 
		\item Let $\bm g^t = \nabla \L_n(\bth^t)$ and $\eta_0 = \eta/\gamma$, where $\gamma$ is the constant in \eqref{eq: rsm general}.  
		\item Let $\bm w_1, \bm w_2, \cdots, \bm w_s$ be the noise vectors dded to $\bth^t - \eta^0\nabla \L_n(\bth^t; Z)$ when the support of $\bth^{t+1}$ is iteratively selected. We define $\bm W = 4\sum_{i \in [s]} \|\bm w_i\|^2_\infty$.
	\end{itemize}
	
	We start by analyzing $\L_n(\bth^{t+1}) - \L_n(\bth^t)$. By the restricted smoothness property \eqref{eq: rsm general}, 
	\begin{align}
	\L_n(\bth^{t+1}) - \L_n(\bth^t) &\leq \langle \bth^{t+1} - \bth^{t}, \bm g^t \rangle + \frac{\gamma}{2}\|\bth^{t+1} - \bth^t\|_2^2 \notag \\
	&= \frac{\gamma}{2}\left\|\bth^{t+1}_{I^t} - \bth^t_{I^t} + \frac{\eta}{\gamma}\bm g^t_{I^t}\right\|_2^2 - \frac{\eta^2}{2\gamma}\left\|\bm g^t_{I^t}\right\|_2^2 + (1-\eta) \langle \bth^{t+1} - \bth^{t}, \bm g^t \rangle. \label{eq: master expansion 1}
	\end{align} 
	We make use of this expansion to analyze each term separately. We first branch out to the third term and obtain the following expression after some calculations.
	\begin{Lemma}\label{lm: third term in master expansion 1}
		For every $c > 0$, we have
		\begin{align*}
		\langle \bth^{t+1} - \bth^{t}, \bm g^t \rangle \leq -\frac{\eta}{2\gamma}\left\|\bm g^t_{S^{t} \cup S^{t+1}}\right\|_2^2 + (1/c)\left(4 + \frac{\eta}{2\gamma}\right) \|\bm g^t_{S^{t+1}}\|_2^2 + c\|\tilde {\bm w}_{S^{t+1}}\|_2^2 + (1+c)\frac{\gamma}{2\eta}\bm W.
		\end{align*}
	\end{Lemma}
	The lemma is proved in Section \ref{sec: proof of lemmas for thm: noisy iterative hard thresholding convergence}. Combining Lemma \ref{lm: third term in master expansion 1} with \eqref{eq: master expansion 1} yields
	\begin{align}
	&\L_n(\bth^{t+1}) - \L_n(\bth^t) \notag \\
	\leq ~ &\frac{\gamma}{2}\left\|\bth^{t+1}_{I^t} - \bth^t_{I^t} + \frac{\eta}{\gamma}\bm g^t_{I^t}\right\|_2^2 - \frac{\eta^2}{2\gamma}\left\|\bm g^t_{I^t}\right\|_2^2 -\frac{\eta(1-\eta)}{2\gamma}\left\|\bm g^t_{S^t \cup S^{t+1}}\right\|_2^2 \notag \\
	&+ \frac{1-\eta}{c}\left(4 + \frac{\eta}{2\gamma}\right) \|\bm g^t_{S^{t+1}}\|_2^2 + (1-\eta)c\|\tilde {\bm w}_{S^{t+1}}\|_2^2 + (1-\eta)(1+c)\frac{\gamma}{2\eta}\bm W \notag \\
	\leq ~ & \frac{\gamma}{2}\left\|\bth^{t+1}_{I^t} - \bth^t_{I^t} + \frac{\eta}{\gamma}\bm g^t_{I^t}\right\|_2^2 - \frac{\eta^2}{2\gamma}\left\|\bm g^t_{I^t \setminus (S^t \cup S^*)}\right\|_2^2 - \frac{\eta^2}{2\gamma}\left\|\bm g^t_{S^t \cup S^*}\right\|_2^2  -\frac{\eta(1-\eta)}{2\gamma}\left\|\bm g^t_{S^t \cup S^{t+1}}\right\|_2^2  \notag\\
	&+ \frac{1-\eta}{c}\left(4 + \frac{\eta}{2\gamma}\right) \|\bm g^t_{S^{t+1}}\|_2^2 + (1-\eta)c\|\tilde {\bm w}_{S^{t+1}}\|_2^2 + (1-\eta)(1+c)\frac{\gamma}{2\eta}\bm W \notag\\
	\leq ~ & \frac{\gamma}{2}\left\|\bth^{t+1}_{I^t} - \bth^t_{I^t} + \frac{\eta}{\gamma}\bm g^t_{I^t}\right\|_2^2 - \frac{\eta^2}{2\gamma}\left\|\bm g^t_{I^t \setminus (S^t \cup S^*)}\right\|_2^2  - \frac{\eta^2}{2\gamma}\left\|\bm g^t_{S^t \cup S^*}\right\|_2^2 -\frac{\eta(1-\eta)}{2\gamma}\left\|\bm g^t_{S^{t+1} \setminus (S^t \cup S^*)}\right\|_2^2  \notag\\
	&+\frac{1-\eta}{c}\left(4 + \frac{\eta}{2\gamma}\right) \|\bm g^t_{S^{t+1}}\|_2^2 + (1-\eta)c\|\tilde {\bm w}_{S^{t+1}}\|_2^2 + (1-\eta)(1+c)\frac{\gamma}{2\eta}\bm W. \label{eq: master expansion 2}
	\end{align}
	The last step is true because $S^{t+1} \setminus (S^t \cup S^*)$ is a subset of $S^t \cup S^{t+1}$. Now we analyze the first two terms $\frac{\gamma}{2}\left\|\bth^{t+1}_{I^t} - \bth^t_{I^t} + \frac{\eta}{\gamma}\bm g^t_{I^t}\right\|_2^2 - \frac{\eta^2}{2\gamma}\left\|\bm g^t_{I^t \setminus (S^t \cup S^*)}\right\|_2^2$
	
	\begin{Lemma}\label{lm: first two terms in master expansion 2}
		Let $\alpha$ be the restricted strong convexity constant as stated in condition \eqref{eq: rsc general}. For every $c > 1$, we have
		\begin{align*}
		&\frac{\gamma}{2}\left\|\bth^{t+1}_{I^t} - \bth^t_{I^t} + \frac{\eta}{\gamma}\bm g^t_{I^t}\right\|_2^2 - \frac{\eta^2}{2\gamma}\left\|\bm g^t_{I^t \setminus (S^t \cup S^*)}\right\|_2^2 \\
		&\leq \frac{3s^*}{s + s^*} \left(\eta \L_n(\hat \bth) - \eta \L_n(\bth^t) + \frac{\gamma-\eta\alpha}{2}\|\hat \bth - \bth^t\|_2^2 + \frac{\eta^2}{2\gamma}\|\bm g^t_{I^t}\|_2^2\right)\\
		& \quad  + \frac{\eta^2}{2c\gamma}(1+1/c)\|\bm g^t_{S^{t+1}}\|_2^2 + \frac{(c+3)\gamma}{2}\bm W + \frac{\gamma}{2}\|\tilde {\bm w}_{S^{t+1}}\|_2^2.
		\end{align*}
	\end{Lemma}
	The lemma is proved in Section \ref{sec: proof of lemmas for thm: noisy iterative hard thresholding convergence}. Substitution into \eqref{eq: master expansion 2} leads to
	\begin{align*}
	&\L_n(\bth^{t+1}) - \L_n(\bth^t) \\
	\leq ~ & \frac{3s^*}{s + s^*} \left(\eta \L_n(\hat \bth) - \eta \L_n(\bth^t) + \frac{\gamma-\eta\alpha}{2}\|\hat \bth - \bth^t\|_2^2 + \frac{\eta^2}{2\gamma}\|\bm g^t_{I^t}\|_2^2\right) \\
	&- \frac{\eta^2}{2\gamma}\left\|\bm g^t_{S^t \cup S^*}\right\|_2^2 -\frac{\eta(1-\eta)}{2\gamma}\left\|\bm g^t_{S^{t+1} \setminus (S^t \cup S^*)}\right\|_2^2\\
	& + (1/c)\left(4(1-\eta) + \frac{\eta}{2\gamma} + \frac{(1+1/c)\eta^2}{2\gamma}\right)\|\bm g^t_{S^{t+1}}\|_2^2  + \frac{\gamma}{2}\left(c+3 + \frac{(1+c)(1-\eta)}{\eta}\right)\bm W \\
	& + \left((1-\eta)c + \frac{\gamma}{2}\right)\|\tilde {\bm w}_{S^{t+1}}\|_2^2.
	\end{align*}
	Up to this point, the inequality holds for every $0 < \eta < 1$ and $c > 1$. We now specify the choice of these parameters: let $\eta  = 2/3$ and set $c$ large enough so that
	\begin{align*}
	\L_n(\bth^{t+1}) - \L_n(\bth^t) \leq ~ & \frac{3s^*}{s + s^*} \left(\eta \L_n(\hat \bth) - \eta \L_n(\bth^t) + \frac{\gamma-\eta\alpha}{2}\|\hat \bth - \bth^t\|_2^2 + \frac{\eta^2}{2\gamma}\|\bm g^t_{I^t}\|_2^2\right) \\
	&- \frac{\eta^2}{4\gamma}\left\|\bm g^t_{S^t \cup S^*}\right\|_2^2 -\frac{\eta(1-\eta)}{4\gamma}\left\|\bm g^t_{S^{t+1} \setminus (S^t \cup S^*)}\right\|_2^2\\
	& + \frac{\gamma}{2}\left(\frac{3c+7}{2}\right)\bm W + \left(\frac{c}{3} + \frac{\gamma}{2}\right)\|\tilde {\bm w}_{S^{t+1}}\|_2^2.
	\end{align*}
	Such a choice of $c$ is available because $\gamma$ is an absolute constant determined by the RSM condition. Now we set $s = 72(\gamma/\alpha)^2 s^*$, so that $\frac{3s^*}{s+s^*} \leq \frac{\alpha^2}{24\gamma(\gamma - \eta\alpha)}$, and $\frac{\alpha^2}{24\gamma(\gamma - \eta\alpha)} \leq 1/8$ because $\alpha < \gamma$. It follows that
	\begin{align*}
	\L_n(\bth^{t+1}) - \L_n(\bth^t) \leq ~ & \frac{3s^*}{s + s^*} \left(\eta \L_n(\hat \bth) - \eta \L_n(\bth^t)\right) + \frac{\alpha^2}{48\gamma}\|\hat \bth - \bth^t\|_2^2 + \frac{1}{36\gamma}\|\bm g^t_{I^t}\|_2^2 \\
	&- \frac{1}{9\gamma}\left\|\bm g^t_{S^t \cup S^*}\right\|_2^2 -\frac{1}{18\gamma}\left\|\bm g^t_{S^{t+1} \setminus (S^t \cup S^*)}\right\|_2^2\\
	& + \frac{\gamma}{2}\left(\frac{3c+7}{2}\right)\bm W + \left(\frac{c}{3} + \frac{\gamma}{2}\right)\|\tilde {\bm w}_{S^{t+1}}\|_2^2.
	\end{align*}
	Because $\|\bm g^t_{I^t}\|_2^2 = \left\|\bm g^t_{S^t \cup S^*}\right\|_2^2 + \left\|\bm g^t_{S^{t+1} \setminus (S^t \cup S^*)}\right\|_2^2$, we have
	\begin{align}
	\L_n(\bth^{t+1}) - \L_n(\bth^t) \leq ~ & \frac{3s^*}{s + s^*} \left(\eta \L_n(\hat \bth) - \eta \L_n(\bth^t)\right) + \frac{\alpha^2}{48\gamma}\|\hat \bth - \bth^t\|_2^2 - \frac{3}{36\gamma}\left\|\bm g^t_{S^t \cup S^*}\right\|_2^2 \notag\\
	& + \frac{\gamma}{2}\left(\frac{3c+7}{2}\right)\bm W + \left(\frac{c}{3} + \frac{\gamma}{2}\right)\|\tilde {\bm w}_{S^{t+1}}\|_2^2 \notag\\
	\leq ~ & \frac{3s^*}{s + s^*} \left(\eta \L_n(\hat \bth) - \eta \L_n(\bth^t)\right) - \frac{3}{36\gamma}\left(\left\|\bm g^t_{S^t \cup S^*}\right\|_2^2 - \frac{\alpha^2}{4}\|\hat \bth - \bth^t\|_2^2\right) \notag\\
	& + \frac{\gamma}{2}\left(\frac{3c+7}{2}\right)\bm W + \left(\frac{c}{3} + \frac{\gamma}{2}\right)\|\tilde {\bm w}_{S^{t+1}}\|_2^2. \label{eq: master expansion 3}
	\end{align}
	To continue the calculations, we invoke a lemma from \cite{jain2014iterative}:
	\begin{Lemma}[\citep{jain2014iterative}, Lemma 6]
		\begin{align*}
		\left\|\bm g^t_{S^t \cup S^*}\right\|_2^2 - \frac{\alpha^2}{4}\|\hat \bth - \bth^t\|_2^2 \geq \frac{\alpha}{2}\left(\L_n(\bth^t) - \L_n(\hat \bth)\right).
		\end{align*}
	\end{Lemma}
	It then follows from \eqref{eq: master expansion 3} and the lemma that, for an appropriate constant $C_\gamma$, 
	\begin{align*}
	\L_n(\bth^{t+1}) - \L_n(\bth^t) &\leq -\left(\frac{3\alpha}{72\gamma} + \frac{2s^*}{s + s^*}\right)\left(\L_n(\bth^t) - \L_n(\hat \bth)\right) + C_\gamma(\bm W + \|\tilde {\bm w}_{S^{t+1}}\|_2^2).
	\end{align*}
	The proof is now complete by adding $\L_n(\bth^t) - \L_n(\hat \bth)$ to both sides of the inequality.
\end{proof}
\subsubsection{Proofs of Lemma \ref{lm: third term in master expansion 1} and Lemma \ref{lm: first two terms in master expansion 2}}\label{sec: proof of lemmas for thm: noisy iterative hard thresholding convergence}
\begin{proof}[Proof of Lemma \ref{lm: third term in master expansion 1}]
	Since $\bth^{t+1}$ is an output from Noisy Hard Thresholding, we may write $\bth^{t+1} = {\tilde \bth}^{t+1} + \tilde {\bm w}_{S^{t+1}}$, so that ${\tilde \bth}^{t+1} = \tilde P_s(\bth^t - \eta^0\nabla \L(\bth^t; Z))$ and $\tilde {\bm w}$ is a vector consisting of $d$ i.i.d. draws from $\text{Laplace}\left(\eta_0B \cdot\frac{2\sqrt{3s\log(T/\delta)}}{n\varepsilon/T}\right)$. 
	\begin{align*}
		\langle \bth^{t+1} - \bth^t, \bm g^t \rangle &= \langle \bth^{t+1}_{S^{t+1}} - \bth^t_{S^{t+1}}, \bm g^t_{S^{t+1}} \rangle - \langle \bth^t_{S^t \setminus S^{t+1}}, \bm g^t_{S^t \setminus S^{t+1}} \rangle \\
		&= \langle {\tilde \bth}^{t+1}_{S^{t+1}} - \bth^t_{S^{t+1}}, \bm g^t_{S^{t+1}} \rangle + \langle \tilde {\bm w}_{S^{t+1}}, \bm g^t_{S^{t+1}} \rangle  - \langle \bth^t_{S^t \setminus S^{t+1}}, \bm g^t_{S^t \setminus S^{t+1}} \rangle.
	\end{align*}
	It follows that, for every $c > 0$,
	\begin{align}\label{eq: proof of third term in master expansion 1}
		\langle \bth^{t+1} - \bth^t, \bm g^t \rangle &\leq -\frac{\eta}{\gamma}\|\bm g^t_{S^{t+1}}\|_2^2 + c\|\tilde {\bm w}_{S^{t+1}}\|_2^2 + (1/4c)\|\bm g^t_{S^{t+1}}\|_2^2 - \langle \bth^t_{S^t \setminus S^{t+1}}, \bm g^t_{S^t \setminus S^{t+1}} \rangle.
	\end{align}
	Now for the last term in the display above, we have
	\begin{align*}
		- \langle \bth^t_{S^t \setminus S^{t+1}}, \bm g^t_{S^t \setminus S^{t+1}} \rangle &\leq \frac{\gamma}{2\eta}\left(\left\|\bth^t_{S^t \setminus S^{t+1}} - \frac{\eta}{\gamma}\bm g^t_{S^t \setminus S^{t+1}}\right\|_2^2 - \left(\frac{\eta}{\gamma}\right)^2\|\bm g^t_{S^t \setminus S^{t+1}}\|_2^2\right) \\
		&\leq \frac{\gamma}{2\eta}\left\|\bth^t_{S^t \setminus S^{t+1}} - \frac{\eta}{\gamma}\bm g^t_{S^t \setminus S^{t+1}}\right\|_2^2 - \frac{\eta}{2\gamma}\|\bm g^t_{S^t \setminus S^{t+1}}\|_2^2.
	\end{align*}
	We apply Lemma \ref{lm: noisy hard thresholding subset accuracy} to $\left\|\bth^t_{S^t \setminus S^{t+1}} - \frac{\eta}{\gamma}\bm g^t_{S^t \setminus S^{t+1}}\right\|_2^2$ to obtain that, for every $c > 0$, 
	\begin{align*}
		- \langle \bth^t_{S^t \setminus S^{t+1}}, \bm g^t_{S^t \setminus S^{t+1}} \rangle &\leq \frac{\gamma}{2\eta}\left[(1+1/c)\left\|\tilde{\bth}^{t+1}_{S^{t+1} \setminus S^t}\right\|_2^2 + (1+c)\bm W\right] - \frac{\eta}{2\gamma}\|\bm g^t_{S^t \setminus S^{t+1}}\|_2^2 \\
		&= \frac{\eta}{2\gamma}\left[(1+1/c)\left\|\bm g^t_{S^{t+1} \setminus S^t}\right\|_2^2 + (1+c)\frac{\gamma}{2\eta}\bm W\right] - \frac{\eta}{2\gamma}\|\bm g^t_{S^t \setminus S^{t+1}}\|_2^2.
	\end{align*}
	Plugging the inequality above back into \eqref{eq: proof of third term in master expansion 1} yields
	\begin{align*}
		\langle \bth^{t+1} - \bth^t, \bm g^t \rangle \leq~ & -\frac{\eta}{\gamma}\|\bm g^t_{S^{t+1}}\|_2^2 + c\|\tilde {\bm w}_{S^{t+1}}\|_2^2 + (1/4c)\|\bm g^t_{S^{t+1}}\|_2^2 \\ 
		&+ \frac{\eta}{2\gamma}\left[(1+1/c)\left\|\bm g^t_{S^{t+1} \setminus S^t}\right\|_2^2 + (1+c)\frac{\gamma}{2\eta}\bm W\right] - \frac{\eta}{2\gamma}\|\bm g^t_{S^t \setminus S^{t+1}}\|_2^2 \\
		\leq~ & \frac{\eta}{2\gamma}\left\|\bm g^t_{S^{t+1} \setminus S^t}\right\|_2^2 - \frac{\eta}{2\gamma}\|\bm g^t_{S^t \setminus S^{t+1}}\|_2^2 - \frac{\eta}{\gamma}\|\bm g^t_{S^{t+1}}\|_2^2 \\
		&+ (1/c)\left(4 + \frac{\eta}{2\gamma}\right) \|\bm g^t_{S^{t+1}}\|_2^2 + c\|\tilde {\bm w}_{S^{t+1}}\|_2^2 + (1+c)\frac{\gamma}{2\eta}\bm W.
	\end{align*}
	Finally, we have
	\begin{align*}
	\langle \bth^{t+1} - \bth^t, \bm g^t \rangle \leq -\frac{\eta}{2\gamma}\left\|\bm g^t_{S^{t} \cup S^{t+1}}\right\|_2^2 + (1/c)\left(4 + \frac{\eta}{2\gamma}\right) \|\bm g^t_{S^{t+1}}\|_2^2 + c\|\tilde {\bm w}_{S^{t+1}}\|_2^2 + (1+c)\frac{\gamma}{2\eta}\bm W.
	\end{align*}
\end{proof}

\begin{proof}[Proof of Lemma \ref{lm: first two terms in master expansion 2}]
Let $R$ be a subset of $S^t \setminus S^{t+1}$ such that $|R| = |I^t \setminus (S^t \cup S^*)| = |S^{t+1} \setminus (S^t \cup S^*)|$. By the definition of $\tilde \bth^{t+1}$ and Lemma \ref{lm: noisy hard thresholding subset accuracy}, we have, for every $c > 1$, 
\begin{align*}
	\frac{\eta^2}{\gamma^2}\left\|\bm g^t_{I^t \setminus (S^t \cup S^*)}\right\|_2^2 = \|\tilde \bth^{t+1}_{I^t \setminus (S^t \cup S^*)}\|_2^2 \geq (1-1/c)\left\|\bth^t_R - \frac{\eta}{\gamma}\bm g^t_R\right\|_2^2 - c\bm W.
\end{align*}
It follows that
\begin{align*}
	&\frac{\gamma}{2}\left\|\bth^{t+1}_{I^t} - \bth^t_{I^t} + \frac{\eta}{\gamma}\bm g^t_{I^t}\right\|_2^2 - \frac{\eta^2}{2\gamma}\left\|\bm g^t_{I^t \setminus (S^t \cup S^*)}\right\|_2^2 \\
	&\leq \frac{\gamma}{2}\|\tilde {\bm w}_{S^{t+1}}\|_2^2 + \frac{\gamma}{2}\left\|\tilde \bth^{t+1}_{I^t} - \bth^t_{I^t} + \frac{\eta}{\gamma}\bm g^t_{I^t}\right\|_2^2 - \frac{\gamma}{2}(1-1/c)\left\|\bth^t_R - \frac{\eta}{\gamma}\bm g^t_R\right\|_2^2 + \frac{c\gamma}{2}\bm W \\
	&= \frac{\gamma}{2}\left\|\tilde \bth^{t+1}_{I^t} - \bth^t_{I^t} + \frac{\eta}{\gamma}\bm g^t_{I^t}\right\|_2^2 - \frac{\gamma}{2}\left\||\tilde \bth^{t+1}_R - \bth^t_R + \frac{\eta}{\gamma}\bm g^t_R\right\|_2^2 + \frac{\gamma}{2}(1/c)\left\|\bth^t_R - \frac{\eta}{\gamma}\bm g^t_R\right\|_2^2+\frac{c\gamma}{2}\bm W + \frac{\gamma}{2}\|\tilde {\bm w}_{S^{t+1}}\|_2^2\\
	&\leq \frac{\gamma}{2}\left\|\tilde \bth^{t+1}_{I^t \setminus R} - \bth^t_{I^t \setminus R} + \frac{\eta}{\gamma}\bm g^t_{I^t \setminus R}\right\|_2^2 + \frac{\eta^2}{2c\gamma}(1+1/c)\left\|\bm g^t_{I^t \setminus (S^t \cup S^*)}\right\|_2^2+\frac{c\gamma}{2}\bm W + \frac{\gamma}{2}\|\tilde {\bm w}_{S^{t+1}}\|_2^2.
\end{align*}
The last inequality is obtained by applying Lemma \ref{lm: noisy hard thresholding subset accuracy} to $\left\|\bth^t_R - \frac{\eta}{\gamma}\bm g^t_R\right\|_2^2$. Now we apply Lemma \ref{lm: noisy hard thresholding overall accuracy} to obtain
\begin{align*}
	&\frac{\gamma}{2}\left\|\bth^{t+1}_{I^t} - \bth^t_{I^t} + \frac{\eta}{\gamma}\bm g^t_{I^t}\right\|_2^2 - \frac{\eta^2}{2\gamma}\left\|\bm g^t_{I^t \setminus (S^t \cup S^*)}\right\|_2^2 \\
	&\leq \frac{3\gamma}{4}\frac{|I^t \setminus R|-s}{|I^t \setminus R|-s^*}\left\|\tilde \hat \bth_{I^t \setminus R} - \bth^t_{I^t \setminus R} + \frac{\eta}{\gamma}\bm g^t_{I^t \setminus R}\right\|_2^2 +\frac{3\gamma}{2}\bm W + \frac{\eta^2(1+c^{-1})}{2c\gamma}\left\|\bm g^t_{I^t \setminus (S^t \cup S^*)}\right\|_2^2+\frac{c\gamma}{2}\bm W + \frac{\gamma}{2}\|\tilde {\bm w}_{S^{t+1}}\|_2^2 \\
	&\leq \frac{3\gamma}{4}\frac{2s^*}{s+s^*}\left\|\tilde \hat \bth_{I^t \setminus R} - \bth^t_{I^t \setminus R} + \frac{\eta}{\gamma}\bm g^t_{I^t \setminus R}\right\|_2^2 +\frac{3\gamma}{2}\bm W + \frac{\eta^2}{2c\gamma}(1+1/c)\left\|\bm g^t_{S^{t+1}}\right\|_2^2+\frac{c\gamma}{2}\bm W + \frac{\gamma}{2}\|\tilde {\bm w}_{S^{t+1}}\|_2^2.
\end{align*}
The last step is true by observing that $|I^t \setminus R| \leq 2s^*+s$, and the inclusion $I^t \setminus (S^t \cup S^*) \subseteq S^{t+1}$. We continue to simplify,
\begin{align*}
	&\frac{\gamma}{2}\left\|\bth^{t+1}_{I^t} - \bth^t_{I^t} + \frac{\eta}{\gamma}\bm g^t_{I^t}\right\|_2^2 - \frac{\eta^2}{2\gamma}\left\|\bm g^t_{I^t \setminus (S^t \cup S^*)}\right\|_2^2 \\
	&\leq \frac{\gamma}{2}\frac{3s^*}{s+s^*}\left\|\tilde \hat \bth_{I^t} - \bth^t_{I^t} + \frac{\eta}{\gamma}\bm g^t_{I^t}\right\|_2^2 +\frac{3\gamma}{2}\bm W + \frac{\eta^2}{2c\gamma}(1+1/c)\left\|\bm g^t_{S^{t+1}}\right\|_2^2+\frac{c\gamma}{2}\bm W + \frac{\gamma}{2}\|\tilde {\bm w}_{S^{t+1}}\|_2^2 \\
	&\leq \frac{3s^*}{s+s^*}\left(\eta\langle\hat \bth - \bth^t, \bm g^t\rangle + \frac{\gamma}{2}\|\hat \bth - \bth^t\|_2^2 + \frac{\eta^2}{2c\gamma}\|\bm g^t_{I^t}\|_2^2\right) \\
	&+ \frac{\eta^2}{2c\gamma}(1+1/c)\left\|\bm g^t_{S^{t+1}}\right\|_2^2+\frac{(c+3)\gamma}{2}\bm W + \frac{\gamma}{2}\|\tilde {\bm w}_{S^{t+1}}\|_2^2 \\
	&\leq \frac{3s^*}{s+s^*}\left(\eta\L_n(\hat \bth) - \eta\L_n(\bth^t) + \frac{\gamma - \eta \alpha}{2}\|\hat \bth - \bth^t\|_2^2 + \frac{\eta^2}{2c\gamma}\|\bm g^t_{I^t}\|_2^2\right) \\
	&\quad + \frac{\eta^2}{2c\gamma}(1+1/c)\left\|\bm g^t_{S^{t+1}}\right\|_2^2+\frac{(c+3)\gamma}{2}\bm W + \frac{\gamma}{2}\|\tilde {\bm w}_{S^{t+1}}\|_2^2.
\end{align*}
\end{proof}

\subsection{Proof of Lemma \ref{lm: glm privacy}}\label{sec: proof of lm: glm privacy}
\begin{proof}[Proof of Lemma \ref{lm: glm privacy}]
	For every pair of adjacent data sets $\bm Z, \bm Z'$ we have
	\begin{align*}
		\|\bbeta^{t+0.5}(\bm Z) - \bbeta^{t+0.5}(\bm Z')\|_\infty &\leq \frac{\eta^0}{n}\left(|\psi'(\bm x^\top \bbeta^t)-\Pi_R(y)|\|\bm x\|_\infty + |\psi'((\bm x')^\top \bbeta^t)-\Pi_R(y')|\|\bm x'\|_\infty\right) \\
		&\leq\frac{\eta^0}{n} 4(R+c_1)\sigma_{\bm x},
	\end{align*}
	where the last step follows from (D1') and (G1). Algorithm \ref{algo: private sparse glm} is $(\varepsilon, \delta)$-differentially private by Lemma \ref{lm: noisy iterative hard thresholding privacy}.
\end{proof}

\subsection{Omitted Steps in Section \ref{sec: proof of thm: glm upper bound}, Proof of Theorem \ref{thm: glm upper bound}}\label{sec: proof of lemmas for thm: glm upper bound}
\subsubsection{Proof of Lemma \ref{lm: glm contraction}}\label{sec: proof of lm: glm contraction}
\begin{proof}[Proof of Lemma \ref{lm: glm contraction}]
	By Algorithm \ref{algo: private sparse glm}, $\bbeta^{k}, \bbeta^{k+1}$ are both $s$-sparse with $s = 4c_0(\gamma/\alpha)^2s^*$. The scaling assumed in Theorem \ref{thm: glm upper bound} guarantees that $n \geq K s^*\log d \sqrt{\log(T/\delta)}/(\varepsilon/T)$ for a sufficiently large constant $K$, \eqref{eq: glm rsm} implies 	
	\begin{align}
	& \langle \nabla \L_n(\bbeta^{k+1}) - \nabla \L_n(\bbeta^k), \bbeta^{k+1} - \bbeta^k \rangle \leq  \frac{4\gamma}{3}\|\bbeta^{k+1} - \bbeta^k\|_2^2. \label{eq: glm rsm modified contraction}
	\end{align}
	Similarly, because $\|\bbeta^{k} - \hat\bbeta\|_2 \leq 3$ by assumption, the RSC condition \eqref{eq: glm rsc} implies that 
	\begin{align}
	& \langle \nabla \L_n(\bbeta^{k}) - \nabla \L_n(\hat\bbeta), \bbeta^{k} - \hat\bbeta \rangle \geq  \frac{2\alpha}{3}\|\bbeta^{k} - \hat\bbeta\|_2^2.\label{eq: glm rsc modified contraction}
	\end{align} 
	Let $\bm g^k = \nabla \L_n(\bbeta^k; Z)$. It follows from \eqref{eq: glm rsm modified contraction} and \eqref{eq: glm rsc modified contraction} that,
	\begin{align}
	&\L_n(\bbeta^{k+1}) - \L_n(\hat\bbeta) \notag \\
	&= \L_n(\bbeta^{k+1}) - \L_n(\bbeta^{k}) + \L_n(\bbeta^{k}) - \L_n(\hat\bbeta) \notag\\
	&\leq \langle \bm g^k ,  \bbeta^{k+1}- \bbeta^{k}\rangle + \frac{2\gamma}{3}\|\bbeta^{k+1}- \bbeta^{k}\|_2^2 + \langle \bm g^k ,  \bbeta^{k}- \hat\bbeta\rangle - \frac{\alpha}{3}\|\bbeta^{k}- \hat\bbeta\|_2^2 \notag\\
	&\leq \langle \bm g^k ,  \bbeta^{k+1}- \hat\bbeta\rangle + \gamma\|\bbeta^{k+1}- \bbeta^{k}\|_2^2 - \frac{\alpha}{3}\|\bbeta^{k}- \hat\bbeta\|_2^2\notag\\
	& = \langle 2\gamma(\bbeta^k - \bbeta^{k+1}) ,  \bbeta^{k+1}- \hat\bbeta\rangle + \gamma\|\bbeta^{k+1}- \bbeta^{k}\|_2^2  - \frac{\alpha}{3}\|\bbeta^{k}- \hat\bbeta\|_2^2 + \langle \bm g^k - 2\gamma(\bbeta^k- \bbeta^{k+1}) ,  \bbeta^{k+1}- \hat\bbeta\rangle \notag\\
	&= \left(\gamma - \frac{\alpha}{3}\right)\|\bbeta^{k}- \hat\bbeta\|_2^2 - \gamma\|\bbeta^{k+1}- \hat\bbeta\|_2^2 + \langle \bm g^k - 2\gamma(\bbeta^k - \bbeta^{k+1}) ,  \bbeta^{k+1}- \hat\bbeta\rangle. \label{eq: glm contraction master expansion 1}
	\end{align}
	Let $S^{k+1}$, $\hat S$ denote the supports of $\bbeta^{k+1}$, $\hat\bbeta$ respectively. Since $\bbeta^{k+1}$ is an output from Noisy Hard Thresholding, we may write $\bbeta^{k+1} = {\tilde \bbeta}^{k+1} + \tilde {\bm w}_{S^{k+1}}$, so that ${\tilde \bth}^{k+1} = \tilde P_s(\bbeta^k - (1/2\gamma)\nabla \L_n(\bbeta^k; Z))$ and $\tilde {\bm w}$ is the Laplace noise vector.
	
	Now we continue the calculation. For the last term of \eqref{eq: glm contraction master expansion 1},
	\begin{align}
	&\langle \bm g^k - 2\gamma(\bbeta^k - \bbeta^{k+1}) ,  \bbeta^{k+1}- \hat\bbeta\rangle \notag \\
	&= 2\gamma \langle \tilde {\bm w}_{S^{k+1}}, \bbeta^{k+1} - \hat\bbeta \rangle +  2\gamma \langle \tilde{\bbeta}^{k+1} - \bbeta^k + (1/2\gamma)\bm g^k, \bbeta^{k+1} - \hat\bbeta \rangle \notag \\
	& \leq \frac{36\gamma^2}{\alpha} \|\tilde {\bm w}_{S^{k+1}}\|_2^2 + \frac{36\gamma^2}{\alpha}\|(\tilde{\bbeta}^{k+1} - \bbeta^k + (1/2\gamma)\bm g^k)_{S^{k+1} \cup \hat S}\|_2^2 + \frac{2\alpha}{9}\|\bbeta^{k+1} - \hat\bbeta\|_2^2 \label{eq: glm contraction master expansion 2}
	\end{align} 
	For the middle term of \eqref{eq: glm contraction master expansion 2}, since $S^{k+1} \subseteq S^{k+1} \cup \hat S$, we have $\tilde P_s((\bbeta^k + (1/2\gamma)\bm g^k)_{S^{k+1} \cup \hat S}) = \tilde{\bbeta}^{k+1}_{S^{k+1} \cup \hat S}$, and therefore Lemma \ref{lm: noisy hard thresholding overall accuracy} applies. Because $|S^{k+1} \cup \hat S| \leq s + s^*$, we have
	\begin{align*}
	&\|(\tilde{\bbeta}^{k+1} - \bbeta^k + (1/2\gamma)\bm g^k)_{S^{k+1} \cup \hat S}\|_2^2 \\
	&\leq \frac{5}{4}\frac{s^*}{s}\|(\hat\bbeta - \bbeta^k + (1/2\gamma)\bm g^k)_{S^{k+1} \cup \hat S}\|_2^2 + 20 \sum_{i \in [s]} \|\bm w_i\|^2_\infty \\
	&\leq \frac{5\alpha^2}{16c_0\gamma^2}\left(\frac{5}{3}\|\bbeta^k - \hat\bbeta\|_2^2 + \frac{5/2}{4\gamma^2}\|\bm g^k\|_2^2\right) + 20 \sum_{i \in [s]} \|\bm w_i\|^2_\infty
	\leq \frac{125\alpha^2}{16c_0\gamma^2} + 20\sum_{i \in [s]} \|\bm w_i\|^2_\infty.
	\end{align*}
	For the last step to go through, we invoke the assumption that $\|\bbeta^k - \hat\bbeta\|_2 < 3$ and we have $\|\bm g^k\|^2_2  = \|\nabla \L_n(\bbeta^k) - \nabla \L_n(\hat\bbeta)\|_2^2 \leq (4\gamma/3)^2\|\bbeta^k - \hat\bbeta\|_2^2 \leq 16\gamma^2$ by \eqref{eq: glm rsm modified contraction}. We recall from the proof of Theorem \ref{thm: noisy iterative hard theresholding convergence} that $c_0 = 72$;  substituting the inequality above into \eqref{eq: glm contraction master expansion 2} yields
	\begin{align}
	&\langle \bm g^k - 2\gamma(\bbeta^k - \bbeta^{k+1}) ,  \bbeta^{k+1}- \hat\bbeta\rangle \notag\\
	&\leq \frac{125\alpha}{32} + \frac{36\gamma^2}{\alpha} \left(\|\tilde {\bm w}_{S^{k+1}}\|_2^2 + 20\sum_{i \in [s]} \|\bm w_i\|^2_\infty\right) + \frac{2\alpha}{9}\|\bbeta^{k+1} - \hat\bbeta\|_2^2. \label{eq: glm contraction remainder term}
	\end{align}	
	 To analyze the noise term in the middle, we apply Lemma \ref{lm: laplace noise bound}. Because the individual coordinates of $\tilde{\bm w}$, $\bm w_i$ are sampled i.i.d. from the Laplace distribution with scale $(2\gamma)^{-1} \cdot \frac{2\sqrt{3s\log(T/\delta)}}{n\varepsilon/T}$, if $n \geq K s^*\log d \sqrt{\log(T/\delta)}/(\varepsilon/T)$ for a sufficiently large constant $K$, Lemma \ref{lm: laplace noise bound} implies that, with probability at least $1-c_3\exp(-c_4\log(d/s^*))$ for some appropriate constants $c_3, c_4$, the noise term $(36\gamma^2/\alpha)\left(\|\tilde {\bm w}_{S^{k+1}}\|_2^2 + 20\sum_{i \in [s]} \|\bm w_i\|^2_\infty\right) < 3\alpha/32$. We substitute this upper bound back into \eqref{eq: glm contraction remainder term}, and then combine \eqref{eq: glm contraction remainder term} with \eqref{eq: glm contraction master expansion 1} to obtain
	\begin{align}
	\L_n(\bbeta^{k+1}) - \L_n(\hat\bbeta) &\leq \left(\gamma - \frac{\alpha}{3}\right)\|\bbeta^{k}- \hat\bbeta\|_2^2 - \left(\gamma-\frac{2\alpha}{9}\right)\|\bbeta^{k+1}- \hat\bbeta\|_2^2 + 4\alpha. \label{eq: glm contraction master expansion 3}
	\end{align}
	Let us now assume by contradiction that $\|\bbeta^{k+1} - \hat\bbeta\|_2 > 3$. From \eqref{eq: glm rsc} and \eqref{eq: glm rsc modified contraction} we know that $\L_n(\bbeta^{k+1}) - \L_n(\hat\bbeta) \geq \alpha \|\bbeta^{k+1} - \hat\bbeta\|_2$. We combine this observation, the assumptions that $\|\bbeta^{k+1} - \hat\bbeta\|_2 > 3, \|\bbeta^k - \hat\bbeta\|_2 < 3$ and \eqref{eq: glm contraction master expansion 3} to obtain
	\begin{align*}
	\left(3\gamma + \frac{\alpha}{3} \right)\|\bbeta^{k+1} - \hat\bbeta\|_2 \leq 9\gamma + \alpha,
	\end{align*}	
	which contradicts the original assumption that $\|\bbeta^{k+1} - \hat\bbeta\|_2 > 3$.
\end{proof}
\subsubsection{Proof of Lemma \ref{lm: laplace noise bound}} \label{sec: proof of lm: laplace noise bound}
\begin{proof}[Proof of Lemma \ref{lm: laplace noise bound}]
	By union bound and the i.i.d. assumption,
	\begin{align*}
	\Pro\left(\|\bm w\|_2^2 > kC^2\lambda^2\right) \leq k\Pro(w_1^2 > C^2\lambda^2)  \leq ke^{-C}. 
	\end{align*}
	It follows that
	\begin{align*}
	\Pro\left(\|\bm w\|_\infty^2 > C^2\lambda^2\log^2k\right) \leq k\Pro(w_1^2 > C^2\lambda^2\log^2k) \leq ke^{-C\log k} = e^{-(C-1)\log k}.
	\end{align*}
\end{proof}

%% file: B_Lower_bound_proofs.tex
%!TEX root = Privacy-GLM-JRSSB.tex

\section{Omitted Proofs in Section \ref{sec: glm lower bounds}}\label{sec: lb proofs}

\subsection{Proof of Theorem \ref{thm: low-dim glm lb}}\label{sec: proof of thm: low-dim glm lb}
\begin{proof}[Proof of Theorem \ref{thm: low-dim glm lb}]
	It can be shown via Theorem \ref{thm: score attack general} that the score attack \eqref{eq: low-dim glm attack} is indeed sound and complete:
	
	\begin{Lemma}\label{lm: low-dim glm attack}
		Under the assumptions of Theorem \ref{thm: low-dim glm lb}, the score attack \eqref{eq: low-dim glm attack} satisfies the following properties.
		\begin{enumerate}
			\item Soundness: For each $i \in [n]$ let $(\bm y'_i, \bm X'_i)$ denote the data set obtained by replacing $(y_i, \bm x_i)$ in $(\bm y, \bm X)$ with an independent copy, then $\E \A_\bbeta ((y_i, \bm x_i), M(\bm y'_i, \bm X'_i)) = 0$ and $\E |\A_\bbeta ((y_i, \bm x_i), M(\bm y'_i, \bm X'_i))| \leq \sqrt{\E\|M(\bm y, \bm X) - \bbeta\|^2_2}\sqrt{Cc_2/c(\sigma)}.$
			\item Completeness: $\sum_{i \in [n]} \E \A_\bbeta ((y_i, \bm x_i), M(\bm y, \bm X)) = \sum_{j \in [d]} \frac{\partial}{\partial \beta_j} \E M(\bm y, \bm X)_j$.
		\end{enumerate}
	\end{Lemma}
	
	We follow the strategy outlined in Section \ref{sec: general lb} to establish appropriate upper and lower bounds for $\sum_{i \in [n]} \E \A_\bbeta ((y_i, \bm x_i), M(\bm y, \bm X))$, using Lemma \ref{lm: low-dim glm attack}.
	
	\textbf{Step 1. upper bounding the score attacks.}
	We first work on the upper bound. Define $A_i = \A_\bbeta ((y_i, \bm x_i), M(\bm y, \bm X))$; the soundness part of Lemma \ref{lm: low-dim glm attack} and Lemma \ref{lm: score attack upper bound} together imply that
	\begin{align*}
	\E A_i \leq 2\varepsilon \sqrt{\E\|M(\bm y, \bm X) - \bbeta\|^2_2} \sqrt{Cc_2/c(\sigma)} +  2\delta T + \int_T^\infty \Pro(|A_i| > t) \d t.
	\end{align*}
	We need to choose $T$ so that the remainder terms are controlled. We have
	\begin{align*}
	\Pro(|A_i| > t) &= \Pro\left(\left|\frac{y_i - \psi'(\bm x_i^\top\bbeta)}{c(\sigma)}\right| \left|\langle \bm x_i, M(\bm y, \bm X)-\bbeta \rangle\right| > t\right) \\
	&\leq \Pro\left(\left|\frac{y_i - \psi'(\bm x_i^\top\bbeta)}{c(\sigma)}\right|d > t\right).
	\end{align*}
	For the first term, consider
	$f_{\theta}(y) = h(y, \sigma)\exp\left(\frac{y\theta - \psi(\theta)}{c(\sigma)}\right)$ and we have 
	\begin{align*}
	\E \exp(\lambda y) = \int e^{\lambda y} h(y, \sigma)\exp\left(\frac{y\theta - \psi(\theta)}{c(\sigma)}\right) \d y = \exp\left(\frac{\psi(\theta + c(\sigma)\lambda) - \psi(\theta)}{c(\sigma)}\right).
	\end{align*}
	We may then compute the moment generating function of $\frac{y_i - \psi'(\langle\bm x_i, \bbeta\rangle)}{c(\sigma)}$, conditional on $\bm x_i$:
	\begin{align*}
	\log \E \exp\left(\lambda 
	\cdot \frac{y_i - \psi'(\bm x_i^\top\bbeta)}{c(\sigma)}\Big| \bm x_i\right) &= \frac{1}{c(\sigma)}\left(\psi(\bm x_i^\top\bbeta + \lambda) - \psi(\bm x_i^\top\bbeta) - \lambda\psi'(\bm x_i^\top\bbeta)\right) \\
	&\leq \frac{1}{c(\sigma)} \cdot \frac{\lambda^2 \psi^{''}(\bm x_i^\top\bbeta + \tilde \lambda)}{2}
	\end{align*}
	for some $\tilde \lambda \in (0, \lambda)$. It follows that $\E \exp\left(\lambda \cdot \frac{y_i - \psi'(\bm x_i^\top\bbeta)}{c(\sigma)}\Big| \bm x_i\right) \leq \exp\left(\frac{c_2\lambda^2}{2c(\sigma)}\right)$ because $\|\psi^{''}\|_\infty <c_2$. The bound for moment generating function implies that
	\begin{align*}
	\Pro(|A_i| > t) 
	&\leq \Pro\left(\left|\frac{y_i - \psi'(\bm x_i^\top\bbeta)}{c(\sigma)}\right|d > t\right)  \leq \exp\left(-\frac{c(\sigma)t^2}{2c_2d^2}\right).
	\end{align*}
	It follows that
	\begin{align*}
	\E A_i &\leq 2\varepsilon\sqrt{\E\|M(\bm y, \bm X) - \bbeta\|^2_2} \sqrt{Cc_2/c(\sigma)} +  2\delta T + \int_T^\infty \Pro(|A_i| > t) \d t \\
	&\leq 2\varepsilon\sqrt{\E\|M(\bm y, \bm X) - \bbeta\|^2_2} \sqrt{Cc_2/c(\sigma)} +  2\delta T + 2\sqrt{c_2/c(\sigma)}d\exp\left(-\frac{c(\sigma)T^2}{2c_2d^2}\right).
	\end{align*}
	We set $T = \sqrt{2c_2/c(\sigma)}d\sqrt{\log(1/\delta)}$ to obtain
	\begin{align}\label{eq: low-dim glm attack upper}
	\sum_{i \in [n]} \E A_i \leq 2n\varepsilon\sqrt{\E\|M(\bm y, \bm X) - \bbeta\|^2_2} \sqrt{Cc_2/c(\sigma)} + 4\sqrt{2}\delta d\sqrt{c_2\log(1/\delta)/c(\sigma)}.
	\end{align}
	
	\textbf{Step 2. lower bounding the score attacks.}
	Next we prove a lower bound for $\sum_{i \in [n]} \E \A_\bbeta ((y_i, \bm x_i), M(\bm y, \bm X))$, or more precisely, an \textit{average} lower bound with respect to an appropriately chosen prior distribution of $\bbeta$. By completeness in Lemma \ref{lm: low-dim glm attack}, 
	\begin{align*}
	\sum_{i \in [n]} \E \A_\bbeta ((y_i, \bm x_i), M(\bm y, \bm X)) = \sum_{j \in [d]} \frac{\partial}{\partial \beta_j} \E M(\bm y, \bm X)_j.
	\end{align*}
	By Lemma \ref{lm: score attack stein's lemma} and the choice of $\bm \pi(\bbeta)$ as the density of $N(\bm 0, \bm I)$, Example \ref{ex: Gaussian stein's lemma} implies
	\begin{align}\label{eq: low-dim glm attack lower}
	\sum_{i \in [n]} \E_{\bm \pi} \E_{\bm y, \bm X|\bbeta} \A_\bbeta ((y_i, \bm x_i), M(\bm y, \bm X)) \gtrsim d.
	\end{align}
	
	\textbf{Step 3. establishing the minimax risk lower bound.}
	We combine \eqref{eq: low-dim glm attack upper} and \eqref{eq: low-dim glm attack lower} to prove the minimax risk lower bound \eqref{eq: low-dim glm lb}. Since \eqref{eq: low-dim glm attack upper} holds for every fixed $\bbeta$, for any choice of prior $\bm \pi$, we have
	\begin{align*}
	d &\lesssim \E_{\bm \pi}\left[\sum_{i \in [n]} \E \A_\bbeta ((y_i, \bm x_i), M(\bm y, \bm X))\right] \\
	&\leq 2n\varepsilon\E_{\bm \pi} \sqrt{\E_{\bm y, \bm X|\bbeta} \|M(\bm y, \bm X) - \bbeta\|^2_2} \sqrt{Cc_2/c(\sigma)} + 4\sqrt{2}n\delta d\sqrt{c_2\log(1/\delta)/c(\sigma)}.
	\end{align*}
	It follows that
	\begin{align*}
	2n\varepsilon\E_{\bm \pi} \sqrt{\E_{\bm y, \bm X|\bbeta} \|M(\bm y, \bm X) - \bbeta\|^2_2} \sqrt{Cc_2/c(\sigma)} \gtrsim d - 4\sqrt{2}n\delta d\sqrt{c_2\log(1/\delta)/c(\sigma)}.
	\end{align*}
	The assumption of $\delta < n^{-(1+\gamma)}$ implies $d - 4\sqrt{2}n\delta d \sqrt{C_1\log(1/\delta)/c(\sigma)} \gtrsim d$. We can then conclude that
	\begin{align*}
	\E_{\bm \pi}\E_{\bm y, \bm X|\bbeta} \|M(\bm y, \bm X) - \bbeta\|^2_2 \gtrsim \frac{c(\sigma)d^2}{n^2\varepsilon^2}
	\end{align*}
	The proof is complete because the minimax risk is always greater than the Bayes risk.
\end{proof}
\subsubsection{Proof of Lemma \ref{lm: low-dim glm attack}}
\begin{proof}
	On the basis of Theorem \ref{thm: score attack general}, it suffices to calculate the score statistic of $f(y, \bm x)$ with respect to $\bbeta$ and the Fisher information matrix. In particular, all regularity conditions required for exchanging integration and differentiation are satisfied since $f_\bbeta(y|x)$ is an exponential family. We have
	\begin{align*}
	\frac{\partial}{\partial \bbeta} \log f(y, \bm x) &= \frac{\partial}{\partial \bbeta} \log \left(f_\bbeta(y|\bm x) f(\bm x)\right) = \frac{\partial}{\partial \bbeta} \log f_\bbeta(y|\bm x) \\
	&= \frac{\partial}{\partial \bbeta}\left(\frac{\bm x^\top \bbeta y - \psi(\bm x^\top \bbeta)}{c(\sigma)}\right) = \frac{[y - \psi'(\bm x^\top \bbeta)]\bm x}{c(\sigma)}.
	\end{align*}
	For the Fisher information, we have
	\begin{align*}
	\mathcal I(\bbeta) = -\E \left(\frac{\partial^2}{\partial \bbeta^2} \log f(y, \bm x) \right) = \E\left(\frac{\psi''(\bm x^\top \bbeta)}{c(\sigma)} \bm x \bm x^\top\right) \preceq \frac{c_2}{c(\sigma)} \E[\bm x \bm x^\top],
	\end{align*}
	where the last inequality holds by $\|\psi''\|_\infty \leq c_2$. We then have $\lambda_{\max}(\mathcal I(\bbeta)) \leq Cc_2 /c(\sigma)$ by $\lambda_{\max}(\E[\bm x \bm x^\top]) \leq C$.
\end{proof}

\subsection{Proof of Theorem \ref{thm: high-dim glm lb}}\label{sec: proof of thm: high-dim glm lb}
\begin{proof}[Proof of Theorem \ref{thm: high-dim glm lb}]
	By Theorem \ref{thm: score attack general}, the sparse GLM score attack is sound and complete.
	Similar to the classical GLM score attack, the sparse attack \eqref{eq: high-dim glm attack} also satisfies soundness and completeness.
	\begin{Lemma}\label{lm: high-dim glm attack}
		Under the assumptions of Theorem \ref{thm: high-dim glm lb}, the score attack \eqref{eq: high-dim glm attack} satisfies the following properties.
		\begin{enumerate}
			\item Soundness: For each $i \in [n]$ let $(\bm y'_i, \bm X'_i)$ denote the data set obtained by replacing $(y_i, \bm x_i)$ in $(\bm y, \bm X)$ with an independent copy, then $\E \A_{\bbeta,s^*} ((y_i, \bm x_i), M(\bm y'_i, \bm X'_i)) = 0$ and $\E |\A_{\bbeta,s^*} ((y_i, \bm x_i), M(\bm y'_i, \bm X'_i))| \leq \sqrt{\E\|(M(\bm y, \bm X) - \bbeta)_{\supp(\bbeta)}\|^2_2}\sqrt{Cc_2/c(\sigma)}.$
			\item Completeness: $\sum_{i \in [n]} \E \A_{\bbeta,s^*} ((y_i, \bm x_i), M(\bm y, \bm X)) = \sum_{j \in [d]} \frac{\partial}{\partial \beta_j} \E(M(\bm y, \bm X)_j\1(\beta_j \neq 0))$.
		\end{enumerate}
	\end{Lemma}
	From the soundness and completeness properties, we can follow the strategy in Section \ref{sec: general lb} to derive the minimax risk lower bound.
	
	\textbf{Step 1. upper bounding the score attacks.}
	Define $A_i = \A_{\bbeta,s^*}((y_i, \bm x_i), M(\bm y, \bm X))$; the soundness part of Lemma \ref{lm: high-dim glm attack} and Lemma \ref{lm: score attack upper bound} together imply that
	\begin{align*}
	\E A_i \leq 2\varepsilon \sqrt{\E\|(M(\bm y, \bm X) - \bbeta)_{\supp(\bbeta)}\|^2_2} \sqrt{Cc_2/c(\sigma)} +  2\delta T + \int_T^\infty \Pro(|A_i| > t) \d t.
	\end{align*}
	We look for $T$ such that the remainder terms are controlled. We have
	\begin{align*}
	\Pro(|A_i| > t) &= \Pro\left(\left|\frac{y_i - \psi'(\bm x_i^\top\bbeta)}{c(\sigma)}\right| \left|\langle \bm x_i, (M(\bm y, \bm X)-\bbeta)_{\supp(\bbeta)} \rangle\right| > t\right) \\
	&\leq \Pro\left(\left|\frac{y_i - \psi'(\bm x_i^\top\bbeta)}{c(\sigma)}\right|s^* > t\right).
	\end{align*}
	In the proof of Theorem \ref{thm: low-dim glm lb}, we have found $\E \exp\left(\lambda \cdot \frac{y_i - \psi'(\bm x_i^\top\bbeta)}{c(\sigma)}\Big| \bm x_i\right) \leq \exp\left(\frac{c_2\lambda^2}{2c(\sigma)}\right)$. The bound for moment generating function then yields
	\begin{align*}
	\Pro(|A_i| > t) 
	&\leq \Pro\left(\left|\frac{y_i - \psi'(\bm x_i^\top\bbeta)}{c(\sigma)}\right|s^* > t\right)
	\leq \exp\left(-\frac{c(\sigma)t^2}{2c_2(s^*)^2}\right).
	\end{align*}
	It follows that
	\begin{align*}
	\E A_i &\leq 2\varepsilon\sqrt{\E\|(M(\bm y, \bm X) - \bbeta)_{\supp(\bbeta)}\|^2_2} \sqrt{Cc_2/c(\sigma)} +  2\delta T + \int_T^\infty \Pro(|A_i| > t) \d t \\
	&\leq 2\varepsilon\sqrt{\E\|M(\bm y, \bm X) - \bbeta\|^2_2} \sqrt{Cc_2/c(\sigma)} +  2\delta T + 2s\sqrt{c_2/c(\sigma)}\exp\left(-\frac{c(\sigma)T^2}{2c_2(s^*)^2}\right).
	\end{align*}
	We choose $T = \sqrt{2c_2/c(\sigma)}s^*\sqrt{\log(1/\delta)}$ to obtain
	\begin{align}\label{eq: high-dim glm attack upper}
	\sum_{i \in [n]} \E A_i \leq 2n\varepsilon\sqrt{\E\|M(\bm y, \bm X) - \bbeta\|^2_2} \sqrt{Cc_2/c(\sigma)} + 4\sqrt{2}\delta s^* \sqrt{c_2\log(1/\delta)/c(\sigma)}.
	\end{align}
	
	\textbf{Step 2. lower bounding the score attacks.}
	By completeness in Lemma \ref{lm: high-dim glm attack}, 
	\begin{align*}
	\sum_{i \in [n]} \E \A_{\bbeta,s^*} ((y_i, \bm x_i), M(\bm y, \bm X)) = \sum_{j \in [d]} \frac{\partial}{\partial \beta_j} \E (M(\bm y, \bm X)_j\1(\beta_j \neq 0)).
	\end{align*}
	For Lemma \ref{lm: score attack stein's lemma} to apply, we will choose some appropriate $\bm \pi(\bbeta)$; unlike the proof of Theorem \ref{thm: low-dim glm lb}, we have to find some prior distribution for $\bbeta$ that obeys the sparsity condition $\|\bbeta\|_0 \leq s^*$. To this end, consider $\bbeta$ generated as follows: let $\tilde \beta_1, \tilde \beta_2, \cdots, \tilde \beta_d$ be drawn i.i.d. from $N(0, 1)$, let $I_{s^*}$ be a subset of $[d]$ with $|I_{s^*}| = s^*$, and define $\beta_j = \tilde \beta_j \1(j \in I_{s^*})$, so that $\|\bbeta\|_0 \leq s^*$ by construction, and $\supp(\bbeta) = I_{s^*}$.
	
	It now follows from Lemma \ref{lm: score attack stein's lemma} that, if $\E_{\bm y, \bm X|\bbeta}\|(M(\bm y, \bm X) - \bbeta)_{\supp(\bbeta)}\|_2^2 < 1$, 
	\begin{align*}
	\E_{\bm \pi} \left(\sum_{j \in [d]} \frac{\partial}{\partial \beta_j} \E (M(\bm y, \bm X)_j\1(\beta_j \neq 0))\right)
	&\geq \E_{\bm \pi}\left(\sum_{j \in [d]}  \beta_j^2\1(j \in I_{s^*}) \right)-  \sqrt{\E_{\bm \pi}\left(\sum_{j \in [d]}  \beta_j^2\1(j \in I_{s^*})\right)}.
	\end{align*}
	We choose $I_{s^*}$ to be the index set of $\tilde \bbeta$ with top $s^*$ greatest absolute values, and invoke the following lemma:
	\begin{Lemma}\label{lm: gaussian top s}
		Consider $Z_1, Z_2, \cdots, Z_d$ drawn i.i.d. from $N(0,1)$, and let $|Z|_{(1)}, |Z|_{(2)}, \cdots, |Z|_{(d)}$ be the order statistics of $\{|Z_j|\}_{j \in [d]}$. If $ s = o(d)$, for sufficiently large $d$ we have
		\begin{align*}
		\E |Z|^2_{(d-s+1)} > c \log(d/s).
		\end{align*}
		for some constant $c > 0$ and consequently $\sum_{k = 0}^{s-1} \E |Z|^2_{(d-k)} \gtrsim s\log(d/s)$.
	\end{Lemma}
	Therefore, by our choice of prior for $\bbeta$, we have
	\begin{align}\label{eq: high-dim glm attack lower}
	\sum_{i \in [n]} 	\E_{\bm \pi}\E \A_{\bbeta,s^*} ((y_i, \bm x_i), M(\bm y, \bm X)) \gtrsim s^*\log(d/s^*) - \sqrt{s^*\log(d/s^*)} \gtrsim s^*\log(d/s^*).
	\end{align}
	
	\textbf{Step 3. establishing the minimax risk lower bound.}
	We combine \eqref{eq: high-dim glm attack upper} and \eqref{eq: high-dim glm attack lower} to prove the minimax risk lower bound \eqref{eq: high-dim glm lb}. Since \eqref{eq: high-dim glm attack upper} holds for every fixed $\bbeta$,  for our choice of prior $\bm \pi$, we have
	\begin{align*}
	s^*\log(d/s^*) &\lesssim \E_{\bm \pi}\left[\sum_{i \in [n]} \E \A_{\bbeta,s^*}((y_i, \bm x_i), M(\bm y, \bm X))\right] \\
	&\leq 2n\varepsilon\E_{\bm \pi} \sqrt{\E_{\bm y, \bm X|\bbeta} \|M(\bm y, \bm X) - \bbeta\|^2_2} \sqrt{Cc_2/c(\sigma)} + 4\sqrt{2}n\delta s^* \sqrt{c_2\log(1/\delta)/c(\sigma)}.
	\end{align*}
	It follows that
	\begin{align*}
	2n\varepsilon\E_{\bm \pi} \sqrt{\E_{\bm y, \bm X|\bbeta} \|M(\bm y, \bm X) - \bbeta\|^2_2} \sqrt{Cc_2/c(\sigma)} \gtrsim s^*\log(d/s^*) - 4\sqrt{2}n\delta s^* \sqrt{c_2\log(1/\delta)/c(\sigma)}.
	\end{align*}
	The assumption that $\delta < n^{-(1+\gamma)}$ for some $\gamma > 0$ implies that for $n$ sufficiently large, $s^*\log(d/s^*) - 4\sqrt{2}n\delta s^* \sqrt{c_2\log(1/\delta)/c(\sigma)} \gtrsim s^*\log(d/s^*)$. We then conclude that
	\begin{align*}
	\E_{\bm \pi}\E_{\bm y, \bm X|\bbeta} \|M(\bm y, \bm X) - \bbeta\|^2_2 \gtrsim \frac{c(\sigma)(s^*\log(d/s^*))^2}{n^2\varepsilon^2}
	\end{align*}
	The proof is complete because the minimax risk is always greater than the Bayes risk.
\end{proof}

\subsubsection{Proof of Lemma \ref{lm: high-dim glm attack}}
\begin{proof}
	We observe that $(M(\bm y, \bm X) - \bbeta)_{\supp(\bbeta)} = M(\bm y, \bm X)_{\supp(\bbeta)} - \bbeta$. The lemma is then a consequence of Theorem \ref{thm: score attack general} and the score and Fisher information calculations in the proof of Lemma \ref{lm: low-dim glm attack}.
\end{proof}

\subsubsection{Proof of Lemma \ref{lm: gaussian top s}}
\begin{proof}
	We denote $Y = |Z|_{(d-s+1)}$ and observe that
	\begin{align*}
	\Pro(Y > t)  = 1 - \Pro(Y \leq t) = 1 - \Pro\left(\sum_{j \in [d]} \1(|Z_j| > t) \leq s \right)
	\end{align*}
	Since $\Pro(|Z_i| > t) \geq t^{-1}\exp(-t^2/2)$ for $t \geq \sqrt{2}$ by Mills ratio, we choose $t = \sqrt{\log(d/2s)}$, so that $t^{-1}\exp(-t^2/2) > 2s/d$ as long as $d > 2s$. Now consider $N \sim$ Binomial$(d, 2s/d)$; we have $\Pro\left(\sum_{j \in [d]} \1(|Z_j| > t) \leq s \right) \leq \Pro (N \leq s)$. By standard Binomial tail bounds \cite{arratia1989tutorial}, 
	\begin{align*}
	\Pro (N \leq s) &\leq \exp\left[-d\left((s/d)\log(1/2) + (1-s/d)\log\left(\frac{1-s/d}{1-2s/d}\right)\right)\right] \\
	&\leq 2^s\left(1-\frac{s}{d-s}\right)^{d-s} < (2/e)^s
	\end{align*}
	It follows that $\Pro\left(Y > \sqrt{\log(d/2s)}\right) > 1- (2/e)^s > 0.$ Because $Y$ is non-negative, the proof is complete.
\end{proof}